%% file: main.tex
\documentclass{article}

\usepackage{microtype}
\usepackage{graphicx}
\usepackage{subcaption}
\usepackage{booktabs} % for professional tables

\usepackage{hyperref}

\usepackage[accepted]{icml2023}

\usepackage{amsmath}
\usepackage{amssymb}
\usepackage{mathtools}
\usepackage{amsthm}

\usepackage[capitalize,noabbrev]{cleveref}

\theoremstyle{plain}
\newtheorem{theorem}{Theorem}[section]

\newtheorem{lemma}[theorem]{Lemma}
\newtheorem{corollary}[theorem]{Corollary}
\newtheoremstyle{TheoremNum}
    {\topsep}{\topsep}              %%% space between body and thm
    {\itshape}                      %%% Thm body font
    {}                              %%% Indent amount (empty = no indent)
    {\bfseries}                     %%% Thm head font
    {.}                             %%% Punctuation after thm head
    { }                             %%% Space after thm head
    {\thmname{#1}\thmnote{ \bfseries #3}}%%% Thm head spec
\theoremstyle{TheoremNum}
\newtheorem{lemmanum}{Lemma}
\theoremstyle{definition}
\newtheorem{definition}[theorem]{Definition}

\theoremstyle{remark}
\newtheorem{remark}[theorem]{Remark}

\usepackage[textsize=tiny]{todonotes}

\usepackage{standalone}
\usepackage{tikz/tikz_include}
\usepackage{bbm}
\usepackage{multirow}
\usepackage{array}
\usepackage{amsbsy}
\usepackage{blindtext}

\input{math_commands.tex}

\icmltitlerunning{Learning Expressive Priors for Generalization and Uncertainty Estimation in Neural Networks}

\begin{document}

\twocolumn[
\icmltitle{Learning Expressive Priors for Generalization and \\ Uncertainty Estimation in Neural Networks}



\icmlsetsymbol{equal}{*}

\begin{icmlauthorlist}
\icmlauthor{Dominik Schnaus}{equal,tum,mcml}
\icmlauthor{Jongseok Lee}{equal,dlr,kit}
\icmlauthor{Daniel Cremers}{tum,mcml}
\icmlauthor{Rudolph Triebel}{dlr,kit}
\end{icmlauthorlist}

\icmlaffiliation{tum}{TU Munich}%{TUM School of Computation, Information and Technology, Technical University of Munich, Garching, Germany}
\icmlaffiliation{mcml}{MCML}
\icmlaffiliation{dlr}{DLR}%{German Aerospace Center, Oberpfaffenhofen, Germany}
\icmlaffiliation{kit}{KIT}%{Department of Informatics, Karlsruhe Institute of Technology, Karlsruhe, Germany}

\icmlcorrespondingauthor{Dominik Schnaus}{dominik.schnaus@tum.de}
\icmlcorrespondingauthor{Jongseok Lee}{jongseok.lee@dlr.de}

\icmlkeywords{Machine Learning, ICML, Bayesian Neural Networks, Bayesian Deep Learning, Continual Learning, Laplace Approximation, Learned Prior, Progressive Neural Networks}

\vskip 0.3in
]



\printAffiliationsAndNotice{\icmlEqualContribution} % otherwise use the standard text.

\input{sections/abstract.tex}

\input{sections/introduction.tex}
\input{sections/methods.tex}
\input{sections/relatedworks.tex}
\input{sections/results.tex}
\input{sections/conclusion.tex}

\bibliography{bib}
\bibliographystyle{icml2023}

%%%%%%%%%%%%%%%%%%%%%%%%%%%%%%%%%%%%%%%%%%%%%%%%%%%%%%%%%%%%%%%%%%%%%%%%%%%%%%%
%%%%%%%%%%%%%%%%%%%%%%%%%%%%%%%%%%%%%%%%%%%%%%%%%%%%%%%%%%%%%%%%%%%%%%%%%%%%%%%
% APPENDIX
%%%%%%%%%%%%%%%%%%%%%%%%%%%%%%%%%%%%%%%%%%%%%%%%%%%%%%%%%%%%%%%%%%%%%%%%%%%%%%%
%%%%%%%%%%%%%%%%%%%%%%%%%%%%%%%%%%%%%%%%%%%%%%%%%%%%%%%%%%%%%%%%%%%%%%%%%%%%%%%
\newpage
\appendix
\onecolumn
\input{sections/appendix.tex}

\end{document}

%% file: math_commands.tex
\usepackage{amsmath,amsfonts,bm}

\def\eqref#1{equation~\ref{#1}}
\def\1{\bm{1}}

\def\rvx{{\mathbf{x}}}
\def\rvy{{\mathbf{y}}}

\DeclareMathAlphabet{\mathsfit}{\encodingdefault}{\sfdefault}{m}{sl}
\SetMathAlphabet{\mathsfit}{bold}{\encodingdefault}{\sfdefault}{bx}{n}

\makeatletter
\DeclareRobustCommand\onedot{\futurelet\@let@token\@onedot}
\def\@onedot{\ifx\@let@token.\else.\null\fi\xspace}

\def\eg{\emph{e.g}\onedot} 
\def\ie{\emph{i.e}\onedot}

\def\wrt{w.r.t\onedot}

\makeatother

\DeclareMathOperator*{\argmax}{arg\,max}
\DeclareMathOperator*{\argmin}{arg\,min}

\DeclareMathOperator*{\er}{er}
\DeclareMathOperator*{\trace}{tr}
\DeclareMathOperator*{\vectorize}{vec}
\DeclareMathOperator*{\diag}{diag}

\DeclareMathOperator*{\rank}{rank}

\DeclareMathOperator*{\aer}{aer}
\DeclareMathOperator*{\kl}{kl}

\newcommand{\approptoinn}[2]{\mathrel{\vcenter{
  \offinterlineskip\halign{\hfil$##$\cr
    #1\propto\cr\noalign{\kern2pt}#1\sim\cr\noalign{\kern-2pt}}}}}

\newcommand{\appropto}{\mathpalette\approptoinn\relax}

%% file: sections/abstract.tex
\begin{abstract}
In this work, we propose a novel prior learning method for advancing generalization and uncertainty estimation in deep neural networks. The key idea is to exploit scalable and structured posteriors of neural networks as informative priors with generalization guarantees. Our learned priors provide expressive probabilistic representations at large scale, like Bayesian counterparts of pre-trained models on ImageNet, and further produce non-vacuous generalization bounds. We also extend this idea to a continual learning framework, where the favorable properties of our priors are desirable. Major enablers are our technical contributions: (1) the sums-of-Kronecker-product computations, and (2) the derivations and optimizations of tractable objectives that lead to improved generalization bounds. Empirically, we exhaustively show the effectiveness of this method for uncertainty estimation and generalization. 
\end{abstract}

%% file: sections/introduction.tex
\section{Introduction}

Within the deep learning approach to real-world AI problems such as autonomous driving, generalization and uncertainty estimation are one of the most important pillars. To achieve this, Bayesian Neural Networks (BNNs)~\citep{mackay1992practical, hinton1993keeping, neal2012bayesian} leverage the tools of Bayesian statistics in order to improve generalization and uncertainty estimation in deep neural networks. Due to their potential and advancements so far, BNNs have become increasingly popular research topics~\citep{gawlikowski2021survey}. However, one of the open problems in BNNs is the prior specifications. While it is widely known that prior selection is a crucial step in any Bayesian modeling~\citep{bayes1763lii}, the choice of well-specified priors is generally unknown in neural networks. Consequently, many current approaches have resorted to uninformative priors, like isotropic Gaussian, despite reported signs of prior misspecifications~\citep{wenzel2020how, fortuin2021priors}.

To address the problem of prior specification, we propose a prior learning method for BNNs. Our method is grounded in sequential Bayesian inference~\citep{opper1999bayesian}, where the posteriors from the past are used as the prior for the future. We achieve this by relying on Laplace Approximation (LA) \citep{mackay1992practical} with Kronecker-factorization of the posteriors \citep{ritter2018scalable, daxberger2021laplace}. Within this framework, we devise key technical novelties to learn expressive BNN priors with generalization guarantees. First, we demonstrate the sums-of-Kronecker-product computations so that the posteriors in matrix normal distributions, \ie, Gaussian with Kronecker-factorized covariance matrices \citep{gupta2000matrix}, can be tractably used as expressive BNN priors. Second, to explicitly improve generalization, we derive and optimize over a tractable approximation of PAC-Bayes bounds \citep{mcallester1999some,germain2016pac} that lead to non-vacuous bounds, \ie, smaller than the upper bound of the loss. Finally, as an added benefit of our idea, a Bayesian re-interpretation to a popular continual learning architecture, namely progressive neural networks \citep{rusu2016progressive}, is provided for uncertainty-awareness, generalization, and resiliency to catastrophic forgetting.

By design, our method has many advantages for generalization and uncertainty estimation in neural networks. The proposed method achieves non-vacuous generalization bounds in deep learning models, while potentially avoiding prior misspecification for uncertainty estimation using BNNs. For these features, we provide computationally tractable techniques in order to learn expressive priors from large amounts of data and deep network architectures. We contend that such probabilistic representation at a large scale, e.g., pre-trained BNNs on ImageNet, can be beneficial in downstream tasks for both transfer and continual learning set-ups. Therefore, we further provide ablation studies and various experiments to show the aforementioned benefits of our method within the small and large-scale transfer and continual learning tasks. In particular, we empirically demonstrate that our priors mitigate cold posterior effects ~\citep{wenzel2020how} -- a potential sign of a bad prior -- and further produce generalization bounds. Moreover, within a practical scenario of robotic continual learning~\citep{denninger2018persistent}, and standardized benchmarks of few-shot classification~\citep{tran2022plex} for the recent uncertainty baselines~\citep{nado2021uncertainty}, considerable performance improvements over the popular methods are reported.

\textbf{Contributions} To summarize, our primary contribution is a novel method for learning scalable and structured informative priors (Section~\ref{sec:method}), backed up by (a) the sums-of-Kronecker-product computations for computational tractability (Section~\ref{sec:method:empirical:prior}), (b) derivation and optimizations over tractable generalization bounds (Section~\ref{sec:method:pacbayes}), (c) a Bayesian re-interpretation of progressive neural networks for continual learning (Section~\ref{sec:method:bayespnn}), (d) exhaustive experiments to show the effectiveness of our method (Section~\ref{sec:results}).%\footnote{We open-source our code at \url{https://github.com/DLR-RM/PBNN}.} 

%% file: sections/methods.tex
\section{Learning Expressive Priors}
\label{sec:method}

\subsection{Notation and Background}
\label{sec:method:background}
Consider a neural network $f_{\boldsymbol{\theta}}$ with layers $l \in [L] := \{1, \dots, L\}$ which consists of a parameterized function $\phi_l: \Omega_{l-1} \to \Omega_{l}$ and a non-linear activation function $\sigma_l: \Omega_{l} \to \Omega_{l}$. We denote an input for a network as $\mathbf{x} \in \Omega_{0} =: \mathcal{X}$ and all learnable parameters as a stacked vector $\boldsymbol{\theta}$. Then the pre-activation $\mathbf{s}_l$ and activation $\mathbf{a}_l$ are recursively applied with $\mathbf{a}_0 = \mathbf{x}$, $\mathbf{s}_l = \phi_l(\mathbf{a}_{l-1})$ and $\mathbf{a}_l = \sigma_l(\mathbf{s}_l)$. The output of the neural network $\mathbf{y}$ is given by the last activation $\mathbf{y} = \mathbf{a}_L = f_{\boldsymbol{\theta}}(\mathbf{x})$. The structure of the learning process is governed by different architectural choices such as fully connected, recurrent, and convolutional layers~\citep{goodfellow2016deep}. The parameters $\boldsymbol{\theta}$ are typically obtained by maximum likelihood principles, given training data $\mathcal{D} = \left((\mathbf{x}_i, \mathbf{y}_i)\right)_{i=1}^N$. Unfortunately, these principles lead to the lack of generalization and calibrated uncertainty in neural networks~\citep{jospin2020hands}.

To address these, BNNs provide a compelling alternative by applying Bayesian principles to neural networks. In BNNs, the first step is to specify a prior distribution $\pi(\boldsymbol{\theta})$. Then the Bayes theorem is used to compute the posterior distribution over the parameters, given the training data: $p(\boldsymbol{\theta} | \mathcal{D}) \propto \pi(\boldsymbol{\theta}) \prod_{i=1}^N p(\mathbf{y}_i | \mathbf{x}_i, f_{\boldsymbol{\theta}})$. This means that each parameter is not a single value, but a probability distribution, representing the uncertainty of the model \citep{gawlikowski2021survey}. The posterior distribution provides a set of plausible model parameters, which can be marginalized to compute the predictive distributions of a new sample $(\mathbf{x}^*, \mathbf{y}^*) \notin \mathcal{D}$ for BNNs: $p(\mathbf{y}^* | \mathbf{x}^*, \mathcal{D}) = \int p(\mathbf{y}^* | \mathbf{x}^*, f_{\boldsymbol{\theta}}) \rho(\boldsymbol{\theta}) d\boldsymbol{\theta}$. In the case of neural networks, the posteriors cannot be obtained in closed form. Hence, many of the current research efforts are on accurately approximating the posteriors of BNNs.

For this, LA~\citep{mackay1992practical} approximates the BNN posteriors with a Gaussian distributions around a local mode. Here, a prior is first specified as an isotropic Gaussian. With this prior, the maximum-a-posteriori (MAP) estimates of the parameters $\boldsymbol{\hat{\theta}}$ are obtained by training the neural networks. Then, the Hessian $\mathbf{H} = \frac{d^2}{{d\boldsymbol{\theta}}^2} \ln{p(\boldsymbol{\theta}|\mathcal{D})}=\mathbf{H}_{likelihood} + \mathbf{H}_{prior}$ is computed to obtain the BNN posteriors. In practice, the covariance matrix is usually further scaled which more generally corresponds to temperature scaling. In summary, the priors-posteriors pairs of LA are:
\begin{flalign}
\label{eq:la:prior:posterior}
    &\text{Prior:} \quad \pi(\boldsymbol{\theta})= \mathcal{N}(\boldsymbol{0}, \gamma \mathbf{I}), &&\\
    &\text{Posterior:} \quad p(\boldsymbol{\theta}|\mathcal{D}) \approx \mathcal{N}(\boldsymbol{\hat{\theta}}, (\mathbf{H}_{likelihood} + \mathbf{H}_{prior})^{-1}). &&\nonumber
\end{flalign}
As the Hessian is computationally intractable for modern architectures, the Kronecker-Factored Approximate Curvature (KFAC) method is widely adopted~\citep{martens2015optimizing}. KFAC approximates the true Hessian by a layer-wise block-diagonal matrix, where each diagonal block is the Kronecker-factored Fisher matrix $\mathbf{F}_l$. Therefore, defining the layer-wise activation $\mathbf{\bar{a}}_l$ and loss gradient w.r.t pre-activation $\mathcal{D} \mathbf{s}_l$, the inverse covariance matrix of the posteriors is~\citep{ritter2018scalable}: 
\begin{align}
\label{eq:la:kfac:hessian}
    &\mathbf{H} \approx\mathbf{F}=\diag(\mathbf{F}_1, \mathbf{F}_2, \cdots, \mathbf{F}_L) + \gamma \mathbf{I}, \\
    &\text{where} \ \mathbf{F}_l \approx \mathbb{E}[\mathcal{D} \mathbf{s}_l (\mathcal{D} \mathbf{s}_l)^T] \otimes \mathbb{E}[\mathbf{\bar{a}}_l (\mathbf{\bar{a}}_l)^T]=\mathbf{L}_l \otimes \mathbf{R}_l.\nonumber
\end{align}
In this way, using LA, BNN posteriors can be obtained by approximating the Hessian of neural networks. We provide more details on LA and KFAC in \cref{sec:laplace_approximation}.

We note several ramifications of this formulation for learning BNN priors from data. First, the BNN posteriors can be obtained by approximating the Hessian, which can be tractably computed from large amounts of data and deeper neural network architectures~\citep{lee2020estimating, ba2017distributed}. Second, the resulting BNN posteriors can capture the correlations between the parameters within the same layer~\citep{ritter2018scalable}. Moreover, several open-source libraries exist to ease the implementations~\citep{daxberger2021laplace,humt2020bayesian}. All these points result in easily-obtainable BNN posteriors with expressive probabilistic representation from large amounts of data, deep architectures, and parameter correlations. As opposed to isotropic Gaussian, we next demonstrate that these BNN posteriors can be used to learn expressive prior distributions in order to advance generalization and uncertainty estimation within the transfer and continual learning set-ups.

\subsection{Empirical Prior Learning with Sums-of-Kronecker-Product Computations}
\label{sec:method:empirical:prior}

Intuitively, the idea is to repeat the LA with the prior chosen as the posterior from \cref{eq:la:prior:posterior}, similar to Bayesian filtering. Since we use the LA with a Kronecker-factored covariance matrix in both the prior and the posterior, we want to approximate the resulting sum of Kronecker products with a single Kronecker product. We denote the task to compute the prior as source task $\mathfrak{T}_0$ and the task to compute the posterior as target task $\mathfrak{T}_1$, both consisting of a data-set $\mathcal{D}_t$ from a distribution $P_t$, $t \in \{0, 1\}$.

Applying LA on $\mathcal{D}_0$, we obtain the model parameters $\boldsymbol{\hat{\theta}}^{(0)}$ and the posteriors $p(\boldsymbol{\theta}|\mathcal{D}_0)$ as before (equations \ref{eq:la:prior:posterior} and \ref{eq:la:kfac:hessian}). Then, for the target task, we specify the prior using the posteriors from $\mathcal{D}_0$: $\pi(\boldsymbol{\theta}) = p(\boldsymbol{\theta}|\mathcal{D}_0)$. The ultimate goal is to compute the new posteriors on $\mathcal{D}_1$, \ie, $p(\boldsymbol{\theta}|\mathcal{D}_1) \propto p(\mathcal{D}_1|f_{\boldsymbol{\theta}}) \pi(\boldsymbol{\theta})$. To achieve this, we train the neural networks on $\mathcal{D}_1$ by optimizing:
\begin{alignat}{2}
    \boldsymbol{\hat{\theta}}^{(1)} &\in \argmax_{\boldsymbol{\theta}} &&\{p(\boldsymbol{\theta}|\mathcal{D}_1)\}\nonumber\\
    &= \argmin_{\boldsymbol{\theta}} &&\{-\ln{\left(p(\mathcal{D}_1|f_{\boldsymbol{\theta}})\right)}\\
    & &&+ \frac{1}{2}(\boldsymbol{\theta} - \boldsymbol{\hat{\theta}}^{(0)})^T(\mathbf{F}^{(0)} + \gamma \mathbf{I})(\boldsymbol{\theta} - \boldsymbol{\hat{\theta}}^{(0)})\}\nonumber,
\end{alignat}
% RT: I added {} to separate better between argmax and p; also, it shows clearer what the min is taken from in the second line.
where $\boldsymbol{\hat{\theta}}^{(1)}$ is the MAP estimate of the model parameters for task $\mathfrak{T}_1$ and $\mathbf{F}^{(0)}$ is the block-diagonal Kronecker-factored Fisher matrix from task $\mathfrak{T}_0$. The final step is to approximate new Hessian on $\mathcal{D}_1$ using the KFAC method. This results in% the new pairs:% the new BNN prior-posteriors pairs:
\begin{flalign}
    &\text{Prior:} \quad \pi = \mathcal{N}(\boldsymbol{\hat{\theta}}^{(0)}, (\mathbf{F}^{(0)} + \gamma \mathbf{I})^{-1}) && \\
    &\text{Posterior:} \quad \rho = \mathcal{N}(\boldsymbol{\hat{\theta}}^{(1)}, \tau (\mathbf{F}^{(1)} + \mathbf{F}^{(0)} + \gamma \mathbf{I})^{-1}), &&\nonumber
\end{flalign}
where $\tau$ is the temperature scaling~\citep{wenzel2020how, daxberger2021laplace}. Here, the precision matrix of the new BNN posteriors is represented by the sum-of-Kronecker-products, \ie, $\mathbf{\tilde{F}}^{(1)} = \mathbf{F}^{(1)} + \mathbf{F}^{(0)} + \gamma \mathbf{I} = \mathbf{L}^{(0)} \otimes \mathbf{R}^{(0)}  + \mathbf{L}^{(1)} \otimes \mathbf{R}^{(1)}  + \gamma \mathbf{I} \otimes \mathbf{I}$.

\begin{algorithm}[ht]
   \caption{Power method for sums of Kronecker products}
   \label{alg:kronpower:main}
\begin{algorithmic}[1]
   \STATE \textbf{Input}: left matrices $(\mathbf{L}^k)_{k \in [K]}$, right matrices $(\mathbf{R}^k)_{k \in [K]}$, number of steps $n^{max}=100$, and stopping precision $\delta=10^{-5}$.
   \STATE \textbf{Output}: The solutions $\mathbf{\hat{L}}$ and $ \mathbf{\hat{R}}$. %\in \argmin_{\mathbf{L}, \mathbf{R}} \|\sum_{k=1}^K \mathbf{L}^k \otimes \mathbf{R}^k - \mathbf{L} \otimes \mathbf{R}\|_F$.
   \STATE $\vectorize(\mathbf{\Bar{L}}^{(0)}) \leftarrow \mathcal{N}(\mathbf{0}, \mathbf{I})$ \COMMENT{initialization of $\mathbf{\Bar{L}}^{(0)}$}
   \STATE $\mathbf{L}^{(0)} \leftarrow \frac{\mathbf{\Bar{L}}^{(0)}}{\|\mathbf{\Bar{L}}^{(0)}\|_F}$ \COMMENT{normalize $\mathbf{\Bar{L}}^{(0)}$}
   \FOR{$n=1$ to $n^{max}$}
        \STATE $\mathbf{\Bar{R}}^{(n)} \leftarrow \sum_{k=1}^K \langle \mathbf{L}^k, \mathbf{L}^{(n-1)} \rangle_F \mathbf{R}^k$ \COMMENT{first power step}
        \STATE $\mathbf{R}^{(n)} \leftarrow \frac{\mathbf{\Bar{R}}^{(n)}}{\|\mathbf{\Bar{R}}^{(n)}\|_F}$ \COMMENT{normalize $\mathbf{\Bar{R}}^{(n)}$}
        \STATE $\mathbf{\Bar{L}}^{(n)} \leftarrow \sum_{k=1}^K \langle \mathbf{R}^k, \mathbf{R}^{(n)} \rangle_F \mathbf{L}^k$ \COMMENT{second power step}
        \STATE $\mathbf{L}^{(n)} \leftarrow \frac{\mathbf{\Bar{L}}^{(n)}}{\|\mathbf{\Bar{L}}^{(n)}\|_F}$ \COMMENT{normalize $\mathbf{\Bar{L}}^{(n)}$}
        \IF{$\|\mathbf{L}^{(n)} - \mathbf{L}^{(n-1)}\|_F < \delta$}
            \STATE {\bfseries break} \COMMENT{stopping criterion}
        \ENDIF
   \ENDFOR
   \STATE $\mathbf{\hat{L}} \leftarrow \mathbf{L}^{(n)}$
   \STATE $\mathbf{\hat{R}} \leftarrow \sum_{k=1}^K \langle \mathbf{L}^k, \mathbf{L}^{(n)} \rangle_F \mathbf{R}^k$ \COMMENT{first power step}
\end{algorithmic}
\end{algorithm}

Unfortunately, the sum-of-Kronecker-products is not a Kronecker-factored matrix. As a result, the above formulation loses all the essence of Kronecker factorization for modeling high-dimensional probability distributions. For example, we can no longer exploit the rule: $(\mathbf{L} \otimes \mathbf{R})^{-1} = (\mathbf{L}^{-1} \otimes \mathbf{R}^{-1})$ where $\mathbf{L}$ and $\mathbf{R}$ are smaller matrices to store and invert when compared to $(\mathbf{L} \otimes \mathbf{R})$. Even more, the storage and the inverse of $\mathbf{\tilde{F}}^{(1)}$ may not be computationally tractable for modern architectures.

To this end, we devise the sum-of-Kronecker-products computations. Concretely, we approximate the sum-of-Kronecker-products as Kronecker-factored matrices by an optimization formulation:\footnote{For the similar causes to maintain the Kronecker factorization, \citet{ritter2018scalable,ritter2018online} assume $(\mathbf{L} \otimes \mathbf{R} + \gamma \mathbf{I} )^{-1} = (\mathbf{L} + \gamma \mathbf{I})^{-1} \otimes (\mathbf{R}+ \gamma \mathbf{I})^{-1}$ which does not hold in general (see Section~\ref{sec:results:ablations:main}).}\footnote{Our formulation for one single Kronecker-factored matrix is similar to \citet{kao2021natural}.}
\begin{align}
\label{eq:opt_kron_sum}
    \mathbf{\hat{L}}, \mathbf{\hat{R}} \in \argmin_{\substack{\mathbf{L} \in \mathbb{R}^{M \times M}\\ \mathbf{R} \in \mathbb{R}^{N \times N}}} \|\sum_{k=1}^K \mathbf{L}^k \otimes \mathbf{R}^k - \mathbf{L} \otimes \mathbf{R}\|_F.
\end{align}
A solution to this problem is not unique, e.g., one could scale  $\mathbf{\hat{L}}$ by $\alpha \neq 0$ and $\mathbf{\hat{R}}$ by $\frac{1}{\alpha}$. Hence, we assume that $\mathbf{\hat{L}}$ is normalized, $\|\mathbf{\hat{L}}\|_F = 1$ (or alternatively, we can also assume $\|\mathbf{\hat{R}}\|_F = 1$). For the solution, we show the equivalence to the well-known best rank-one approximation problem. 

\begin{lemma}
\label{lem:vec_kron}
Let $M, N, K \in \mathbb{N}$, $\mathbf{L}^k \in \mathbb{R}^{M \times M}$ and $\mathbf{R}^k \in \mathbb{R}^{N \times N}$ for $k \in [K]$. Then
\begin{flalign}
    \label{eq:vec_kron}
    &\|\sum_{k=1}^K \mathbf{L}^k \otimes \mathbf{R}^k - \mathbf{L} \otimes \mathbf{R}\|_F &&\\
    &= \|\sum_{k=1}^K \vectorize(\mathbf{L}^k) \vectorize(\mathbf{R}^k)^T - \vectorize(\mathbf{L}) \vectorize(\mathbf{R})^T\|_F.&&\nonumber
\end{flalign}
\end{lemma}
\begin{proof}
The proof can be found in \cref{subsec:lemma1_proof}.
\end{proof}
This result indicates that the solution to \eqref{eq:opt_kron_sum} can be obtained using the power method, which iterates: %While the precise algorithms are presented in Appendix X, the power method roughly iterates:
\begin{flalign*}
&\mathbf{R}^{(n)} \leftarrow \frac{\sum_{k=1}^K \langle \mathbf{L}^k, \mathbf{L}^{(n-1)} \rangle_F \mathbf{R}^k }{\|\sum_{k=1}^K \langle \mathbf{L}^k, \mathbf{L}^{(n-1)} \rangle_F \mathbf{R}^k\|_F} \quad \text{and}&&\\
&\mathbf{L}^{(n)} \leftarrow \frac{\sum_{k=1}^K \langle \mathbf{R}^k, \mathbf{R} \rangle_F \mathbf{L}^k} {\|\sum_{k=1}^K \langle \mathbf{R}^k, \mathbf{R} \rangle_F \mathbf{L}^k\|_F}.&&
\end{flalign*}
Given randomly initialized matrices, the power method iterates until the stop criterion is reached. The full procedure is presented in \cref{alg:kronpower:main}, whereas in Appendix~\ref{appendix:bpnn_overview}, we discuss the computational complexity of the algorithm.

Importantly, we can now obtain a single Kronecker factorization from the sum of Kronecker products. Such Kronecker-factored approximations have advantages on tractable memory consumption for the storage and the inverse computations, which enables sampling from the resulting matrix normal distributions~\citep{martens2015optimizing}. Therefore, such computations allow us to learn expressive BNN priors from the posterior distributions of previously encountered tasks. Finally, we can also prove the convergence of the power method to an optimal rank-one solution.
\begin{lemma}
    \label{lem:opt_kron_sum}
    Let $\mathbf{A} = \sum_{k=1}^K \vectorize(\mathbf{L}^k) \vectorize(\mathbf{R}^k)^T$ and $\mathbf{A} = \sum_{i=1}^r \sigma_i \mathbf{u}_i \mathbf{v}_i^T$ be its singular value decomposition with $\sigma_1 \geq \sigma_2 \geq \dots \geq \sigma_r > 0$ and $\mathbf{u}_i^T \mathbf{u}_j = \mathbf{v}_i^T \mathbf{v}_j = \mathbbm{1}[i = j]$. Then there is a solution of Equation~\ref{eq:opt_kron_sum} with $\vectorize(\mathbf{\hat{L}}) = \mathbf{u}_1$,$ \vectorize(\mathbf{\hat{R}}) = \sigma_1 \mathbf{v}_1$. If $\sigma_1 > \sigma_2$, the solution is unique up to changing the sign of both factors, and our power method converges almost surely to this solution.
\end{lemma}
\begin{proof}
The proof can be found in \cref{subsec:lemma2_proof}.
\end{proof}

\subsection{Derivations and Optimizations over Tractable PAC-Bayes Bounds}
\label{sec:method:pacbayes}

So far, we have presented a method for learning a scalable and structured prior. In this section, we show how generalization bounds for the LA can be adapted to explicitly minimize the generalization bounds with the learned prior. In particular, the proposed approach allows tuning the hyperparameters of the LA on the training set without a costly grid search. This is achieved by optimizing a differentiable cost function that approximates generalization bounds for the LA. We choose to optimize generalization bounds to trade off the broadness of the distribution and the performance on the training data to improve generalization.

A common method in BNNs, and especially in LA, is to scale the covariance matrix by a positive scalar called the temperature $\tau > 0$~\citep{wenzel2020how, daxberger2021laplace}. In addition, often either the Hessian (or Fisher matrix) of the likelihood~\citep{ritter2018scalable, lee2020estimating} or the prior~\citep{gawlikowski2021survey} is scaled. Including all these terms, the posterior has the following form:
\begin{align}
    \rho = \mathcal{N}(\boldsymbol{\hat{\theta}}^{(1)}, \tau (\beta \mathbf{F}^{(1)} + \alpha \mathbf{\tilde{F}}^{(0)})^{-1}),
\end{align}
where $\mathbf{\tilde{F}}^{(0)} = \mathbf{F}^{(0)} + \gamma \mathbf{I}$ is the precision matrix of the prior. This reweighting of the three scalars allows us to cope with misspecified priors~\citep{wilson2020bayesian}, approximation errors in the Fisher matrices, and to improve the broadness of the distribution. While these values are usually hyperparameters that have to be manually scaled by hand, we argue that these values should be application-dependent. Therefore, we propose to find all three scales -- $\alpha$, $\beta$, and $\tau$ -- by minimizing generalization bounds.

In the following, we will first introduce our two approaches and then explain how the optimization of PAC-Bayes bounds can be made tractable for the LA.

\textbf{Method 1: Curvature Scaling} 
\textit{Previous approaches typically scale only one or two of the scales~\citep{daxberger2021laplace, ritter2018scalable, lee2020estimating}. With our automatic scaling, we can optimize all three scales for each individual layer, allowing the model to decide the weighting in a more fine-grained way. We can use this method not only to scale the posterior, but also to compute the prior. Thus, the prior and posterior are defined as}
\begin{align}
    \pi = \mathcal{N}(\boldsymbol{\hat{\theta}}^{(0)}, (\mathbf{\tilde{F}}^{(0)})^{-1}),\quad \rho = \mathcal{N}(\boldsymbol{\hat{\theta}}^{(1)}, (\mathbf{\tilde{F}}^{(1)})^{-1}),
\end{align}
\textit{where the diagonal blocks of the precision matrix corresponding to each layer are defined as} $\mathbf{\tilde{F}}^{(0)}_l = \frac{1}{\mathbf{\tau}^{(0)}_l}(\mathbf{\beta}^{(0)}_l \mathbf{F}^{(0)}_l + \mathbf{\alpha}^{(0)}_l \gamma \mathbf{I})$ and $\mathbf{\tilde{F}}^{(1)}_l = \frac{1}{\mathbf{\tau}^{(0)}_l}(\mathbf{\beta}^{(1)}_l \mathbf{F}^{(1)}_l + \mathbf{\alpha}^{(1)}_l \mathbf{\tilde{F}}^{(0)}_l)$, \textit{respectively.}
\begin{remark}
Although the temperature scaling could be captured within the other curvature scales, we find it easier to optimize the bounds with all three parameters per layer.
\end{remark}

\textbf{Method 2: Frequentist Projection}
\textit{Going one step further, we propose to scale not only the curvature but also the network parameters using PAC-Bayes bounds. Since the Fisher matrix is known only after training, we assume the same covariance for the posterior as for the prior when optimizing the network parameters. Since this method does not optimize for the MAP parameters, but only for minimizing the PAC-Bayes bounds, we call it frequentist projection.}

Both methods aim at improving the generalization of our model. To do this, we use PAC-Bayes bounds~\citep{germain2016pac, guedj2019primer} to derive a tractable objective that can be optimized without going through the data-set. The main idea of PAC-Bayes theory is to upper bound the expected loss on the true data distribution $P^N$ with high probability. The bounds typically depend on the expected empirical loss on the training data and the KL-divergence between a data-independent prior and the posterior distribution. For $\varepsilon > 0$, that is
\begin{align}
    \label{eq:pac_bayes}
    P_{\mathcal{D} \sim P^N} &( \forall \rho \ll \pi:\ \mathbb{E}_{\boldsymbol{\theta} \sim \rho}[\mathcal{L}^{l}_{P}(f_{\boldsymbol{\theta}})] \\
    &\leq \delta(\mathbb{E}_{\boldsymbol{\theta} \sim \rho}[\hat{\mathcal{L}}^{l}_{\mathcal{D}}(f_{\boldsymbol{\theta}})], \mathbb{KL}(\rho\|\pi), N, \varepsilon) ) \geq 1 - \varepsilon, \nonumber
\end{align}
where the loss on the true data distribution is denoted as $\mathcal{L}^{l}_{P}(f_{\boldsymbol{\theta}}) = \mathbb{E}_{(\mathbf{x}, \mathbf{y}) \sim P}[l(f_{\boldsymbol{\theta}}(\mathbf{x}), \mathbf{y})]$ and the empirical loss on the training data  is $\hat{\mathcal{L}}^{l}_{\mathcal{D}}(f_{\boldsymbol{\theta}})=\frac{1}{N}\sum_{i = 1}^N l(f_{\boldsymbol{\theta}}(\mathbf{x}), \mathbf{y})$. In general, the bounds balance the fit to the available data (empirical risk term) and the similarity of the posterior to the prior (the KL-divergence term), and the bounds hold with a probability greater than $1 - \varepsilon$. Various forms of the $\delta$ bound can be found in the literature, with varying degrees of tightness. For classification, we rely on the McAllester~\citep{mcallester1999some} and Catoni~\citep{catoni2007pac} bounds. \footnote{This work focuses on classification tasks as PAC-Bayes framework is usually for bounded losses. Yet, using recent theories of PAC-Bayes, we also comment on regression in Appendix~\ref{subsec:regression:comments}.}

Since these bounds depend on the expected empirical loss, we have to iterate over the entire training data and sample from the posterior multiple times at each step in order to evaluate the bound once. This makes optimization intractable. Using the error function as a loss, we propose to compute an approximate upper bound by only using quantities that were already computed during the LA, \ie the Fisher matrix $\diag(\{\mathbf{F}_l\}_{l \in [L]})$ and the network parameters $\boldsymbol{\hat{\theta}}$:
\begin{align}
    \label{eq:approximate_upper_bound}
    &\mathbb{E}_{\boldsymbol{\theta} \sim \rho}\left[\frac{1}{N} \sum_{i=1}^N \mathbbm{1}[\argmax_{\mathbf{y}'} p(\mathbf{y}'|\mathbf{x}_i, f_{\boldsymbol{\theta}}) \neq \mathbf{y}_i] \right] \nonumber\\
    &\leq \mathbb{E}_{\boldsymbol{\theta} \sim \rho}\left[\frac{1}{N} \sum_{i=1}^N - \frac{\ln{p(\mathbf{y}_i|\mathbf{x}_i,f_{\boldsymbol{\theta}})}}{\ln{2}}\right] \\
    &\approx \frac{- \ln{p(\mathcal{D}|f_{\boldsymbol{\hat{\theta}}})} + \frac{1}{2} \sum_{l \in [L]} \mathbf{\tau}_l \trace \left(\mathbf{F}_l (\mathbf{\beta}_l \mathbf{F}_l + \mathbf{\alpha}_l \mathbf{\tilde{F}}_l)^{-1}\right)}{N \ln{2}},\nonumber
\end{align}
where $\mathbf{\tilde{F}}_l$ is from the precision matrix of the prior.  Here, we first use a second-order Taylor approximation of around $\boldsymbol{\hat{\theta}}$. Furthermore, the expectation can be converted to a trace by using the cyclic property of the trace together with the fact that the posterior is a multivariate Gaussian. The negative data log-likelihood of the optimal parameters can be computed jointly with the Fisher matrix during the LA. We denote this approximation as $\aer(\mathbf{\alpha}, \mathbf{\beta}, \mathbf{\tau})$.  On the other hand, the KL-divergence can also be computed in closed form for our prior-posterior pair, since both distributions are multivariate Gaussians. Thus, we can plug both terms into the McAllester bound~\citep{guedj2019primer} to obtain the objective
\begin{align}
    \label{eq:mc_allester_objective}
    ma(\mathbf{\alpha}, \mathbf{\beta}, \mathbf{\tau}) = \aer(\mathbf{\alpha}, \mathbf{\beta}, \mathbf{\tau}) + \sqrt{\frac{\kl(\mathbf{\alpha}, \mathbf{\beta}, \mathbf{\tau}) + \ln{\frac{2\sqrt{N}}{\varepsilon}}}{2 N}}
\end{align}
where we write $\kl(\mathbf{\alpha}, \mathbf{\beta}, \mathbf{\tau})=\mathbb{KL}(\rho\|\pi)$ for the KL-divergence to emphasize the dependence on the scales. Similarly for the Catoni bound~\citep{catoni2007pac}, we obtain
\begin{align}
    \label{eq:catoni_objective}
    ca(\mathbf{\alpha}, \mathbf{\beta}, \mathbf{\tau}) = \inf_{c > 0} \frac{1 - \exp(-c \aer(\mathbf{\alpha}, \mathbf{\beta}, \mathbf{\tau}) - \frac{\kl(\mathbf{\alpha}, \mathbf{\beta}, \mathbf{\tau}) - \ln{\varepsilon}}{N})}{1 - \exp(-c)}.
\end{align}
Overall, these objectives can be evaluated and minimized without using any data samples. As opposed to the cross-validated marginal likelihood or other forms of grid searching, we (a) obtain generalization bounds and analysis, (b) do not need a separate validation set, (c) can find multiple hyperparameters, i.e., scale better with the dimensionality of the considered hyperparameters due to the differentiability of the proposed objectives, and (d) the complexity of the optimization is independent of the data set size and the complexity of the forward pass. The last point is because we only use the precomputed terms from the LA. A full derivation of our objective is given in \cref{appendix:derivations}.

\subsection{Bayesian Progressive Neural Networks}
\label{sec:method:bayespnn}

\begin{figure}
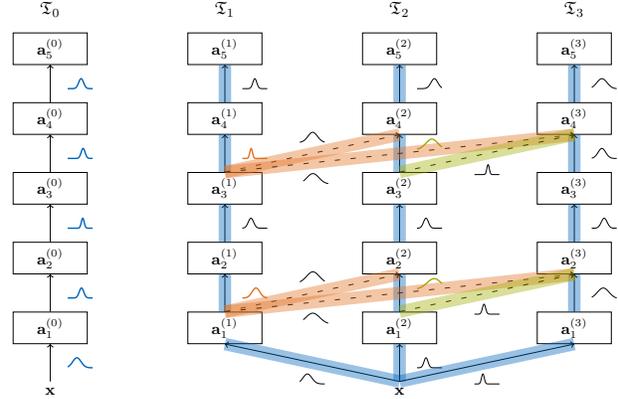

    \centering
    \includestandalone[width=\columnwidth]{tikz/method/pbnn}
    \caption{Illustration of our continual learning architecture. BPNNs use the empirically learned prior from task $\mathfrak{T}_0$ (blue) in all columns. The prior for the lateral connections is the posterior from the lateral connection (orange and green). Best viewed in color.}
    \label{fig:pbnn}
    \vspace*{-5mm}
\end{figure}

\begin{figure*}
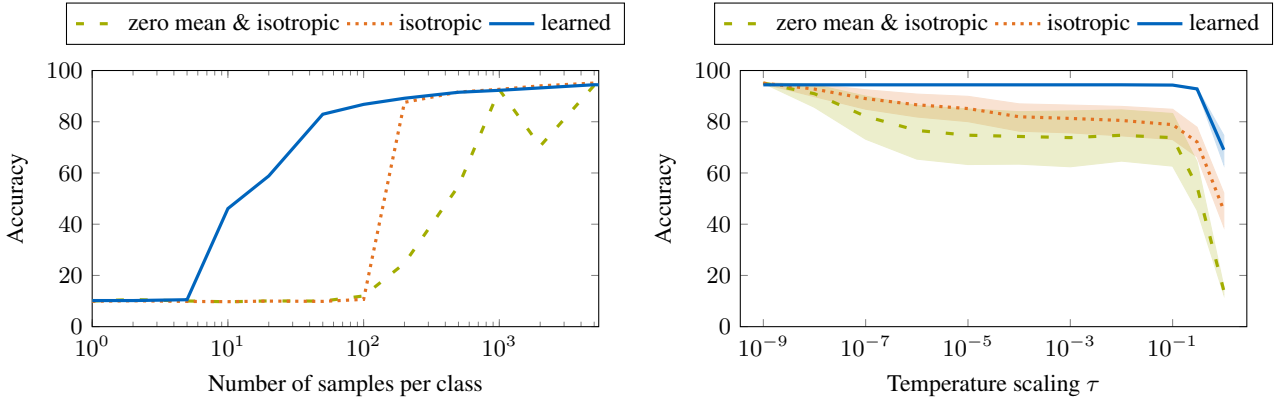

    \centering
    \begin{subfigure}[b]{\textwidth}
        \begin{center}
        \begin{scriptsize}
        \begin{sc}
         \centering
         \caption{PAC-Bayes bounds on the NotMNIST data-set using five different seeds. The ablations, namely baseline (zero mean isotropic Gaussian), grid search, learned prior (our method without PAC-Bayes objectives), curvature scaling, and frequentist projection, validate the design of our approach by each step improving the PAC-Bayes bounds, leading to a non-vacuous bound when all methods are combined.}
         \label{tab:results_each_contribution}
         \begin{tabular}{llllll}
            \toprule
            Method & Baseline & + grid Search & + learned prior & + curvature scaling & + frequentist projection \\
            \midrule
            Catoni bound & $0.999 \pm 0.001$ & $0.990 \pm 0.003$ & $0.978 \pm 0.001$ & $0.925 \pm 0.010$ & $\mathbf{0.885 \pm 0.015}$ \\
            McAllester bound & $2.185 \pm 0.557$ & $1.489 \pm 0.041$ & $1.347 \pm 0.004$ & $1.098 \pm 0.026$ & $\mathbf{1.006 \pm 0.031}$ \\
             \bottomrule
         \end{tabular} 
         \end{sc}
        \end{scriptsize}
        \end{center}
        % \vskip -0.1in
     \end{subfigure}
     \vskip 0.05in
     \hfill
     \begin{subfigure}[b]{0.485\textwidth}
         \centering
         \includestandalone{tikz/results/results_small_data_1e-5_1e-5}
         \caption{The learned prior (blue, solid) needs a magnitude fewer data to have the same accuracy as an isotropic prior around zero (green, dashed) or around the pre-trained weights (orange, dotted). Note that we use a log-log scale.}
         \label{fig:results_small_data}
     \end{subfigure}
     \hfill
     \begin{subfigure}[b]{0.485\textwidth}
         \centering
         \includestandalone{tikz/results/results_cold_posterior_1e-5}
         \caption{The cold posterior effect is more prominent for the isotropic priors (zero mean: green, dashed; pre-trained mean: orange, dashed) than for the learned prior (blue, solid). This means that the temperature scale can be larger without degrading accuracy.}
         \label{fig:results_cold_posterior}
     \end{subfigure}
    \caption{The results of the ablation studies. The results show that the combination of our approaches can lead to non-vacuous bounds (a) and that the learned prior is more data-efficient (b) and needs less temperature scaling (c) than isotropic priors.}
    \label{fig:results_ablations}
\end{figure*}

Having the essentials of learning BNN priors with generalization guarantees, we now present our extension to continual learning, which shows the versatility of our prior learning method. Here, we use so-called Progressive Neural Networks (PNNs) \citep{rusu2016progressive}, where a new neural network (column) is added for each new incoming task, and features from previous columns are also added for positive transfer between each task. Thus, PNNs are immune to catastrophic forgetting at the cost of an increase in memory, while being also applicable in transfer learning more generally \citep{wang2017how, rusu2017sim}. As such, we extend the set-up in Section~\ref{sec:method:empirical:prior} by sequentially considering T tasks: $\mathfrak{T}_1, \dots, \mathfrak{T}_T$. Moreover, to keep the generality of our method, an additional task $\mathfrak{T}_0$ is defined for the priors.

The idea behind our Bayesian re-interpretation of PNNs, dubbed as BPNNs for Bayesian PNNs, is as follows. First, the BNN posteriors are learned at task $\mathfrak{T}_0$ (depicted in \cref{fig:pbnn}). Then, for an incoming sequence of tasks, the BNN priors are specified from the BNN posteriors from $\mathfrak{T}_0$. The proposed methods of the sums-of-Kronecker-products and PAC-Bayes bounds are used. For the lateral connections, the BNN posterior from which the lateral connection originates is used. This ensures that the weight distribution from the prior already produces reasonable features given the activations from the previous column. Altogether, the resulting architecture accounts for model uncertainty, generalization, and resiliency to catastrophic forgetting. All these properties are desirable to have within one unified framework. We note that BPNN is in line with our PAC-Bayes theory, i.e., we increase the complexity without increasing the PAC-Bayes bounds. This is because we can reuse features from previous columns even though they do not contribute to the bounds due to their a-priori fixed weight distribution. In \cref{appendix:bpnn_overview,appendix:derivations}, we provide more details such as its full derivations, and its training and testing procedures. 

%% file: tikz/method/pbnn.tex
\begin{tikzpicture}[font=\footnotesize,
    node distance=.75cm and 2cm,
    squarednode/.style={rectangle, align=center, draw=black, text width=1.25cm, minimum height=.5cm}
    ]
    %Nodes
    \node               (input)                                 {$\mathbf{x}$};
    
    %column 2
    \node[squarednode]  (h21)           [above=of input]        {$\mathbf{a}^{(2)}_1$};
    \node[squarednode]  (h22)           [above=of h21]          {$\mathbf{a}^{(2)}_2$};
    \node[squarednode]  (h23)           [above=of h22]          {$\mathbf{a}^{(2)}_3$};
    \node[squarednode]  (h24)           [above=of h23]          {$\mathbf{a}^{(2)}_4$};
    \node[squarednode]  (h25)           [above=of h24]          {$\mathbf{a}^{(2)}_5$};
    \node               (t2)            [above=0.2cm of h25]    {\normalsize$\mathfrak{T}_2$};
    
    %column 1
    \node[squarednode]  (h11)           [left=of h21]           {$\mathbf{a}^{(1)}_1$};
    \node[squarednode]  (h12)           [left=of h22]          {$\mathbf{a}^{(1)}_2$};
    \node[squarednode]  (h13)           [left=of h23]          {$\mathbf{a}^{(1)}_3$};
    \node[squarednode]  (h14)           [left=of h24]          {$\mathbf{a}^{(1)}_4$};
    \node[squarednode]  (h15)           [left=of h25]          {$\mathbf{a}^{(1)}_5$};
    \node               (t1)            [above=0.2cm of h15]    {\normalsize$\mathfrak{T}_1$};

    %column 0
    \node[squarednode]  (h01)           [left=2cm of h11]           {$\mathbf{a}^{(0)}_1$};
    \node[squarednode]  (h02)           [left=2cm of h12]          {$\mathbf{a}^{(0)}_2$};
    \node[squarednode]  (h03)           [left=2cm of h13]          {$\mathbf{a}^{(0)}_3$};
    \node[squarednode]  (h04)           [left=2cm of h14]          {$\mathbf{a}^{(0)}_4$};
    \node[squarednode]  (h05)           [left=2cm of h15]          {$\mathbf{a}^{(0)}_5$};
    \node               (t0)            [above=0.2cm of h05]    {\normalsize$\mathfrak{T}_0$};
    \node               (input0)        [below=of h01]             {$\mathbf{x}$};
    
    %column 3
    \node[squarednode]  (h31)           [right=of h21]          {$\mathbf{a}^{(3)}_1$};
    \node[squarednode]  (h32)           [right=of h22]          {$\mathbf{a}^{(3)}_2$};
    \node[squarednode]  (h33)           [right=of h23]          {$\mathbf{a}^{(3)}_3$};
    \node[squarednode]  (h34)           [right=of h24]          {$\mathbf{a}^{(3)}_4$};
    \node[squarednode]  (h35)           [right=of h25]          {$\mathbf{a}^{(3)}_5$};
    \node               (t3)            [above=0.2cm of h35]    {\normalsize$\mathfrak{T}_3$};
    
    %Lines
    
    %column 0
    \draw[->] (input0.north) -- (h01.south) node[midway, right] (w21) {$\phantom{\mathbf{W}^{(1, 1)}_1}$};
    \draw [xscale=.5, yscale=.2, color=TUMBlue, thick, shift=(w21.center)] plot file {data/normal_21.table};
    \draw[->] (h01.north) -- (h02.south) node[midway, right] (w22) {$\phantom{\mathbf{W}^{(1, 1)}_2}$};
    \draw [xscale=.5, yscale=.2, color=TUMBlue, thick, shift=(w22.center)] plot file {data/normal_22.table};
    \draw[->] (h02.north) -- (h03.south) node[midway, right] (w23) {$\phantom{\mathbf{W}^{(1, 1)}_3}$};
    \draw [xscale=.5, yscale=.2, color=TUMBlue, thick, shift=(w23.center)] plot file {data/normal_23.table};
    \draw[->] (h03.north) -- (h04.south) node[midway, right] (w24) {$\phantom{\mathbf{W}^{(1, 1)}_4}$};
    \draw [xscale=.5, yscale=.2, color=TUMBlue, thick, shift=(w24.center)] plot file {data/normal_24.table};
    \draw[->] (h04.north) -- (h05.south) node[midway, right] (w25) {$\phantom{\mathbf{W}^{(1, 1)}_5}$};
    \draw [xscale=.5, yscale=.2, color=TUMBlue, thick, shift=(w25.center)] plot file {data/normal_25.table};

    %column 1
    \draw[->] (input.north) -- (h11.south) node[midway, below] (w0) {$\phantom{\mathbf{W}^{(1, 1)}_1}$};
    \draw[line width=.25cm, TUMBlue, opacity=.4] (input.north) -- (h11.south);
    \draw [xscale=.5, yscale=.2, shift=(w0.center)] plot file {data/normal_0.table};
    \draw[->] (h11.north) -- (h12.south) node[midway, right] (w1) {$\phantom{\mathbf{W}^{(1, 1)}_2}$};
    \draw[line width=.25cm, TUMBlue, opacity=.4] (h11.north) -- (h12.south);
    \draw [xscale=.5, yscale=.2, color=TUMAccentOrange, thick, shift=(w1.center)] plot file {data/normal_1.table};
    \draw[->] (h12.north) -- (h13.south) node[midway, right] (w2) {$\phantom{\mathbf{W}^{(1, 1)}_3}$};
    \draw[line width=.25cm, TUMBlue, opacity=.4] (h12.north) -- (h13.south);
    \draw [xscale=.5, yscale=.2, shift=(w2.center)] plot file {data/normal_2.table};
    \draw[->] (h13.north) -- (h14.south) node[midway, right] (w3) {$\phantom{\mathbf{W}^{(1, 1)}_4}$};
    \draw[line width=.25cm, TUMBlue, opacity=.4] (h13.north) -- (h14.south);
    \draw [xscale=.5, yscale=.2, color=TUMAccentOrange, thick, shift=(w3.center)] plot file {data/normal_3.table};
    \draw[->] (h14.north) -- (h15.south) node[midway, right] (w6) {$\phantom{\mathbf{W}^{(1, 1)}_5}$};
    \draw[line width=.25cm, TUMBlue, opacity=.4] (h14.north) -- (h15.south);
    \draw [xscale=.5, yscale=.2, shift=(w6.center)] plot file {data/normal_6.table};
    
    %column 2
    %intra
    \draw[->] (input.north) -- (h21.south) node[midway, right] (w5) {$\phantom{\mathbf{W}^{(2, 2)}_1}$};
    \draw[line width=.25cm, TUMBlue, opacity=.4] (input.north) -- (h21.south);
    \draw [xscale=.5, yscale=.2, shift=(w5.center)] plot file {data/normal_5.table};
    \draw[->] (h21.north) -- (h22.south) node[near end, right] (w4) {$\phantom{\mathbf{W}^{(2, 2)}_2}$};
    \draw[line width=.25cm, TUMBlue, opacity=.4] (h21.north) -- (h22.south);
    \draw [xscale=.5, yscale=.2, color=TUMAccentGreen, thick, shift=(w4.center)] plot file {data/normal_4.table};
    \draw[->] (h22.north) -- (h23.south) node[midway, right] (w7) {$\phantom{\mathbf{W}^{(2, 2)}_3}$};
    \draw[line width=.25cm, TUMBlue, opacity=.4] (h22.north) -- (h23.south);
    \draw [xscale=.5, yscale=.2, shift=(w7.center)] plot file {data/normal_7.table};
    \draw[->] (h23.north) -- (h24.south) node[near end, right] (w12) {$\phantom{\mathbf{W}^{(2, 2)}_4}$};
    \draw[line width=.25cm, TUMBlue, opacity=.4] (h23.north) -- (h24.south);
    \draw [xscale=.5, yscale=.2, color=TUMAccentGreen, thick, shift=(w12.center)] plot file {data/normal_12.table};
    \draw[->] (h24.north) -- (h25.south) node[midway, right] (w9) {$\phantom{\mathbf{W}^{(2, 2)}_5}$};
    \draw[line width=.25cm, TUMBlue, opacity=.4] (h24.north) -- (h25.south);
    \draw [xscale=.5, yscale=.2, shift=(w9.center)] plot file {data/normal_9.table};
    %inter
    \draw[->, loosely dashed] (h11.north) -- (h22.south) node[midway, above] (w10) {$\phantom{\mathbf{W}^{(1, 2)}_2}$};
    \draw[line width=.25cm, TUMAccentOrange, opacity=.4] (h11.north) -- (h22.south);
    \draw [xscale=.5, yscale=.2, shift=(w10.center)] plot file {data/normal_10.table};
    \draw[->, loosely dashed] (h13.north) -- (h24.south) node[midway, above] (w11) {$\phantom{\mathbf{W}^{(1, 2)}_4}$};
    \draw[line width=.25cm, TUMAccentOrange, opacity=.4] (h13.north) -- (h24.south);
    \draw [xscale=.5, yscale=.2, shift=(w11.center)] plot file {data/normal_11.table};
    
    %column 3
    %intra
    \draw[->] (input.north) -- (h31.south) node[midway, below] (w8) {$\phantom{\mathbf{W}^{(3, 3)}_1}$};
    \draw[line width=.25cm, TUMBlue, opacity=.4] (input.north) -- (h31.south);
    \draw [xscale=.5, yscale=.2, shift=(w8.center)] plot file {data/normal_8.table};
    \draw[->] (h31.north) -- (h32.south) node[midway, right] (w13) {$\phantom{\mathbf{W}^{(3, 3)}_2}$};
    \draw[line width=.25cm, TUMBlue, opacity=.4] (h31.north) -- (h32.south);
    \draw [xscale=.5, yscale=.2, shift=(w13.center)] plot file {data/normal_13.table};
    \draw[->] (h32.north) -- (h33.south) node[midway, right] (w14) {$\phantom{\mathbf{W}^{(3, 3)}_3}$};
    \draw[line width=.25cm, TUMBlue, opacity=.4] (h32.north) -- (h33.south);
    \draw [xscale=.5, yscale=.2, shift=(w14.center)] plot file {data/normal_14.table};
    \draw[->] (h33.north) -- (h34.south) node[midway, right] (w15) {$\phantom{\mathbf{W}^{(3, 3)}_4}$};
    \draw[line width=.25cm, TUMBlue, opacity=.4] (h33.north) -- (h34.south);
    \draw [xscale=.5, yscale=.2, shift=(w15.center)] plot file {data/normal_15.table};
    \draw[->] (h34.north) -- (h35.south) node[midway, right] (w16) {$\phantom{\mathbf{W}^{(3, 3)}_5}$};
    \draw[line width=.25cm, TUMBlue, opacity=.4] (h34.north) -- (h35.south);
    \draw [xscale=.5, yscale=.2, shift=(w16.center)] plot file {data/normal_16.table};
    %inter
    \draw[->, loosely dashed] (h11.north) -- (h32.south) node[pos=0.26, below] (w17) {$\phantom{\mathbf{W}^{(1, 3)}_2}$};
    \draw[line width=.25cm, TUMAccentOrange, opacity=.4] (h11.north) -- (h32.south);
    \draw [xscale=.5, yscale=.2, shift=(w17.center)] plot file {data/normal_17.table};
    \draw[->, loosely dashed] (h13.north) -- (h34.south) node[pos=0.26, below] (w18) {$\phantom{\mathbf{W}^{(1, 3)}_4}$};
    \draw[line width=.25cm, TUMAccentOrange, opacity=.4] (h13.north) -- (h34.south);
    \draw [xscale=.5, yscale=.2, shift=(w18.center)] plot file {data/normal_18.table};
    \draw[->, loosely dashed] (h21.north) -- (h32.south) node[midway, below] (w19) {$\phantom{\mathbf{W}^{(2, 3)}_2}$};
    \draw[line width=.25cm, TUMAccentGreen, opacity=.4] (h21.north) -- (h32.south);
    \draw [xscale=.5, yscale=.2, shift=(w19.center)] plot file {data/normal_19.table};
    \draw[->, loosely dashed] (h23.north) -- (h34.south) node[midway, below] (w20) {$\phantom{\mathbf{W}^{(2, 3)}_4}$};
    \draw[line width=.25cm, TUMAccentGreen, opacity=.4] (h23.north) -- (h34.south);
    \draw [xscale=.5, yscale=.2, shift=(w20.center)] plot file {data/normal_20.table};
\end{tikzpicture}

%% file: tikz/results/results_small_data_1e-5_1e-5.tex
\begin{tikzpicture}
\begin{axis}[
footnotesize,
width=\textwidth,
height=0.6\textwidth,
log basis x={10},
xlabel={Number of samples per class},
xmin=1, xmax=5400,
xmode=log,
ylabel={Accuracy},
ymin=0, ymax=100,
legend style={at={(0.5,1.1)},anchor=south},
legend columns=3,
]
\addplot [very thick, TUMAccentGreen, loosely dashed]
table {%
1 10.0200004577637
2 10.5700006484985
5 10.0299997329712
10 9.71000003814697
20 10.0499992370605
50 9.97999858856201
100 11.9899988174438
200 24.7300071716309
500 54.9000091552734
1000 92.5300064086914
2000 70.4700088500977
5000 94.2099990844727
5400 95.010009765625
};
\addlegendentry{zero mean \& isotropic}
\addplot [very thick, TUMAccentOrange, dotted]
table {%
1 9.90000152587891
2 10.1100015640259
5 9.85000228881836
10 9.76000118255615
20 9.97999858856201
50 9.8700008392334
100 10.6499996185303
200 87.5700073242188
500 91.5999984741211
1000 92.650016784668
2000 94.0199966430664
5400 95.1900024414062
};
\addlegendentry{isotropic}
\addplot [very thick, TUMBlue]
table {%
1 10.1899995803833
2 10.1899995803833
5 10.520001411438
10 46.1000137329102
20 58.8299980163574
50 82.9300003051758
100 86.7800140380859
200 89.1700057983398
500 91.5100021362305
1000 92.2700119018555
5000 94.4499969482422
5400 94.4300155639648
};
\addlegendentry{learned}
\end{axis}
\end{tikzpicture}

%% file: tikz/results/results_cold_posterior_1e-5.tex
\begin{tikzpicture}
\begin{axis}[
footnotesize,
width=\textwidth,
height=0.6\textwidth,
log basis x={10},
xlabel={Temperature scaling $\tau$},
xmin=3.54813389233576e-10, xmax=2.81838293126445,
xmode=log,
max space between ticks=40pt,
ylabel={Accuracy},
ymin=0, ymax=100,
legend style={at={(0.5,1.1)},anchor=south},
legend columns=3,
]
\path [draw=TUMAccentOrange, fill=TUMAccentOrange, opacity=0.2]
(axis cs:1e-09,95.1976033333333)
--(axis cs:1e-09,94.79573)
--(axis cs:1e-08,89.7981)
--(axis cs:1e-07,84.8921)
--(axis cs:1e-06,81.86079)
--(axis cs:1e-05,80.0416033333333)
--(axis cs:0.0001,76.3112133333333)
--(axis cs:0.001,75.6591033333333)
--(axis cs:0.01,74.4946633333333)
--(axis cs:0.1,73.0455333333334)
--(axis cs:0.3,66.2803033333333)
--(axis cs:1,38.99312)
--(axis cs:1,52.18532)
--(axis cs:1,52.18532)
--(axis cs:0.3,77.85027)
--(axis cs:0.1,84.8615866666667)
--(axis cs:0.01,85.9926266666667)
--(axis cs:0.001,86.51298)
--(axis cs:0.0001,86.9821933333333)
--(axis cs:1e-05,89.9201066666667)
--(axis cs:1e-06,90.7989466666667)
--(axis cs:1e-07,92.5044233333333)
--(axis cs:1e-08,94.6343433333333)
--(axis cs:1e-09,95.1976033333333)
--cycle;
\path [draw=TUMAccentGreen, fill=TUMAccentGreen, opacity=0.2]
(axis cs:1e-09,95.2015865384616)
--(axis cs:1e-09,95.0076634615384)
--(axis cs:1e-08,85.6179903846154)
--(axis cs:1e-07,73.1764423076923)
--(axis cs:1e-06,65.4083173076923)
--(axis cs:1e-05,63.3000480769231)
--(axis cs:0.0001,63.4484423076923)
--(axis cs:0.001,62.4258365384615)
--(axis cs:0.01,64.6311826923077)
--(axis cs:0.1,62.6682980769231)
--(axis cs:0.3,45.4736442307692)
--(axis cs:1,12.4982403846154)
--(axis cs:1,15.7435576923077)
--(axis cs:1,15.7435576923077)
--(axis cs:0.3,64.4124807692308)
--(axis cs:0.1,83.2133365384615)
--(axis cs:0.01,84.6215480769231)
--(axis cs:0.001,84.2296346153846)
--(axis cs:0.0001,83.9238076923077)
--(axis cs:1e-05,85.0125769230769)
--(axis cs:1e-06,86.3350961538462)
--(axis cs:1e-07,89.9970480769231)
--(axis cs:1e-08,94.5351730769231)
--(axis cs:1e-09,95.2015865384616)
--cycle;
\path [draw=TUMBlue, fill=TUMBlue, opacity=0.2]
(axis cs:1e-09,94.4430666666667)
--(axis cs:1e-09,94.37892)
--(axis cs:1e-08,94.3810633333333)
--(axis cs:1e-07,94.3769166666667)
--(axis cs:1e-06,94.3755966666667)
--(axis cs:1e-05,94.3798666666667)
--(axis cs:0.0001,94.3731866666667)
--(axis cs:0.001,94.3785166666667)
--(axis cs:0.01,94.3875933333334)
--(axis cs:0.1,94.2801333333333)
--(axis cs:0.3,92.4445866666667)
--(axis cs:1,63.2573033333333)
--(axis cs:1,74.69537)
--(axis cs:1,74.69537)
--(axis cs:0.3,93.17546)
--(axis cs:0.1,94.34468)
--(axis cs:0.01,94.4546833333333)
--(axis cs:0.001,94.44467)
--(axis cs:0.0001,94.44161)
--(axis cs:1e-05,94.4449533333333)
--(axis cs:1e-06,94.4424066666667)
--(axis cs:1e-07,94.4402733333333)
--(axis cs:1e-08,94.4422733333333)
--(axis cs:1e-09,94.4430666666667)
--cycle;
\addplot [very thick, TUMAccentGreen, loosely dashed]
table {%
9.99999971718069e-10 95.114616394043
9.9999777347648e-09 90.9373092651367
9.9999887481772e-08 82.0988464355469
9.99998860606865e-07 76.6053848266602
9.99998883344233e-06 74.7415390014648
9.99999974737875e-05 74.2584609985352
0.000999999465420842 73.8057708740234
0.00999999418854713 74.7180786132812
0.100000001490116 73.7850036621094
0.299999982118607 54.7165374755859
1 14.0569229125977
};
\addlegendentry{zero mean \& isotropic}
\addplot [very thick, TUMAccentOrange, dotted]
table {%
9.99999971718069e-10 95.0550689697266
9.9999777347648e-09 92.7482681274414
9.9999887481772e-08 89.0274658203125
9.99998860606865e-07 86.6202697753906
9.99998883344233e-06 85.1470642089844
9.99999974737875e-05 81.9434661865234
0.000999999465420842 81.2945327758789
0.00999999418854713 80.5098648071289
0.100000001490116 78.8677368164062
0.299999982118607 72.1231994628906
1 45.4121322631836
};
\addlegendentry{isotropic}
\addplot [very thick, TUMBlue]
table {%
9.99999971718069e-10 94.4110641479492
9.99998883344233e-06 94.4113311767578
9.99999974737875e-05 94.4069366455078
0.00999999418854713 94.4205322265625
0.100000001490116 94.3122634887695
0.299999982118607 92.8201370239258
1 69.0609359741211
};
\addlegendentry{learned}
\end{axis}
\end{tikzpicture}

%% file: sections/relatedworks.tex
\section{Related Work}

There are different work streams related to this paper. The primary area is on BNN priors while we also contribute to Bayesian continual learning and PAC-Bayes theory.

\begin{figure*}
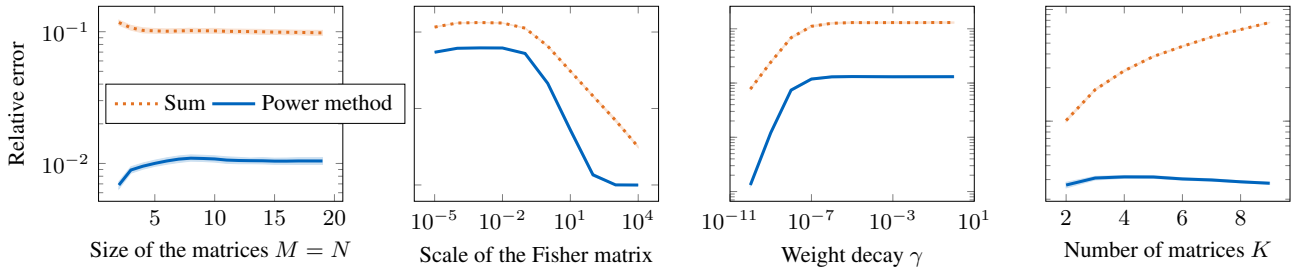

    \centering
    \begin{subfigure}[b]{\textwidth}
        \begin{center}
        \begin{scriptsize}
        \begin{sc}
         \centering
         \caption{Relative Frobenius error as a function of iteration. The convergence was reached in the second iteration.}
         \label{tab:results_convergence:kronecker}
         \begin{tabular}{llllllll}
            \toprule
            Iteration & 1 & 2 & 3 & 4 & 5 & 6 & 7 \\
            \midrule
            Error & $0.793 \pm 0.010$ & $0.049 \pm 0.001$ & $0.049 \pm 0.001$ & $0.049 \pm 0.001$ & $0.049 \pm 0.001$ & $0.049 \pm 0.001$ & $0.049 \pm 0.001$ \\
             \bottomrule
         \end{tabular} 
         \end{sc}
        \end{scriptsize}
        \end{center}
        % \vskip -0.1in
     \end{subfigure}
     \vskip 0.10in
     \hfill
     \begin{subfigure}[b]{\textwidth}
         \centering
         \includestandalone[scale=0.95]{tikz/appendix/results/skp_isotropic}
         \caption{Approximation quality on the sums-of-Kronecker products as two Kronecker factors, \ie, relative error to approximate $\sum_{k=1}^K \mathbf{L}^k \otimes \mathbf{R}^k - \mathbf{L} \otimes \mathbf{R}$. We vary the sizes of the matrices, the scale of the fisher matrix, weight decays within $(\mathbf{L} \otimes \mathbf{R} + \gamma \mathbf{I} )$, and the number of matrices for the sums-of-Kronecker products. The results show that the proposed method can be more accurate.} \label{fig:results_ablations:appendix:sum:kronecker:maintext}
     \end{subfigure}
     \vspace*{-5mm}
     %\hfill
    \caption{The results of the ablation studies. The results show the accuracy of the proposed sums-of-Kronecker product computations.}
    \label{fig:results_ablations:kronecker}
    \vspace*{-5mm}
\end{figure*}

\textbf{Bayesian Neural Networks Priors} A prior specification is a prerequisite in Bayesian modeling. However, for neural network parameters, the choice of the prior is generally unknown. So far, uninformative priors such as isotropic Gaussian have been the de-facto-standard for BNNs~\citep{fortuin2021priors}. Such uninformative priors may cause undesirable effects such as cold posteriors and worse predictive performance than standard neural networks~\citep{wenzel2020how, fortuin2021bayesian}. To this end, recent research efforts have focused on exploiting function-space~\citep{sun2018functional}, sparsity-inducing weight-space~\citep{carvalho2009handling, ghosh2018structured}, structured~\citep{louizos2016structured}, and learning-based priors~\citep{immer2021scalable, fortuin2021bayesian, wu2019deterministic}. Amongst these, we build upon learning-based priors. Learning-based priors can be an alternative to uninformative prior when no useful prior knowledge is available to encode, or when there exists relevant data and tasks to learn from \citep{fortuin2021priors}.

The idea of learning a prior distribution on a similar task is certainly not new. In this domain, Bayesian meta-learning~\citep{thrun1998learning} and their modern extensions~\citep{rothfuss2021pacoh, finn2018probabilistic} can be viewed as another form of learning the prior from data, although their focus is typically not on advancing BNNs. Empirical prior learning is closely related to our work~\citep{robbins1992empirical}. For BNNs, \citet{fortuin2021bayesian} learns the empirical weight distributions during stochastic gradient descent. \citet{wu2019deterministic} used a moment-matching approach, while \citet{immer2021scalable} exploited the so-called Laplace-Generalized-Gauss-Newton method. In \citep{krishnan2020specifying}, the mean of a Gaussian prior distribution is learned from a relevant data-set. The concurrent approaches of \citet{shwartz2022pre, tran2022plex} share a similar spirit of learning expressive priors from large-scale data-set and architectures. The former utilizes so-called the SWAG~\citep{maddox2019simple} framework with an inspiring idea of combining self-supervised learning, while the latter builds upon Frequentist batch ensembles~\citep{wen2019batchensemble}. All these works show the strong relevance of learning-based priors for the current BNNs. Inspired by the aforementioned works, our main idea is to learn scalable and structured posterior as expressive informative priors, like Bayesian foundation models from broader data and models.

\textbf{Bayesian Continual Learning} Continual learning~\citep{thrun1995lifelong} methods can be broadly divided into replay-based, regularization-based, and parameter-isolation methods~\citep{de2021continual}. We build on parameter-isolation methods, e.g. \citet{rusu2016progressive}, where different model parameters are dedicated to new tasks. Within this branch, Bayesian methods have been investigated \citep{Ebrahimi2020Uncertainty-guided,ardywibowo2022varigrow,kumar2021bayesian}, which could benefit from our work on advancing the BNN priors. \citet{rudner2022continual} shares a similar spirit of bringing the state-of-the-art BNN priors to Bayesian continual learning by relying on the function space priors~\citep{sun2018functional}. On the other hand, using learning-based expressive priors, we design a single unified framework of continual learning for generalization, uncertainty-awareness, and resiliency to catastrophic forgetting.

\textbf{Generalization Theory} In supervised learning, we can obtain models with a bound on the true expected loss from its true data-generating process. Typical early works involved Vapnik-Chervonenkis theory and Rademacher complexity~\citep{vapnik2015uniform, shalev2014understanding}. However, for high-dimensional models like neural networks, these methods often provide vacuous bounds. In recent years, PAC-Bayes theory~\citep{mcallester1999some,germain2016pac} has become an alternative method with wide applications. The seminar paper of \citep{germain2016pac} showed the connection to approximate Bayesian inference. \citet{rothfuss2021meta, rothfuss2021pacoh} devise compelling meta-learning priors for BNNs with generalization guarantees. These works form our inspiration to explore PAC-Bayes theory for learning-based BNN priors. We note that our goal is not to advance PAC-Bayes theory, but to investigate a method for scaling the BNN priors with an approximate differentiable objective for generalization.

%% file: tikz/appendix/results/skp_isotropic.tex
\begin{tikzpicture}

% power method 0.01020 \pm 0.03129
% sum 0.10167 \pm 0.22935

\begin{groupplot}[
    footnotesize,
    %width=0.5\columnwidth,
    %height=0.3\columnwidth,
    group style={group size=4 by 1, y descriptions at=edge left, x descriptions at=edge bottom},
    log basis y={10},
    ylabel={Relative error},
    ymode=log,
    max space between ticks=30pt,
    ]
    \nextgroupplot[
    legend to name={LegendAblations},
    legend style={legend columns=2}, 
    xlabel={Size of the matrices $M=N$}
    ]
        \path [draw=TUMAccentOrange, fill=TUMAccentOrange, opacity=0.2]
(axis cs:2,0.122519445456618)
--(axis cs:2,0.11217354340406)
--(axis cs:3,0.102707662526528)
--(axis cs:4,0.0979445677221596)
--(axis cs:5,0.0972459958845971)
--(axis cs:6,0.0967570388466431)
--(axis cs:7,0.0974619036010183)
--(axis cs:8,0.0973784399288036)
--(axis cs:9,0.0974543706486668)
--(axis cs:10,0.0972605563888383)
--(axis cs:11,0.0967748747399673)
--(axis cs:12,0.096346468997471)
--(axis cs:13,0.0962050045378822)
--(axis cs:14,0.0955431475890601)
--(axis cs:15,0.0953737444835598)
--(axis cs:16,0.0948585369226398)
--(axis cs:17,0.094812314351928)
--(axis cs:18,0.0940735141920699)
--(axis cs:19,0.0939019243934296)
--(axis cs:19,0.10232970281399)
--(axis cs:19,0.10232970281399)
--(axis cs:18,0.102583842945261)
--(axis cs:17,0.102685756819855)
--(axis cs:16,0.103151269420594)
--(axis cs:15,0.103497014649711)
--(axis cs:14,0.104091981312271)
--(axis cs:13,0.104281441602981)
--(axis cs:12,0.104623786705543)
--(axis cs:11,0.104783301058989)
--(axis cs:10,0.106021029059163)
--(axis cs:9,0.105616655675487)
--(axis cs:8,0.106265894355935)
--(axis cs:7,0.105410598834249)
--(axis cs:6,0.104993230786945)
--(axis cs:5,0.105688693281875)
--(axis cs:4,0.107092067446272)
--(axis cs:3,0.11168197800059)
--(axis cs:2,0.122519445456618)
--cycle;

\path [draw=TUMBlue, fill=TUMBlue, opacity=0.2]
(axis cs:2,0.00727942811941512)
--(axis cs:2,0.00642017944251745)
--(axis cs:3,0.00843523795717179)
--(axis cs:4,0.00904032316581011)
--(axis cs:5,0.00950774343177528)
--(axis cs:6,0.00986333808834376)
--(axis cs:7,0.0101713546906818)
--(axis cs:8,0.0104108406526779)
--(axis cs:9,0.0103497563831788)
--(axis cs:10,0.0102398088233816)
--(axis cs:11,0.0100186628330287)
--(axis cs:12,0.009950900166723)
--(axis cs:13,0.00995275853377395)
--(axis cs:14,0.0099247529701048)
--(axis cs:15,0.00984924474926948)
--(axis cs:16,0.00984640436627141)
--(axis cs:17,0.00982096400985531)
--(axis cs:18,0.00983534986413616)
--(axis cs:19,0.00983990933391216)
--(axis cs:19,0.0110745143261731)
--(axis cs:19,0.0110745143261731)
--(axis cs:18,0.0110707036021174)
--(axis cs:17,0.0110630455364607)
--(axis cs:16,0.0109980451321899)
--(axis cs:15,0.0109723099745768)
--(axis cs:14,0.0110753096649742)
--(axis cs:13,0.010997069245972)
--(axis cs:12,0.0111232378316728)
--(axis cs:11,0.0112369569007491)
--(axis cs:10,0.0114316356131214)
--(axis cs:9,0.0115411951166504)
--(axis cs:8,0.0116160598279252)
--(axis cs:7,0.0114046966894963)
--(axis cs:6,0.0110461652199009)
--(axis cs:5,0.0105943295377921)
--(axis cs:4,0.0100788140084504)
--(axis cs:3,0.00943863552071243)
--(axis cs:2,0.00727942811941512)
--cycle;

\addplot [very thick, TUMAccentOrange, dotted]
table {%
2 0.117359779775143
3 0.107098817825317
4 0.102358065545559
5 0.101466372609138
6 0.100831739604473
8 0.101897649466991
9 0.101597025990486
10 0.101472586393356
11 0.100654318928719
13 0.100175522267818
14 0.0998386740684509
15 0.0992229357361794
16 0.0990400537848473
18 0.0984412059187889
19 0.0979265347123146
};
%\addlegendentry{Sum}
\addplot [very thick, TUMBlue]
table {%
2 0.00685918796807528
3 0.00891984719783068
4 0.00953147374093533
5 0.0100155184045434
6 0.0104451477527618
7 0.0107878446578979
8 0.0109677398577332
9 0.0109123745933175
10 0.0108228689059615
11 0.0106102610006928
12 0.0105324434116483
13 0.010500269010663
14 0.0104835936799645
15 0.01041797734797
16 0.0104163018986583
17 0.0104519994929433
19 0.0104445368051529
};
\addlegendentry{Sum}

\addlegendentry{Power method}
%\addlegendentry{Power method}

    \nextgroupplot[
    xlabel={Scale of the Fisher matrix},
    log basis x={10},
    xmode=log
    ]
\path [draw=TUMAccentOrange, fill=TUMAccentOrange, opacity=0.2]
(axis cs:9.99999974737875e-06,0.15949104815747)
--(axis cs:9.99999974737875e-06,0.153117476699692)
--(axis cs:9.99999974737875e-05,0.221730644184837)
--(axis cs:0.00100000004749745,0.231097612618991)
--(axis cs:0.00999999977648258,0.221618860527023)
--(axis cs:0.100000001490116,0.138028521710087)
--(axis cs:1,0.0266214179559934)
--(axis cs:10,0.0027676806279262)
--(axis cs:100,0.000286831375106288)
--(axis cs:1000,3.02719686370979e-05)
--(axis cs:10000,2.8883829246082e-06)
--(axis cs:10000,3.33027983448225e-06)
--(axis cs:10000,3.33027983448225e-06)
--(axis cs:1000,3.8910745818821e-05)
--(axis cs:100,0.000335476267170331)
--(axis cs:10,0.0031008283840438)
--(axis cs:1,0.0296052220449408)
--(axis cs:0.100000001490116,0.145586506676433)
--(axis cs:0.00999999977648258,0.230251148489612)
--(axis cs:0.00100000004749745,0.239782700668611)
--(axis cs:9.99999974737875e-05,0.230071238274465)
--(axis cs:9.99999974737875e-06,0.15949104815747)
--cycle;

\path [draw=TUMBlue, fill=TUMBlue, opacity=0.2]
(axis cs:9.99999974737875e-06,0.0165306621593339)
--(axis cs:9.99999974737875e-06,0.0156219075265511)
--(axis cs:9.99999974737875e-05,0.0224325211585872)
--(axis cs:0.00100000004749745,0.0231808373290533)
--(axis cs:0.00999999977648258,0.0230477103542872)
--(axis cs:0.100000001490116,0.0139305872694702)
--(axis cs:1,0.000879822772653518)
--(axis cs:10,1.17879351739696e-05)
--(axis cs:100,2.19110369342283e-07)
--(axis cs:1000,9.97547452940425e-08)
--(axis cs:10000,9.86456361486534e-08)
--(axis cs:10000,1.00031716983231e-07)
--(axis cs:10000,1.00031716983231e-07)
--(axis cs:1000,1.01356290865702e-07)
--(axis cs:100,2.85406080057958e-07)
--(axis cs:10,1.76825641832699e-05)
--(axis cs:1,0.00106245308189574)
--(axis cs:0.100000001490116,0.0149630587176351)
--(axis cs:0.00999999977648258,0.0242769070532058)
--(axis cs:0.00100000004749745,0.0244609756343751)
--(axis cs:9.99999974737875e-05,0.0236379138783032)
--(axis cs:9.99999974737875e-06,0.0165306621593339)
--cycle;

\addplot [very thick, TUMAccentOrange, dotted]
table {%
9.99998883344233e-06 0.156207799911499
9.99999974737875e-05 0.225641176104546
0.000999999465420842 0.235772371292114
0.00999999418854713 0.225922420620918
0.100000001490116 0.141840279102325
1 0.0280060842633247
10 0.00292558851651847
100.000053405762 0.000309999100863934
1000.00054931641 3.39839862135705e-05
10000 3.10000973513525e-06
};
\addplot [very thick, TUMBlue]
table {%
9.99998883344233e-06 0.0160848349332809
9.99999974737875e-05 0.0230251085013151
0.000999999465420842 0.0238454043865204
0.00999999418854713 0.0236268881708384
0.100000001490116 0.0144133623689413
1 0.000969854823779315
10 1.44384675877518e-05
100.000053405762 2.49558013365458e-07
1000.00054931641 1.00592338014849e-07
10000 9.93396795934132e-08
};

    \nextgroupplot[
    xlabel={Weight decay $\gamma$},
    log basis x={10},
    xmode=log
    ]
\path [draw=TUMAccentOrange, fill=TUMAccentOrange, opacity=0.2]
(axis cs:1.00000001335143e-10,0.00800384564081939)
--(axis cs:1.00000001335143e-10,0.00748036735254617)
--(axis cs:9.99999971718069e-10,0.0233749677630515)
--(axis cs:9.99999993922529e-09,0.0661440773135691)
--(axis cs:1.0000000116861e-07,0.107579103557585)
--(axis cs:9.99999997475243e-07,0.122682268287629)
--(axis cs:9.99999974737875e-06,0.126492794872407)
--(axis cs:9.99999974737875e-05,0.126154131839957)
--(axis cs:0.00100000004749745,0.126219770078045)
--(axis cs:0.00999999977648258,0.125812794531712)
--(axis cs:0.100000001490116,0.126845518078519)
--(axis cs:1,0.126743559755679)
--(axis cs:1,0.133830067244778)
--(axis cs:1,0.133830067244778)
--(axis cs:0.100000001490116,0.134500472906008)
--(axis cs:0.00999999977648258,0.13337201661712)
--(axis cs:0.00100000004749745,0.133913409809404)
--(axis cs:9.99999974737875e-05,0.133413637672231)
--(axis cs:9.99999974737875e-06,0.133861408031326)
--(axis cs:9.99999997475243e-07,0.129946500371253)
--(axis cs:1.0000000116861e-07,0.114584095151212)
--(axis cs:9.99999993922529e-09,0.0708711824254153)
--(axis cs:9.99999971718069e-10,0.025064986943017)
--(axis cs:1.00000001335143e-10,0.00800384564081939)
--cycle;

\path [draw=TUMBlue, fill=TUMBlue, opacity=0.2]
(axis cs:1.00000001335143e-10,0.000138007772286905)
--(axis cs:1.00000001335143e-10,0.000126575694019881)
--(axis cs:9.99999971718069e-10,0.00115574957699472)
--(axis cs:9.99999993922529e-09,0.0070916500617298)
--(axis cs:1.0000000116861e-07,0.0114156686148012)
--(axis cs:9.99999997475243e-07,0.0124881879860005)
--(axis cs:9.99999974737875e-06,0.0126900314337843)
--(axis cs:9.99999974737875e-05,0.0125848797075437)
--(axis cs:0.00100000004749745,0.0125456998778994)
--(axis cs:0.00999999977648258,0.0125705089965836)
--(axis cs:0.100000001490116,0.012550365063812)
--(axis cs:1,0.012580335158)
--(axis cs:1,0.0135925367383384)
--(axis cs:1,0.0135925367383384)
--(axis cs:0.100000001490116,0.0135858048266453)
--(axis cs:0.00999999977648258,0.0136421512112014)
--(axis cs:0.00100000004749745,0.0135529175736348)
--(axis cs:9.99999974737875e-05,0.0136578802637709)
--(axis cs:9.99999974737875e-06,0.0137090225558173)
--(axis cs:9.99999997475243e-07,0.0135257228799567)
--(axis cs:1.0000000116861e-07,0.0122832811125334)
--(axis cs:9.99999993922529e-09,0.0077446496317379)
--(axis cs:9.99999971718069e-10,0.00126018585477271)
--(axis cs:1.00000001335143e-10,0.000138007772286905)
--cycle;

\addplot [very thick, TUMAccentOrange, dotted]
table {%
9.99997792905383e-11 0.00773275131359696
9.99999971718069e-10 0.0241776145994663
9.9999777347648e-09 0.0684606060385704
9.9999887481772e-08 0.111149370670319
9.99998860606865e-07 0.126350447535515
9.99998883344233e-06 0.130162164568901
9.99999974737875e-05 0.129636883735657
0.000999999465420842 0.130178287625313
0.00999999418854713 0.129693582653999
0.100000001490116 0.130434811115265
1 0.130352601408958
};
\addplot [very thick, TUMBlue]
table {%
9.99997792905383e-11 0.000131806242279708
9.99999971718069e-10 0.00120960350614041
9.9999777347648e-09 0.00740867806598544
9.9999887481772e-08 0.0118495272472501
9.99998860606865e-07 0.0130121950060129
9.99998883344233e-06 0.0131593504920602
0.000999999465420842 0.0130519764497876
0.100000001490116 0.0130749037489295
1 0.0130970738828182
};

    \nextgroupplot[
    %xlabel={Weight decay $\gamma$},
    %log basis x={10},
    %xmode=log
    %legend pos=middle east,
    log basis y={10},
    xlabel={Number of matrices $K$},
    ymode=log,
    ]

\path [draw=TUMAccentOrange, fill=TUMAccentOrange, opacity=0.2]
(axis cs:2,1.03864355260134)
--(axis cs:2,0.9853386426121)
--(axis cs:3,1.87895284417272)
--(axis cs:4,2.78600546628237)
--(axis cs:5,3.75371983194351)
--(axis cs:6,4.63441372185946)
--(axis cs:7,5.62662349140644)
--(axis cs:8,6.56534879124165)
--(axis cs:9,7.60414712679386)
--(axis cs:9,7.69065375852585)
--(axis cs:9,7.69065375852585)
--(axis cs:8,6.65625695133209)
--(axis cs:7,5.72423603761196)
--(axis cs:6,4.72003535586596)
--(axis cs:5,3.83988089162111)
--(axis cs:4,2.87205077725649)
--(axis cs:3,1.94894404053688)
--(axis cs:2,1.03864355260134)
--cycle;

\path [draw=TUMBlue, fill=TUMBlue, opacity=0.2]
(axis cs:2,0.283380622072145)
--(axis cs:2,0.252634905166924)
--(axis cs:3,0.297439105752856)
--(axis cs:4,0.307871565226465)
--(axis cs:5,0.307534970406443)
--(axis cs:6,0.296102681148797)
--(axis cs:7,0.290397300831974)
--(axis cs:8,0.280523322787136)
--(axis cs:9,0.272456125639379)
--(axis cs:9,0.284319386590272)
--(axis cs:9,0.284319386590272)
--(axis cs:8,0.292414742611349)
--(axis cs:7,0.304976718049496)
--(axis cs:6,0.311138603907078)
--(axis cs:5,0.325094482898712)
--(axis cs:4,0.32630025895685)
--(axis cs:3,0.320748204290867)
--(axis cs:2,0.283380622072145)
--cycle;

\addplot [very thick, TUMAccentOrange, dotted]
table {%
2 1.01334500312805
3 1.91478598117828
4 2.82916951179504
5 3.79627466201782
6 4.67773628234863
7 5.67293787002563
8 6.61105346679688
9 7.64502191543579
};
\addplot [very thick, TUMBlue]
table {%
2 0.267955362796783
3 0.309673100709915
4 0.316953808069229
5 0.316233664751053
6 0.30384173989296
7 0.297444760799408
8 0.28637620806694
9 0.277988106012344
};

\end{groupplot}

    \path (group c1r1.east) -- node[left=-1cm]{\pgfplotslegendfromname{LegendAblations}} (group c1r1.east);
\end{tikzpicture}
\noindent

%% file: sections/results.tex
\section{Results}
\label{sec:results}

The goal of the experiments is to investigate whether our approach provides generalization and calibrated uncertainty estimates. To this end, in addition to ablation studies on the presented algorithm, we show its utility for continual learning and uncertainty estimation. Implementation details are presented in Appendix~\ref{appendix:implementation:details}. The code is released at \url{https://github.com/DLR-RM/BPNN}.

\begin{table*}
% \vskip 0.15in
    \begin{center}
    \begin{scriptsize}
    \begin{sc}
     \centering
     \caption{Continual learning experiment. BPNN with a learned prior (our method) improves the averaged accuracy over all tasks compared to using an isotropic prior around zero or around a learned mean. Our method also improves over PNN with and without MC Dropout. Following \citet{denninger2018persistent}, each column represents incoming continual learning tasks, where we receive new objects.}
     \label{tab:results_continual learning}
    \begin{tabular}{llllllll}
        \toprule
         & Other & Banana & Coffee Mug & Stapler & Flashlight & Apple & Average \\
         \midrule
        PNN (weight decay $10 ^{-3}$) & $95.7 \pm 0.4$ &  $98.5 \pm 2.3$ &  $\mathbf{100.0 \pm 0.0}$ &  $91.7 \pm 1.0$ &   $93.8 \pm 9.1$ &  $93.3 \pm 4.7$ &  $95.5 \pm 2.9$ \\
        PNN (weight decay $10 ^{-5}$) & $\mathbf{96.7 \pm 0.3}$ &  $\underline{99.0 \pm 1.2}$ &  $\mathbf{100.0 \pm 0.0}$ &  $90.7 \pm 2.3$ &   $99.7 \pm 0.4$ &  $\underline{94.1 \pm 3.2}$ &  $\underline{96.7 \pm 1.2}$ \\
        MC Dropout & $\underline{96.3 \pm 0.1}$ &  $\mathbf{99.3 \pm 0.9}$ &  $\mathbf{100.0 \pm 0.0}$ &  $92.7 \pm 0.8$ &   $\underline{99.8 \pm 0.4}$ &  $90.4 \pm 5.1$ &  $96.4 \pm 1.2$ \\
        zero mean \& isotropic & $\underline{96.3 \pm 0.3}$ &  $95.5 \pm 4.2$ &  $\mathbf{100.0 \pm 0.1}$ &  $91.9 \pm 1.7$ &  $\mathbf{100.0 \pm 0.1}$ &  $87.7 \pm 7.8$ &  $95.2 \pm 2.4$ \\
        isotropic & $96.1 \pm 0.3$ &  $98.7 \pm 1.0$ &  $\mathbf{100.0 \pm 0.0}$ &  $\underline{93.1 \pm 0.6}$ &  $\mathbf{100.0 \pm 0.0}$ &  $87.8 \pm 6.6$ &  $96.0 \pm 1.4$ \\
        learned & $96.2 \pm 0.2$ &  $98.4 \pm 1.6$ &  $\mathbf{100.0 \pm 0.0}$ &  $\mathbf{93.9 \pm 0.9}$ &  $\mathbf{100.0 \pm 0.0}$ &  $\mathbf{95.1 \pm 4.7}$ &  $\mathbf{97.3 \pm 1.2}$ \\
        \bottomrule
    \end{tabular}
     \end{sc}
    \end{scriptsize}
    \end{center}
    \vskip -0.2in
 \end{table*}

\subsection{Ablation Studies}
\label{sec:results:ablations:main}

Our method is to improve generalization in the LA and also to provide an informative prior for the posterior inference. Therefore, in the ablation studies, we want to investigate the impact of each of our methods on the generalization bounds. Moreover, we want to study the learned prior in the small-data regime and see if it can mitigate the cold posterior effect, \ie phenomena in BNNs where the performance improves by cooling the posterior with a temperature of less than one~\citep{wenzel2020how}. This effect highlights the discrepancy of current BNNs, where no posterior tempering should be theoretically needed. For this, we use a LeNet-5 architecture~\citep{lecun1989backpropagation}, learn the prior on MNIST~\citep{lecun1998gradient} and compute the posterior on NotMNIST~\citep{bulatov2011notmnist}. In our experiments, we compare the performance of our learned prior against an isotropic Gaussian prior with multiple weight decays, either with a mean zero or using the pre-trained model as the mean.

In Table~\ref{tab:results_each_contribution}, one can see the contribution of each method to the final generalization bounds. For this, we report the mean and standard deviation using $5$ different seeds. We observe that learning the prior improves the generalization bounds, as the posterior can reuse some of the pre-trained weights. The curvature further controls the flexibility of each parameter, leading to a smaller KL-divergence while still providing a small error rate. In addition, we show that scaling the curvature with our approximate bounds reduces the bounds. In particular, the generalization bounds are also better than a thorough grid search. Frequentist projection further leads to a non-vacuous bound of $0.885$. 
In the small data regime, we train our method on a random subset of NotMNIST, \ie, we vary the number of available samples per class. Figure~\ref{fig:results_small_data} shows that the learned prior requires a magnitude fewer data to achieve similar accuracy compared to isotropic priors. In Appendix~\ref{appendix:results_small_data}, we further evaluate the impact of using different temperature scalings and weight decays in this experiment. We observe larger improvements when the model is more probabilistic, \ie when the temperature scaling is large, and when the weight decay is small and thus the learned prior is more dominant.
Finally, following \citet{wenzel2020how}, we examine the cold posterior effect using the learned prior in Figure~\ref{fig:results_cold_posterior}. Although the optimal value for the temperature scaling is less than one, the temperature scaling can be closer to one compared to the isotropic priors. This suggests that the cold posterior effect is reduced by using a learned prior.

Additionally, we further validate the proposed sums-of-Kronecker product computations (see Figure~\ref{fig:results_ablations:kronecker}). In Appendix~\ref{appendix:results}, further experimental results are provided, including qualitative analysis of our tractable approximate bound used to optimize the curvature scales (Appendix~\ref{appendix:results_approximate_bound}). Furthermore, results for a larger cold posterior experiment using ResNet-50~\citep{he2016deep}, ImageNet-1K~\citep{deng2009imagenet}, and CIFAR-10~\citep{krizhevsky2009learning} are presented in Appendix~\ref{appendix:results_cold_posterior}. Here, the learned prior improves the accuracy but does not reduce the cold posterior effect for all evaluated weight decays. This suggests the use-case of our method. %That is, a learned prior is effective when there exists relevant data and tasks to learn from. If this condition holds,%Overall, our results demonstrate the effectiveness of our method in improving the generalization bounds and reducing the cold posterior effect in Bayesian neural networks.
Overall, our results demonstrate the effectiveness of our method in improving the generalization bounds. Furthermore, we show that the learned prior is particularly useful for BNNs when little data is available. %While we were able to show significant improvements for the cold posterior effect in the small-scale setting, we also report the corner cases in Appendix~\ref{appendix:results}.

\subsection{Generalization in Bayesian Continual Learning}
\label{sec:results:continual:learning}

Another advantage of our prior is its connection to continual learning. In the following experiments, we, therefore, evaluate our method for continual learning tasks. To do so, we closely follow the setup of \citet{denninger2018persistent}, who introduces a practical robotics scenario. This allows us to also obtain meaningful results for practitioners. We do not use the typical MNIST setup here because we also want to test the effectiveness of the expressive prior from ImageNet~\citep{deng2009imagenet}. Each continual learning task consists of recognizing the specific object instances of the Washington RGB-D data-set (WRGBD)~\citep{lai2011large}, while only a subset of all classes is available at a given time. We neglect the depth information and split the data similar to \citet{denninger2018persistent}. The prior is learned with the ImageNet-1K~\citep{deng2009imagenet} and the ResNet-50~\citep{he2016deep} is used. We specify a total of four lateral connections, each at the downsampling $1 \times 1$ convolution of the first "bottleneck" building block of the ResNet layers. The baselines are PNNs with several weight decays and using MC dropout~\citep{gal2016dropout}. Moreover, we compare to both isotropic priors as in the ablations.

\cref{tab:results_continual learning} shows the mean and the standard deviation of the accuracy for each task of the continual learning experiment for $5$ different random seeds. We report the mean and the standard deviation of two independent runs. Overall, in our experiments, the accuracy averaged over all tasks is improved by $0.6$ percent points compared to the second best approach PNN with weight decay $10 ^{-5}$. In addition, the increase in accuracy is $1.3$ percent points greater compared to an isotropic prior, and $2.1$ percent point greater when using zero as the prior mean. Therefore, these experimental results illustrate that with our idea of expressive posterior as BNN prior, we can improve the generalization on the test set for practical tasks of robotic continual learning.

\subsection{Few-Shot Generalization and Uncertainty}
\label{sec:results:fewshot}

\begin{figure*}
    \centering
    \includegraphics[width=\textwidth]{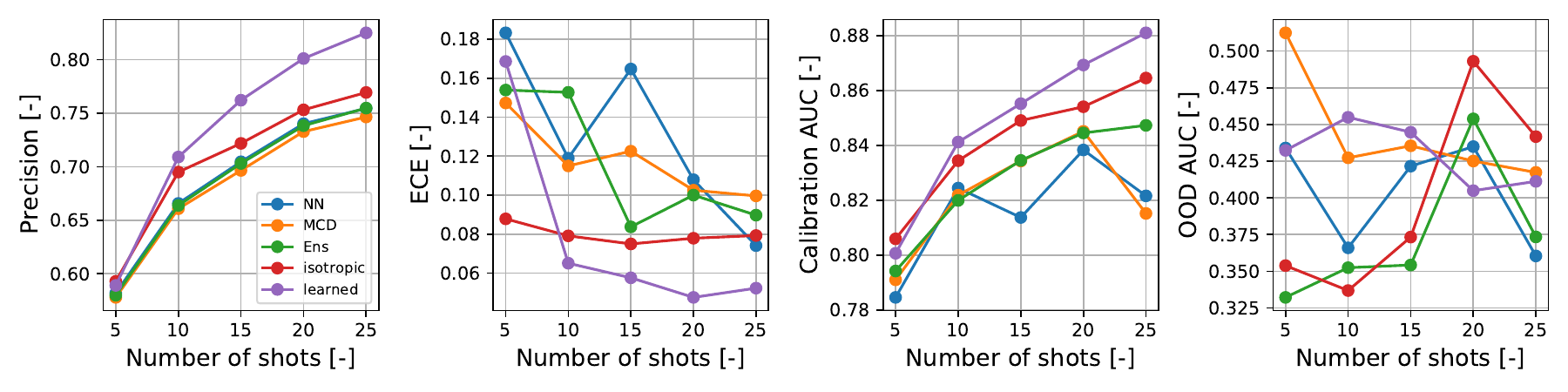}
    \caption{Few-shot learning experiments. Results are averaged over eight data-sets. Higher is better for accuracy and AUC measures, while lower is better for ECE measures. The results show the benefits of our method in uncertainty calibration and generalization.}
    \label{fig:main:result:3}
    \vspace*{-5mm}
\end{figure*}

Finally, we consider the task of few-shot learning. Our key hypothesis is that the choice of prior matters more in the small data regime, \ie, the prior may dominate over the likelihood term. To this end, closely following \citet{tran2022plex}, we use a few-shot learning set-up across several data-sets, namely CIFAR-100, UC Merced, Caltech 101, Oxford-IIIT Pets, DTD, Colorectal Histology, Caltech-UCSD Birds 200, and Cars196. CIFAR-10 is used as an out-of-distribution (OOD) data-set \footnote{Unlike \citet{tran2022plex}, we did not use few-shot learning on ImageNet since we obtain our prior using the entire training set.}. Standardized metrics such as accuracy, ECE, AUROC, and OOD AUROC are examined. The implementations are based on uncertainty baselines~\citep{nado2021uncertainty}. We use ResNet-50 for the architecture. 

For the baselines, we choose MC dropout \citep{gal2016dropout} (MCD) and additionally include deep ensemble \citep{lakshminarayanan2017simple} (Ens) and a standard network (NN). These uncertainty estimation techniques are often used methods in practice~\citep{gustafsson2020evaluating}. We do not use the plex~\citep{tran2022plex} due to the lack of industry-level resources to comparably train our Bayesian prior. To increase their competitiveness, we uniformly searched over ten weight decays in the range from $10^{-1}$ to $10^{-10}$. We also tried fixed representation or only last-layer fine-tuning. Additionally, we add LA which represents the use of isotropic prior. To facilitate the fair comparison between isotropic and learned prior, we carefully searched the following combinations: (a) curvature scaling without PAC-Bayes, with McAllester and Cantoni, (b) ten temperature scaling ranging from $10^{-10}$ to $10^{-28}$, and (c) three weight decays $10^{-6}$ to $10^{-8}$ and $10^{-10}$. For all hyperparameters, we selected the best model using validation accuracy.

The results are reported in \cref{fig:main:result:3}, where we averaged across eight data-sets, closely following \citet{tran2022plex}. The results per data-set are in Appendix~\ref{appendix:fewshot:resultsperdata} whereas in Appendix~\ref{appendix:results_approximate_bound}, we also analyze the influence of these weight decays and temperature scales with varying numbers of available samples per class against accuracy. For shots up to 25 per class, we observe that our method often outperforms the baselines in terms of generalization and uncertainty calibration metrics. In particular, in the experiments, our learned prior from ImageNet significantly outperforms the isotropic prior within directly comparable set-ups. We interpret that the prior learning method is effective in this small data regime. This motivates the key idea of expressive posteriors as informative priors for BNNs. Moreover, for isotropic and learned prior, the McAllester bound often resulted in the best model. This also motivates the use of explicit curvature scaling for the generalization bounds.

%% file: sections/conclusion.tex
\section{Conclusion}

This paper presents a prior learning method for BNNs. Our prior can be learned from large-scale data-set and architecture, resulting in expressive probabilistic representations with generalization guarantees. Empirically, we demonstrated how we mitigate cold posterior effects and further obtain non-vacuous generalization bound as low as 0.885 in neural networks using LA. In our benchmark experiments within continual and few-shot learning tasks, we further showed advancements over the prior arts. 

Importantly, we find that the use-case of our prior learning method is more within the small data regime, e.g., when prior may dominate over the likelihood term. Finally, one of the fundamental assumptions of prior learning methods is on the existence of relevant data and tasks to learn the prior from. Moreover, for a valid PAC-Bayes analysis, the data-sets for individual tasks should not share common examples. In the future, we would like to see follow-ups that address this limitation, by either (a) learning prior in larger scale data-set and architectures like foundational models, or (b) combining self-supervised pre-training to quickly fine-tune the prior for the domain of the relevant task~\citep{bommasani2021opportunities}. Another direction is to obtain tighter generalization bounds by adapting clever tricks such as optimizing the entire covariance instead of individual scales with our objectives and using a subset of the training data to improve the prior~\citep{perez2021tighter}.

\textbf{Acknowledgments} The authors would like to thank the anonymous reviewers and area chairs for valuable feedback. This work is also supported by the Helmholtz Association’s Initiative and Networking Fund (INF) under the Helmholtz AI platform grant agreement (ID ZT-I-PF-5-1).

\newpage

%% file: sections/appendix.tex
\input{appendix/appendixA.tex}
\input{appendix/appendixB.tex}
\input{appendix/appendixC.tex}
\input{appendix/appendixD.tex}
\input{appendix/appendixE.tex}

%% file: appendix/appendixA.tex
\section{Appendix: Preliminaries}

Here, we first give an overview of the concept on Laplace approximation. Moreover, we discuss the main idea of the PAC-Bayesian theory.

\subsection{Laplace Approximation}
\label{sec:laplace_approximation}

\begin{figure}
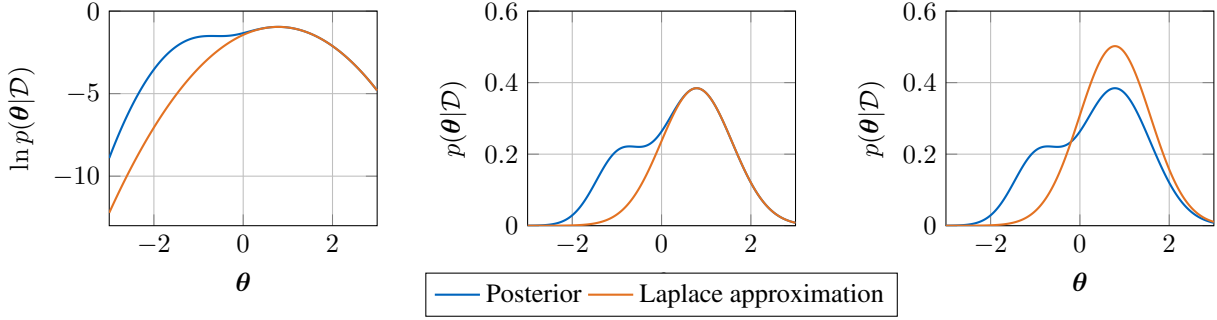

    \centering
    \includestandalone{tikz/related_work/laplace_approximation}
    \caption[Example of Laplace approximation.]{In the left plot one can see the log-posterior and the corresponding Taylor approximation. The middle and right plots show the posterior and the unnormalized and normalized Laplace approximation, respectively.}
    \label{fig:la}
\end{figure}

Laplace approximation~\citep{mackay1992practical} locally approximates the posterior around a mode $\boldsymbol{\hat{\theta}}$, namely the MAP estimate, by a normal distribution. For this, the second-order Taylor approximation of the log-posterior around $\boldsymbol{\hat{\theta}}$ is considered,

\begin{align}
    \ln{p(\boldsymbol{\theta}|\mathcal{D})} \approx \ln{p(\boldsymbol{\hat{\theta}}|\mathcal{D})} + \frac{1}{2} (\boldsymbol{\theta} - \boldsymbol{\hat{\theta}})^T \mathbf{H} (\boldsymbol{\theta} - \boldsymbol{\hat{\theta}}),\nonumber
\end{align}

with the Hessian of the log-posterior at the mode
    \begin{align}
        \mathbf{H} &= \frac{d^2}{{d\boldsymbol{\theta}}^2} \ln{p(\boldsymbol{\theta}|\mathcal{D})}\Bigr|_{\boldsymbol{\theta} = \boldsymbol{\theta}^*}\nonumber 
        = \frac{d^2}{{d\boldsymbol{\theta}}^2} \ln{p(\mathcal{D}|f_{\boldsymbol{\theta}})}\Bigr|_{\boldsymbol{\theta} = \boldsymbol{\theta}^*} + \frac{d^2}{{d\boldsymbol{\theta}}^2} \ln{p(\boldsymbol{\theta})}\Bigr|_{\boldsymbol{\theta} = \boldsymbol{\theta}^*}\nonumber
        =: \mathbf{H}_{likelihood} + \mathbf{H}_{prior}.\nonumber
    \end{align}
This is shown in the left plot in Figure~\ref{fig:la}. As the Taylor approximation is around a mode, the first-order term vanishes. The unnormalized Laplace approximation (shown in the middle of Figure~\ref{fig:la}) can then be obtained by taking the exponential,
    \begin{align}
        p(\boldsymbol{\theta} | \mathcal{D}) \appropto p(\boldsymbol{\hat{\theta}}|\mathcal{D})\exp{\left(\frac{1}{2} (\boldsymbol{\theta} - \boldsymbol{\hat{\theta}})^T \mathbf{H} (\boldsymbol{\theta} - \boldsymbol{\hat{\theta}})\right)}.
    \end{align}
The resulting normalized Laplace approximation is then a normal distribution at the mode with the precision matrix $\mathbf{H}$, \ie
    \begin{align}
        \label{eq:normal_la}
        p(\boldsymbol{\theta}|\mathcal{D}) \approx \mathcal{N}(\boldsymbol{\hat{\theta}}, \mathbf{H}^{-1}).
    \end{align}

\paragraph{Fisher Information Matrix}
As $\boldsymbol{\hat{\theta}}$ could be a saddle point or because of numerical instabilities, the Hessian could be indefinite in which case the normal distribution in Equation~\ref{eq:normal_la} is not well defined~\citep{martens2014new, botev2017practical}. Therefore, positive semi-definite approximations of the Hessian are used like the \emph{Fisher Information Matrix} (or the Fisher matrix) or the \emph{Gauss-Newton matrix}. For a likelihood of an exponential family, \eg categorical and normal distributions, both approximations are the same and for piece-wise linear activation functions like ReLU~\citep{nair2010rectified}, they moreover coincide with the Hessian~\citep{martens2015optimizing}.

\begin{definition}[Fisher Information Matrix~\citep{martens2015optimizing}]
    Let $P$ be a distribution over $\mathcal{X} \times \mathcal{Y}$ and $p(\cdot|\mathbf{x}, f_{\boldsymbol{\theta}})$ be a conditional distribution over $\mathcal{Y}$ dependent on the parameter vector $\boldsymbol{\theta}$. Then the Fisher matrix is defined as
    
    \begin{align}
        \mathbf{F} = \mathbb{E}_{(\mathbf{x}, \_) \sim P}\mathbb{E}_{\mathbf{y} \sim p(\cdot|\mathbf{x}, f_{\boldsymbol{\theta}})}\left[\frac{d \ln{p(\mathbf{y}|\mathbf{x}, f_{\boldsymbol{\theta}})}}{d\boldsymbol{\theta}} \frac{d\ln{ p(\mathbf{y}|\mathbf{x}, f_{\boldsymbol{\theta}})}}{d\boldsymbol{\theta}}^T\right].\nonumber
    \end{align}
    
    Here, the underscore denotes that the corresponding variable is not used.
\end{definition}
\begin{remark}
    Note that the targets from the training data are not used to compute the Fisher matrix. Using the ground truth targets instead of samples from the predictive model distribution would lead to the empirical Fisher matrix.
\end{remark}

For an easier notation, we write $\mathbb{E}$ instead of $\mathbb{E}_{(\mathbf{x}, \_) \sim P}\mathbb{E}_{\mathbf{y} \sim p(\cdot|\mathbf{x}, f_{\boldsymbol{\theta}})}$ and $\mathcal{D}\cdot = \frac{d \ln{p(\mathbf{y}|\mathbf{x}, f_{\boldsymbol{\theta}})}}{d\cdot}$ for the derivative of the log-likelihood in the following. Moreover, we drop the layer index for the activations and pre-activations, \ie $\mathbf{\bar{a}} = \mathbf{\bar{a}}_l$ and $\mathbf{s} = \mathbf{s}_l$.

\paragraph{The Fisher Approximations}
The full Fisher matrix and even a block-diagonal approximation without correlations between different layers are usually not feasible to store or compute for modern neural networks~\citep{martens2015optimizing}. 
Common approximations of the block-diagonal form are by a diagonal or a Kronecker-factored matrix. The Kronecker-factored approximation comes from the fact that for a single sample, the Fisher matrix is the sum of Kronecker-factored matrices. 
For fully-connected layers, the derivative after the weight matrix can be computed as $\mathcal{D} \mathbf{W} = \mathcal{D} \mathbf{s} (\mathbf{\bar{a}})^T$. With this, the block of the Fisher matrix corresponding to layer $l$ is given by $\mathbf{F}_l = \mathbb{E}[\mathcal{D} \mathbf{s} (\mathcal{D} \mathbf{s})^T \otimes \mathbf{\bar{a}} (\mathbf{\bar{a}})^T].$

As convolutional layers share the weight tensor among the spatial positions, the derivative is a sum of outer products: $\mathcal{D}\mathbf{W}_{i, k} = \sum_{\mathbf{t} \in \mathcal{T}} \mathcal{D} \mathbf{s}_{\mathbf{t}, i} \mathbf{\bar{a}}_{k, \mathbf{t}}$. Therefore, the Fisher matrix block for layer $l$ can be formulated as $\mathbf{F}_l = \mathbb{E}[\sum_{\mathbf{t} \in \mathcal{T}} \sum_{\mathbf{t'} \in \mathcal{T}} {\mathcal{D} \mathbf{s}}_{\mathbf{t}} (\mathcal{D} \mathbf{s})_{\mathbf{t'}}^T \otimes \mathbf{\bar{a}}_{:, \mathbf{t}} (\mathbf{\bar{a}}_{:, \mathbf{t'}})^T]$.

\subparagraph{Kronecker-Factored Approximate Curvature}
The \emph{Kronecker-Factored Approximate Curvature} (KFAC) approximates this Fisher matrix of a fully-connected layer~\citep{martens2015optimizing} and a convolutional layer~\citep{grosse2016kronecker} by
    \begin{align}
        \mathbf{F}_l \approx \mathbb{E}[\mathcal{D} \mathbf{s} (\mathcal{D} \mathbf{s})^T] \otimes \mathbb{E}[\mathbf{\bar{a}} (\mathbf{\bar{a}})^T] \quad \text{and} \quad \mathbf{F}_l \approx \mathbb{E}[(\mathcal{D} \mathbf{s})^T {\mathcal{D} \mathbf{s}}] \otimes \frac{1}{|\mathcal{T}|}\mathbb{E}[\mathbf{\bar{a}}(\mathbf{\bar{a}})^T],\nonumber
    \end{align}
respectively. 
In particular, KFAC approximates the expected Kronecker product as a Kronecker product of expectations, which is not true in general but leads to Kronecker-factored blocks of the Fisher matrix. The Kronecker factorization enables the storage of two smaller matrices rather than one large matrix~\citep{martens2015optimizing}. However, these factorizations assume that the activations and the corresponding pre-activations are statistically independent, which is usually not met in practice~\citep{tang2021skfac}. Furthermore, additional assumptions like the independence of the first and second-order statistics of the spatial positions are used for convolutional layers, which might impair its approximation quality.

\subparagraph{Kronecker-Factored Optimal Curvature}
Another tractable approximation of the Fisher matrix is the \emph{Kronecker-factored Optimal Curvature} (KFOC)~\citep{schnaus2021kronecker} which finds optimal Kronecker factors for each batch of data points and approximates both, the linear and the convolutional layer, as Kronecker product with two factors,
    \begin{align}
        \mathbf{F}_l \approx \mathbf{L}_l \otimes \mathbf{R}_l.\nonumber
    \end{align}
It transforms the problem of finding optimal Kronecker-factors into the best rank-$1$-problem and solves it with a scalable version of the power method. The approximation is usually closer to the Fisher matrix than the approximation by KFAC in terms of the Frobenius error.

\paragraph{Laplace Approximation for KFAC and KFOC}
Given a block-diagonal Kronecker-factored precision matrix $H = \text{diag}(\mathbf{L}_1 \otimes \mathbf{R}_1, \mathbf{L}_2 \otimes \mathbf{R}_2, \cdots, \mathbf{L}_L \otimes \mathbf{R}_L)$,
the normal distribution of Equation~\ref{eq:normal_la} reduces to $L$ independent matrix normal:
    \begin{align}
        \mathcal{N}(\boldsymbol{\hat{\theta}}, \mathbf{F}^{-1}) = \prod_{l=1}^L \mathcal{MN}(\mathbf{\hat{W}}^l, \mathbf{L}_l, \mathbf{R}_l).\nonumber
    \end{align}
In the following, we assume a block-diagonal approximation of the Fisher matrix where each block $\mathbf{F}_l \in \mathbb{R}^{n_l \times n_l}$ corresponds to a layer $l \in [L]$. For fully-connected and convolutional layers, the dimensionality $n_l$ is the size of the vectorized weight matrix, \ie for fully-connected layers $n_l = d_l (d_{l-1} + 1)$ and for convolutional layers $n_l = c_l (c_{l-1} |\Delta^l| + 1)$.

\subsection{PAC-Bayesian Bounds}
\label{sec:pac_bayesian_bounds}
Even though Bayesian neural networks were introduced in the Bayesian framework, \emph{Probably Approximately Correct (PAC)-Bayesian bounds} introduce a frequentist method to bound the generalization of statistical functions like Bayesian neural networks~\citep{germain2016pac}. PAC bounds aim to upper bound the risk, \ie the expected loss on the true data distribution, with high probability using properties of the architecture and optimization of neural networks~\citep{wolf2018mathematical}. PAC-Bayesian bounds obtain similar bounds for Bayesian neural networks by comparing the posterior distribution with a data-set independent prior distribution~\citep{guedj2019primer}.

Let $\mathcal{G} = \{g: \mathcal{X} \to \mathcal{Y}\ |\  g\  \text{is measurable}\}$ be the set of hypotheses and a corresponding $\sigma$-algebra $\mathfrak{G}$ such that $(\mathcal{G}, \mathfrak{G})$ is a measurable space. Denote the set of probability distributions on this measurable space as
    \begin{align}
        \mathcal{M} = \{\mu: \mathcal{G} \to [0, 1]\ |\ \mu \text{ is a probability distribution on } (\mathcal{G}, \mathfrak{G})\}.\nonumber
    \end{align}
Moreover, let $l: \mathcal{G} \times \mathcal{X} \times \mathcal{Y} \to \mathbb{R}$ be a loss function. Then define the risk for $g \in \mathcal{G}$ as
    \begin{align}
        \mathcal{L}^{l}_{P}(g) = \mathbb{E}_{(\mathbf{x}, \mathbf{y}) \sim P}\left[l(g, \mathbf{x}, \mathbf{y})\right]\nonumber
    \end{align}
and its empirical counterpart on the training data as
    \begin{align}
        \hat{\mathcal{L}}^{l}_{\mathcal{D}}(g) = \frac{1}{N}\sum_{i = 1}^N l(g, \mathbf{x}_i, \mathbf{y}_i).\nonumber
    \end{align}
Note that the neural network $f_{\boldsymbol{\theta}}$ as defined above is not in $\mathcal{G}$ because its output does not have to be in $\mathcal{Y}$. It only specifies the parameters of a distribution $p(\cdot|\mathbf{x}, f_{\boldsymbol{\theta}})$ over $\mathcal{Y}$. Nonetheless, for example, the function defined by $x \mapsto \argmax_{y \in \mathcal{Y}} p(y|\mathbf{x}, f_{\boldsymbol{\theta}})$ is in $\mathcal{G}$.

Given a data-set independent prior distribution over the hypothesis set $\pi \in \mathcal{M}$, the PAC-Bayesian theory bounds the probability that the expected risk is large for hypotheses sampled from another probability distribution $\rho \in \mathcal{M}$ which is absolutely continuous \wrt $\pi$:
\begin{align}
    \label{eq:pac_bayes_bound}
    P_{D \sim P^N} \left( \forall \rho \in \mathcal{M},\ \rho \ll \pi:\ \mathbb{E}_{g \sim \rho}[\mathcal{L}^{l}_{P}(g)] \leq \delta(\rho, \pi, \mathcal{D}, \varepsilon) \right) \geq 1 - \varepsilon,
\end{align}
for $\varepsilon > 0$ and the PAC-Bayesian bound $\delta(\rho, \pi, \mathcal{D}, \varepsilon)$~\citep{guedj2019primer}.

The first bound was introduced by McAllester~\citep{mcallester1999some, mcallester1999pac, mcallester2003pac, mcallester2003simplified} for bounded loss functions.
\begin{definition}[McAllester Bound~\citep{guedj2019primer}]
    Let $\varepsilon > 0$, $\rho, \pi \in \mathcal{M}$, $\rho \ll \pi$ and $l: \mathcal{G} \times \mathcal{X} \times \mathcal{Y} \to [0, 1]$, then
    \begin{align}
        \label{eq:mc_allester}
        \delta(\rho, \pi, \mathcal{D}, \varepsilon) = \mathbb{E}_{g \sim \rho}[\hat{\mathcal{L}}^{l}_{\mathcal{D}}(g)] + \sqrt{\frac{\mathbb{KL}(\rho\|\pi) + \ln{\frac{2\sqrt{N}}{\varepsilon}}}{2 N}}
    \end{align}
    is an upper bound for Equation~\ref{eq:pac_bayes_bound}.
\end{definition}
The bound was improved for the error loss function by \citet{catoni2007pac} when $\frac{\mathbb{KL}(\rho\|\pi)}{N}$ is large. The error loss function, or $0$-$1$-loss, is defined as
\begin{align}
    \er: \mathcal{G} \times \mathcal{X} \times \mathcal{Y} \to \{0, 1\},\ \er(g, \mathbf{x}, \mathbf{y}) = \mathbbm{1}[f(\mathbf{x}) \neq \mathbf{y}],
\end{align}
where $\mathbbm{1}[s]$ is one if statement $s$ is true and else zero.
\begin{definition}[Catoni Bound~\citep{catoni2007pac}]
    Let $\varepsilon > 0$, $\rho, \pi \in \mathcal{M}$ and $\rho \ll \pi$, then
    \begin{align}
        \label{eq:catoni}
        \delta(\rho, \pi, \mathcal{D}, \varepsilon) = \inf_{c > 0} \frac{1 - \exp(-c \mathbb{E}_{g \sim \rho}[\hat{\mathcal{L}}^{\er}_{\mathcal{D}}(g)] - \frac{\mathbb{KL}(\rho\|\pi) - \ln{\varepsilon}}{N})}{1 - \exp(-c)}
    \end{align}
    is an upper bound for Equation~\ref{eq:pac_bayes_bound}.
\end{definition}
\begin{remark}
    Note that in the PAC-Bayesian literature, $\rho$ is called posterior even though it is an arbitrary distribution that is dependent on the data~\citep{guedj2019primer}. To distinguish this from the posterior computed by Bayes' rule, we will explicitly write PAC-Bayes posterior if we do not use Bayes' rule.
\end{remark}
The PAC-Bayes bounds depend mainly on two terms: the KL-divergence of the PAC-Bayes prior and posterior and the empirical risk. Therefore, these bounds are small when the posterior is close to the prior and the loss on the training data is low. Thus, a good model should explain the training data well while depending not too much on it. In PAC-Bayesian bounds, the prior and the posterior are both distributions over the functions and not over the weights like in Bayesian neural networks. Hence, to apply the PAC-Bayesian bounds for Bayesian neural networks, one needs to identify each weight sample with a sample in the function space. However, neural networks are not identifiable~\citep{brea2019weight, pourzanjani2017improving}. Thus, multiple different weight vectors can explain the same function given by a neural network architecture. The KL-divergence is, therefore, smaller in function space than in weight space and one obtains an upper bound by considering the distributions over the weights~\citep{dziugaite2017computing}. 

In this work, we use the two bounds introduced above. Nonetheless, we find an approximate upper bound of the expected empirical error for Laplace approximation and can compute the KL-divergence in closed form for our models. Hence, all results of this work can directly be applied to other PAC-Bayesian bounds given that they depend on the expected empirical error and the KL-divergence of the posterior and the prior.

%% file: tikz/related_work/laplace_approximation.tex
\begin{tikzpicture}
    \begin{groupplot}[
        group style={group size=3 by 1, horizontal sep=2cm, vertical sep=1.5cm, y descriptions at=all},
        width=.3\textwidth,
        grid,
        xmin=-3,
        xmax=3,
        legend style={at={(0.5, -0.25)}, anchor=north}
    ]
    \nextgroupplot[ymin=-13., ymax=0., xlabel=$\boldsymbol{\theta}$, ylabel=$\ln{p(\boldsymbol{\theta}|\mathcal{D})}$]
        \addplot +[mark=none, thick] table {data/log_posterior.table};
        \addplot +[mark=none, thick] table {data/log_la.table};
    
    \nextgroupplot[ymin=0., ymax=.6, xlabel=$\boldsymbol{\theta}$, ylabel=$p(\boldsymbol{\theta}|\mathcal{D})$]
        \addplot +[mark=none, thick] table {data/posterior.table};
        \addplot +[mark=none, thick] table {data/unnormalized_la.table};
    
    \nextgroupplot[legend to name={CommonLegend3},legend style={legend columns=2}, ymin=0., ymax=.6, xlabel=$\boldsymbol{\theta}$, ylabel=$p(\boldsymbol{\theta}|\mathcal{D})$]
        \addplot +[mark=none, thick] table {data/posterior.table}; \addlegendentry{Posterior}
        \addplot +[mark=none, thick] table {data/la.table}; \addlegendentry{Laplace approximation}
    
    \end{groupplot}
    \path (group c1r1.south east) -- node[below=.5cm]{\pgfplotslegendfromname{CommonLegend3}} (group c3r1.south west);
\end{tikzpicture}

%% file: appendix/appendixB.tex
\section{Appendix: Algorithmic Overview, Complexity and Extensions}
\label{appendix:bpnn_overview}

This section summarizes the training of BPNNs and their computational complexity. After that, we present the details about the proposed sums-of-Kronecker-product computations. First, we review the basic idea of progressive neural networks (PNNs)~\citep{rusu2016progressive} and their extension to arbitrary network architectures such as ResNets~\citep{rusu2016progressive}. We then describe the training process and present a pseudocode for it. Finally, we analyze the complexity of training BPNNs. We use the same setup as in \cref{sec:method:bayespnn} by considering the tasks $\mathfrak{T}_0, \dots, \mathfrak{T}_T$, where $\mathfrak{T}_0$ is used to learn the prior. Each task $\mathfrak{T}_t$ contains a training data-set $\mathcal{D}_t = \left((\mathbf{x}_i^{(t)}, \mathbf{y}_i^{(t)})\right)_{i=1}^{N_t}$.

\paragraph{Implementation of PNNs}
Given a network architecture, PNNs create a copy of that network, called a column, for each new task and also add lateral connections between different columns to promote a positive transfer between tasks. Lateral connections are parameterized functions that combine the features of previous layers with the current layer. Unlike \citet{rusu2016progressive}, we use the same function for the lateral connections as for the main column, instead of their adapter architecture. Let $\mathbf{s}_l^{(t)}$ and $\mathbf{a}_{l}^{(t)}$ denote the activations and pre-activations at layer $l$ from column $t$, then we compute the lateral connection as
\begin{align}
    \mathbf{s}_l^{(t)} = \frac{1}{t}\sum_{t' = 1}^t \phi_l^{(t', t)}(\mathbf{a}_{l-1}^{(t')}).\nonumber
\end{align}
Here, we use a superscript $(t', t)$ to denote the lateral connection from column $t'$ to $t$. We also use this notation for the main column with $t = t'$, for the weights, the prior and the posterior over the weights and the Fisher matrix as summarized in \cref{tab:notation_PBNN}.

We implement PNNs for arbitrary network architectures by storing intermediate activations on the one hand and on the other hand by introducing aggregation layers that combine the stored activations at the layer level. The first part is implemented by creating forward hooks for the lateral connection layers $l \in \mathfrak{L}$ that store the activations of the previous layer $a^{(t')}_{l-1}$ during the forward pass in the base network. To combine the activation, the aggregation layers apply the lateral layers to the stored activations of the previous columns, $\phi_l^{(t', t)}(a^{(t')}_{l-1})$, and compute the mean over the resulting pre-activations. The aggregation layer is incorporated into the base network by replacing each layer $l \in \mathfrak{L}$ with the composition of the layer followed by the aggregation layer. With this architectural change, PNNs can be applied to complex network architectures.

\begin{table}
    \centering
    {\small
    \renewcommand{\arraystretch}{1.3}
    \begin{tabular}{lllll}
        \toprule
         & \textbf{Weight} & \textbf{Lateral} & \textbf{Column} & \textbf{Full}\\
        \midrule
        Weight & $\boldsymbol{\theta}_l^{(t', t)}$ & $\boldsymbol{\theta}^{(t', t)} = \vectorize(\boldsymbol{\theta}_l^{(t', t)})_{l \in [L]}$ & $\boldsymbol{\theta}^{(t)} = \vectorize(\boldsymbol{\theta}^{(t', t)})_{t' \in [t]}$ & $\boldsymbol{\theta}^{(\leq t)} = \vectorize(\boldsymbol{\theta}^{(t')})_{t' \in [t]}$\\
        Prior & $\pi_l^{(t', t)}$ & $\pi^{(t', t)} = \prod_{l=1}^L \pi_l^{(t', t)}$ & $\pi^{(t)} = \prod_{t'=1}^t \pi^{(t', t)}$ & $\pi^{(\leq t)} = \prod_{t'=1}^t \pi^{(t')}$\\
        Posterior & $\rho_l^{(t', t)}$ & $\rho^{(t', t)} = \prod_{l=1}^L \rho_l^{(t', t)}$ & $\rho^{(t)} = \prod_{t'=1}^t \rho^{(t', t)}$ & $\rho^{(\leq t)} = \prod_{t'=1}^t \rho^{(t')}$\\
        Fisher matrix & $\mathbf{F}_l^{(t', t)}$ & $\mathbf{F}^{(t', t)}=\diag(\mathbf{F}_l^{(t', t)})_{l \in [L]}$ & $\mathbf{F}^{(t)}=\diag(\mathbf{F}^{(t', t)})_{t' \in [t]}$ & $\mathbf{F}^{(\leq t)}=\diag(\mathbf{F}^{(t')})_{t' \in [t]}$\\
        \bottomrule
    \end{tabular}
    }
    \caption[The notation for BPNN.]{The notation for BPNN. Here, $1 \leq t' \leq t$ and $l \in [L]$ for $t=t'$ and $l \in \mathfrak{L}$ for $t' < t$. For task $0$, we write $(0)$ instead of $(0, 0)$ as no inter-column weights are used.}
    \label{tab:notation_PBNN}
\end{table}

\paragraph{Training of BPNNs}
In Bayesian Progressive Neural Networks (BPNNs), we combine this architecture with Bayesian neural networks and, in particular, LA with our learned prior. Therefore, we start by learning a prior on the data-set $\mathcal{D}_0$. This prior is then used as the prior for the main columns excluding the lateral connections. As a prior for the lateral connections, we use the posterior from the corresponding layer of the originating column as shown in \cref{fig:pbnn}. Unlike the learned prior, this layer was trained to produce reasonable features given the features from the previous layer of the same column, so we use it here instead. After training the prior, the training for each column is similar to LA with the learned prior and curvature scaling. That is, we first optimize the parameters of the column, including all incoming lateral connections, using either MAP estimation or the frequentist projection. During the training, we already sample the weights from the previous columns to make the new column robust to small changes in the activations from the lateral connections. After finding the optimal network parameters for a column, we compute the Fisher matrix. We keep the previous weights and distributions fixed, so the Fisher matrix is computed only for the new parameters for the current task. Finally, the curvature is scaled using a PAC-Bayes objective as explained in \cref{sec:method:pacbayes}. The complete training procedure is also shown in \cref{alg:bpnn_training}.

\begin{algorithm}
   \caption{Training of Bayesian Progressive Neural Networks}
   \label{alg:bpnn_training}
\begin{algorithmic}[1]
   \STATE {\bfseries Input:} network architecture $f_{\cdot}$, datasets $(\mathcal{D}_t)_{t=0}^T$, lateral connections $\mathfrak{L}$, weight decay for task $\mathfrak{T}_0$
   \STATE {\bfseries Output:} The posterior distribution for BPNNs $\rho^{(\leq T)} = \mathcal{N}(\boldsymbol{\hat{\theta}}^{(\leq T)}, (\mathbf{\tilde{F}}^{(\leq T)})^{-1})$
   \STATE $\boldsymbol{\hat{\theta}}^{(0)} \leftarrow$ optimal parameters for $f_{\cdot}$ by MAP estimation or a PAC-Bayes objective (\cref{subsubsec:frequentist_projection}) using $\mathcal{D}_0$
   \STATE $\mathbf{F}^{(0)} \leftarrow$ compute the Fisher using $\mathcal{D}_0$ \COMMENT{simultaneously compute the negative data log-likelihood for $f_{\boldsymbol{\hat{\theta}}^{(0)}}$}
   \STATE $\mathbf{\alpha}^{(0)}, \mathbf{\beta}^{(0)}, \mathbf{\tau}^{(0)} \leftarrow \argmin_{\mathbf{\alpha}, \mathbf{\beta}, \mathbf{\tau}} g(\mathbf{\alpha}, \mathbf{\beta}, \mathbf{\tau})$ \COMMENT{here, g is $ma$ (\cref{eq:mc_allester_objective}) or $ca$ (\cref{eq:catoni_objective})}
   \STATE $\mathbf{\tilde{F}}^{(0)}_l \leftarrow \frac{1}{\mathbf{\tau}^{(0)}_l} (\mathbf{\beta}^{(0)}_l \mathbf{F}^{(0)}_l + \mathbf{\alpha^{(0)}_l} \gamma I)$ for $l \in [L]$
   \FOR{$t=1$ to $T$}
        \STATE add a new column and new lateral connections
        \STATE $\boldsymbol{\hat{\theta}}^{(t)} \leftarrow$ optimize the new parameters by MAP estimation or a PAC-Bayes objective using $\mathcal{D}_t$
       \STATE $\mathbf{F}^{(t)} \leftarrow$ compute the Fisher using $\mathcal{D}_t$
       \STATE $\mathbf{\alpha}^{(t)}, \mathbf{\beta}^{(t)}, \mathbf{\tau}^{(t)} \leftarrow \argmin_{\mathbf{\alpha}, \mathbf{\beta}, \mathbf{\tau}} g(\mathbf{\alpha}, \mathbf{\beta}, \mathbf{\tau})$
       \STATE $\mathbf{\tilde{F}}^{(t, t)}_l \leftarrow \frac{1}{\mathbf{\tau}^{(t, t)}_l} (\mathbf{\beta}^{(t, t)}_l \mathbf{F}^{(t, t)}_l + \mathbf{\alpha^{(t, t)}_l} \mathbf{\tilde{F}}^{(0)}_l)$ for $l \in [L]$
       \STATE $\mathbf{\tilde{F}}^{(i, t)}_l \leftarrow \frac{1}{\mathbf{\tau}^{(i, t)}_l} (\mathbf{\beta}^{(i, t)}_l \mathbf{F}^{(i, t)}_l + \mathbf{\alpha^{(i, t)}_l} \mathbf{\tilde{F}}^{(i, i)}_l)$ for $i \in [t-1]$, $l \in \mathfrak{L}$
       \STATE $\rho^{(i, t)}_l \leftarrow \mathcal{N}(\boldsymbol{\hat{\theta}}^{(t)}, (\mathbf{\tilde{F}}^{(i, t)}_l)^{-1})$ for all available $i, t, l$
   \ENDFOR
\end{algorithmic}
\end{algorithm}

\paragraph{Computational Complexity} 
For each task and additionally for the prior task, the computational complexity boils down to finding the optimal parameters, computing the Fisher matrix, and scaling the curvature. Parameter optimization is equivalent to finding the MAP estimate. This can be solved efficiently by computing the negative log-likelihood over mini-batches using common variants of stochastic gradient descent~\citep{kingma2015adam}. In addition, the quadratic term induced by the prior must be computed in each update step. For a diagonal $\mathbf{\tilde{F}}_l$, it can be computed with two element-wise multiplications of vectors of size $n_l$ with $(\boldsymbol{\hat{\theta}}_l - \boldsymbol{\tilde{\theta}}_l)^T \mathbf{\tilde{F}}_l (\boldsymbol{\hat{\theta}}_l - \boldsymbol{\tilde{\theta}}_l) = (\boldsymbol{\hat{\theta}}_l - \boldsymbol{\tilde{\theta}}_l) \odot \diag(\mathbf{\tilde{F}}_l) \odot (\boldsymbol{\hat{\theta}}_l - \boldsymbol{\tilde{\theta}}_l)$, where $\diag(\mathbf{\tilde{F}}_l)$ is the vector containing the diagonal elements of $\mathbf{\tilde{F}}_l$. The computations for Kronecker-factored matrices involve two matrix-matrix products with the Kronecker-factors and one element-wise product:
    \begin{align}
        (\boldsymbol{\hat{\theta}}_l - \boldsymbol{\tilde{\theta}}_l)^T \mathbf{\tilde{F}}_l (\boldsymbol{\hat{\theta}}_l - \boldsymbol{\tilde{\theta}}_l) &= (\boldsymbol{\hat{\theta}}_l - \boldsymbol{\tilde{\theta}}_l)^T (\mathbf{L} \otimes \mathbf{R}) (\mathbf{\hat{W}}^l - \mathbf{\tilde{W}}^l)\nonumber\\
        &= (\boldsymbol{\hat{\theta}}_l - \boldsymbol{\tilde{\theta}}_l) \odot \vectorize(\mathbf{L}(\mathbf{\hat{W}}^l - \mathbf{\tilde{W}}^l)\mathbf{R}),\nonumber
    \end{align}
where $\mathbf{\hat{W}}^l$ and $\mathbf{\tilde{W}}^l$ are the weight matrices, which  correspond to $\boldsymbol{\hat{\theta}}_l$ and $\boldsymbol{\tilde{\theta}}_l$ respectively. Hence, this term can be computed efficiently, so that the parameter optimization has a complexity similar to standard MAP training. For common approximations of the Fisher matrix, such as the KFAC~\cite{martens2014new} and KFOC~\cite{schnaus2021kronecker}, the computation of the Fisher matrix corresponds to an additional epoch over the training data~\citep{martens2015optimizing, grosse2016kronecker}, although an update step is usually slightly slower than updating the weights. The curvature scaling is an optimization problem with a total of $3 L$ parameters, where $L$ is the number of layers in the network when all three scales are optimized. Thus, it is usually a much lower-dimensional optimization than finding the weights. Also, the computation of the negative log-likelihood can be incorporated into the computation of the Fisher matrix. Therefore, no additional pass through the data is needed. Overall, BPNNs have a similar computational complexity during training as PNNs, with the main overhead being the training of an additional task. However, this overhead can be reduced by using a pre-trained model which then leads to the minor overhead of about one additional epoch to compute the Fisher matrix for each task. During inference, BPNNs use multiple network samples and thus, have a computational complexity comparable to multiple forward passes in PNN.

We further comment on the complexity of the existing baseline methods. Deep ensemble corresponds to training ensembles of deep learning models with different initializations and is known to be more expensive than Kronecker-factored Laplace approximation \citep{daxberger2021laplace}. Full LA involves the Hessian, which requires the memory complexity of storing matrices and inverting these matrices, which is cubic in cost. Therefore, full LA is known not to scale, and previous research introduced several approximations such as layer-wise Kronecker-factorization. Of course, learning the prior incurs more cost than relying on an isotropic Gaussian prior.

\subsection{Complexity of the Power Method for Sums of Kronecker Products}
For all posteriors in BPNN, the final covariance matrix is a sum of Kronecker products. Hence, in general, we are interested in approximating
\begin{align}
    \mathbf{\hat{L}}, \mathbf{\hat{R}} \in \argmin_{\substack{\mathbf{L} \in \mathbb{R}^{M \times M}\\ \mathbf{R} \in \mathbb{R}^{N \times N}}} \|\sum_{k=1}^K \mathbf{L}^k \otimes \mathbf{R}^k - \mathbf{L} \otimes \mathbf{R}\|_F,
\end{align}
for $M, N, K \in \mathbb{N}$, $\mathbf{L}^k \in \mathbb{R}^{M \times M}$ and $\mathbf{R}^k \in \mathbb{R}^{N \times N}$ for $k \in [K]$. \cref{lem:vec_kron} shows that this problem is equivalent to a rank-one approximation. Therefore, one can use the power method to solve the problem. Nonetheless, each step of the plain power method consists of a matrix multiplication with an $M^2 \times N^2$-sized matrix. The complexity can be reduced by utilizing that the matrix is a sum of few rank-one matrices. This is shown in Algorithm \ref{alg:kronpower}. With this, the convergence properties of the power method are achieved with a computational complexity of $\mathcal{O}\left(n^{max} K (N^2 + M^2)\right)$. Also, only $\mathcal{O}\left(K (N^2 + M^2)\right)$ memory is needed. Here, we assume that a matrix multiplication $\mathbf{A} \mathbf{B}$ for $\mathbf{A} \in \mathbb{R}^{m \times n}$ and $\mathbf{B} \in \mathbb{R}^{n \times k}$ has the complexity $\mathcal{O}(m n k)$ while a Hadamard product $\mathbf{A} \odot \mathbf{C}$ for $\mathbf{C} \in \mathbb{R}^{m \times n}$ can be computed in $\mathcal{O}(m n)$.

\begin{algorithm}
   \caption{Power method for sums of Kronecker products}
   \label{alg:kronpower}
\begin{algorithmic}[1]
   \STATE {\bfseries Input:} left matrices $(\mathbf{L}^k)_{k \in [K]}$, right matrices $(\mathbf{R}^k)_{k \in [K]}$, number of steps $n^{max}=100$, stopping precision $\delta=10^{-5}$
   \STATE {\bfseries Output:} $\mathbf{\hat{L}}, \mathbf{\hat{R}} \in \argmin_{\mathbf{L}, \mathbf{R}} \|\sum_{k=1}^K \mathbf{L}^k \otimes \mathbf{R}^k - \mathbf{L} \otimes \mathbf{R}\|_F$
   \STATE $\vectorize(\mathbf{\Bar{L}}^{(0)}) \leftarrow \mathcal{N}(\mathbf{0}, \mathbf{I})$ \COMMENT{standard normal initialization of $\mathbf{\Bar{L}}^{(0)}$}
   \STATE $\mathbf{L}^{(0)} \leftarrow \frac{\mathbf{\Bar{L}}^{(0)}}{\|\mathbf{\Bar{L}}^{(0)}\|_F}$ \COMMENT{normalize $\mathbf{\Bar{L}}^{(0)}$}
   \FOR{$n=1$ to $n^{max}$}
        \STATE $\mathbf{\Bar{R}}^{(n)} \leftarrow \sum_{k=1}^K \langle \mathbf{L}^k, \mathbf{L}^{(n-1)} \rangle_F \mathbf{R}^k$ \COMMENT{first power iteration step}
        \STATE $\mathbf{R}^{(n)} \leftarrow \frac{\mathbf{\Bar{R}}^{(n)}}{\|\mathbf{\Bar{R}}^{(n)}\|_F}$ \COMMENT{normalize $\mathbf{\Bar{R}}^{(n)}$}
        \STATE $\mathbf{\Bar{L}}^{(n)} \leftarrow \sum_{k=1}^K \langle \mathbf{R}^k, \mathbf{R}^{(n)} \rangle_F \mathbf{L}^k$ \COMMENT{second power iteration step}
        \STATE $\mathbf{L}^{(n)} \leftarrow \frac{\mathbf{\Bar{L}}^{(n)}}{\|\mathbf{\Bar{L}}^{(n)}\|_F}$ \COMMENT{normalize $\mathbf{\Bar{L}}^{(n)}$}
        \IF{$\|\mathbf{L}^{(n)} - \mathbf{L}^{(n-1)}\|_F < \delta$}
            \STATE {\bfseries break} \COMMENT{stopping criterion}
        \ENDIF
   \ENDFOR
   \STATE $\mathbf{\hat{L}} \leftarrow \mathbf{L}^{(n)}$
   \STATE $\mathbf{\hat{R}} \leftarrow \sum_{k=1}^K \langle \mathbf{L}^k, \mathbf{L}^{(n)} \rangle_F \mathbf{R}^k$ \COMMENT{first power iteration step}
\end{algorithmic}
\end{algorithm}

%% file: appendix/appendixC.tex
\section{Appendix: Proof and Derivations}
\label{appendix:derivations}

This section contains proofs for the presented theory. Detailed remarks and follow up derivations are further provided.

\subsection{Proof of Lemma \ref{lem:vec_kron}}
\label{subsec:lemma1_proof}
\begin{lemmanum}[\ref{lem:vec_kron}]
Let $M, N, K \in \mathbb{N}$, $\mathbf{L}^k \in \mathbb{R}^{M \times M}$ and $\mathbf{R}^k \in \mathbb{R}^{N \times N}$ for $k \in [K]$. Then
\begin{align}
    \tag{\ref{eq:vec_kron}}
    \|\sum_{k=1}^K \mathbf{L}^k \otimes \mathbf{R}^k - \mathbf{L} \otimes \mathbf{R}\|_F = \|\sum_{k=1}^K \vectorize(\mathbf{L}^k) \vectorize(\mathbf{R}^k)^T - \vectorize(\mathbf{L}) \vectorize(\mathbf{R})^T\|_F.
\end{align}
\end{lemmanum}

\begin{proof}
Let $i, j \in [M N]$. Then the $i,j$-th entry of the left matrix is
\begin{align}
    &\left(\sum_{k=1}^K \mathbf{L}^k \otimes \mathbf{R}^k - \mathbf{L} \otimes \mathbf{R}\right)_{i,j} = \sum_{k=1}^K \mathbf{L}^k_{i_1, j_1} \mathbf{R}^k_{i_2, j_2} - \mathbf{L}_{i_1, j_1} \mathbf{R}_{i_2, j_2} \nonumber\\
    &= \left(\sum_{k=1}^K \vectorize(\mathbf{L}^k) \vectorize(\mathbf{R}^k)^T - \vectorize(\mathbf{L}) \vectorize(\mathbf{R})^T\right)_{M(i_1 - 1) + j_1, N(i_2 - 1) + j_2},
\end{align}
with $i = N (i_1 - 1) + i_2$, $j = N (j_1 - 1) + j_2$.
Hence, both matrices have the same entries and only the order of the entries is in general different. Therefore, the sum over the squared entries and thus the Frobenius norm is the same:
\begin{align}
    &\|\sum_{k=1}^K \mathbf{L}^k \otimes \mathbf{R}^k - \mathbf{L} \otimes \mathbf{R}\|_F^2 = \sum_{i=1}^{M^2} \sum_{j=1}^{N^2} \left(\sum_{k=1}^K \mathbf{L}^k \otimes \mathbf{R}^k - \mathbf{L} \otimes \mathbf{R}\right)_{i,j}^2\nonumber\\
    &= \sum_{i_1,j_1=1}^N \sum_{i_2,j_2=1}^M \left(\sum_{k=1}^K \mathbf{L}^k \otimes \mathbf{R}^k - \mathbf{L} \otimes \mathbf{R}\right)_{N (i_1 - 1) + i_2,N (j_1 - 1) + j_2}^2\nonumber\\
    &= \sum_{i_1,j_1=1}^N \sum_{i_2,j_2=1}^M \left(\sum_{k=1}^K \mathbf{L}_{i_1, j_1}^k \otimes \mathbf{R}_{i_2, j_2}^k - \mathbf{L}_{i_1, j_1} \otimes \mathbf{R}_{i_2, j_2}\right)^2\nonumber\\
    &= \sum_{i_1,j_1=1}^N \sum_{i_2,j_2=1}^M \left(\sum_{k=1}^K \vectorize(\mathbf{L}^k) \vectorize(\mathbf{R}^k)^T - \vectorize(\mathbf{L}) \vectorize(\mathbf{R})^T\right)_{M(i_1 - 1) + j_1, N(i_2 - 1) + j_2}^2\nonumber\\
    &= \|\sum_{k=1}^K \vectorize(\mathbf{L}^k) \vectorize(\mathbf{R}^k)^T - \vectorize(\mathbf{L}) \vectorize(\mathbf{R})^T\|_F^2.
\end{align}
\end{proof}

\subsection{Proof of Lemma \ref{lem:opt_kron_sum}}
\label{subsec:lemma2_proof}
\begin{lemmanum}[\ref{lem:opt_kron_sum}]
    Let $\mathbf{A} = \sum_{k=1}^K \vectorize(\mathbf{L}^k) \vectorize(\mathbf{R}^k)^T$ and $\mathbf{A} = \sum_{i=1}^r \sigma_i \mathbf{u}_i \mathbf{v}_i^T$ be its singular value decomposition with $\sigma_1 \geq \sigma_2 \geq \dots \geq \sigma_r > 0$ and $\mathbf{u}_i^T \mathbf{u}_j = \mathbf{v}_i^T \mathbf{v}_j = \mathbbm{1}[i = j]$. Then there is a solution of equation \ref{eq:opt_kron_sum} with
    \begin{align}
        \label{eq:best_LR}
        \vectorize(\mathbf{\hat{L}}) = \mathbf{u}_1, \vectorize(\mathbf{\hat{R}}) = \sigma_1 \mathbf{v}_1.
    \end{align}
    If $\sigma_1 > \sigma_2$, the solution is unique up to changing the sign of both factors, and Algorithm \ref{alg:kronpower} converges almost surely to this solution.
\end{lemmanum}
\begin{proof}
    The main idea of the proof is to use Lemma \ref{lem:vec_kron} to identify the problem with a best rank-one approximation. The algorithm then corresponds to the power method that utilizes the Kronecker factorization for a faster and memory-efficient computation of the matrix-vector products in the Kronecker matrix space.\\
    By the Eckart–Young–Mirsky theorem~\citep{eckart1936approximation}, an optimal rank-one approximation for $A$ in the Frobenius norm is
        \begin{align}
            \sigma_1 \mathbf{u}_1 \mathbf{v}_1^T \in \argmin_{\mathbf{\hat{A}} \in \mathbb{R}^{M^2 \times N^2}: \rank(\mathbf{\hat{A}}) = 1} \|\mathbf{A} - \mathbf{\hat{A}}\|_F,
        \end{align}
    which is unique up to changing the sign of both factors if $\sigma_1 > \sigma_2$.\\
    Therefore, the matrices $\mathbf{\hat{L}}$ and $\mathbf{\hat{R}}$ that satisfy equation (\ref{eq:best_LR}) are optimal solutions of equation (\ref{eq:opt_kron_sum}). Moreover, the left factor is normalized, e.g. $\|\mathbf{\hat{L}}\|_F = \|\mathbf{u}_1\|_2 = 1$.\\
    The equivalence of Algorithm \ref{alg:kronpower} with the power method can be seen by multiplying $\mathbf{A} \mathbf{A}^T$ with $\vectorize(\mathbf{L}^{(n-1)})$ for $\mathbf{L}^{(n-1)} \in \mathbb{R}^{N^2}$:
        \begin{align}
            \mathbf{A}^T \vectorize(\mathbf{L}^{(n-1)}) &= \sum_{k=1}^K \vectorize(\mathbf{R}^k) \vectorize(\mathbf{L}^k)^T \vectorize(\mathbf{L}^{(n-1)})\\
            &= \sum_{k=1}^K \langle \mathbf{L}^k, \mathbf{L}^{(n-1)} \rangle_F \vectorize(\mathbf{R}^k)\\
            &= \vectorize(\mathbf{\Bar{R}}^{(n)}),
        \end{align}
    and
        \begin{align}
            \mathbf{A} \mathbf{A}^T \vectorize(\mathbf{L}^{(n-1)}) &= \sum_{k=1}^K \vectorize(\mathbf{L}^k) \vectorize(\mathbf{R}^k)^T \vectorize(\mathbf{\Bar{R}}^{(n)})\\
            &= \sum_{k=1}^K \langle \mathbf{R}^k, \mathbf{\Bar{R}}^{(n)} \rangle_F \vectorize(\mathbf{L}^k)\\
            &= \|\mathbf{\Bar{R}}^{(n)}\|_F \sum_{k=1}^K \langle \mathbf{R}^k, \mathbf{R}^{(n)} \rangle_F \vectorize(\mathbf{L}^k)\\
            &= \|\mathbf{\Bar{R}}^{(n)}\|_F \vectorize(\mathbf{\Bar{L}}^{(n)}).
        \end{align}
    Hence, we can compute the same iterations as the standard power method, like Algorithm \ref{alg:kronpower}:
        \begin{align}
        \frac{\mathbf{A} \mathbf{A}^T \vectorize(\mathbf{L}^{(n-1)})}{\|\mathbf{A} \mathbf{A}^T \vectorize(\mathbf{L}^{(n-1)})\|_2} = \frac{\|\mathbf{\Bar{R}}^{(n)}\|_F \vectorize(\mathbf{\Bar{L}}^{(n)})}{\|\mathbf{\Bar{R}}^{(n)}\|_F \|\mathbf{\Bar{L}}^{(n)}\|_F} = \frac{\vectorize(\mathbf{\Bar{L}}^{(n)})}{\|\mathbf{\Bar{L}}^{(n)}\|_F} = \vectorize(\mathbf{L}^{(n)}).
        \end{align}
    The final right factor then corresponds to $\mathbf{A}^T \vectorize(\mathbf{L}^{(n)}) \approx \sigma_1 \mathbf{v}_1$.\\
    For $\sigma_1 > \sigma_2$, the convergence properties are inherited from the power method, e.g., see the work of \citet{bindel2016power}.
\end{proof}
\begin{remark}
    Even in the case when the first singular value is not (much) larger than the other singular values and no convergence is achieved, the resulting matrices of Algorithm \ref{alg:kronpower} are with high probability in the span of the singular vectors corresponding to the set of large singular values~\citep{blum2020best}. Hence, in this case, the approximation will still converge to good Kronecker factors with high probability.
\end{remark}

\subsection{Derivation of the PAC-Bayes Objectives}
In this section, we provide further information on the PAC-Bayes objectives of \cref{eq:mc_allester_objective} and \cref{eq:catoni_objective}. In particular, we first derive the upper bound and its approximation that is shown in \cref{eq:approximate_upper_bound} together with computing the KL-divergence in dependence on the curvature scales. This is then used to derive the PAC-Bayes objectives. Next, we show, how these objectives can be adapted for network parameter optimization in the frequentist projection. Finally, we present the extension of the objectives to continual learning setup, in particular the proposed BPNNs.

\paragraph{Upper Bound of the Expected Empirical Error}
Here, we show the first equation of \cref{eq:approximate_upper_bound}:
\begin{align}
    \label{eq:upper_bound}
    \mathbb{E}_{\boldsymbol{\theta} \sim \rho}\left[\frac{1}{N} \sum_{i=1}^N \mathbbm{1}[\argmax_{\mathbf{y}'} p(\mathbf{y}'|\mathbf{x}_i, f_{\boldsymbol{\theta}}) \neq \mathbf{y}_i] \right] \leq \mathbb{E}_{\boldsymbol{\theta} \sim \rho}\left[\frac{1}{N} \sum_{i=1}^N - \frac{\ln{p(\mathbf{y}_i|\mathbf{x}_i,f_{\boldsymbol{\theta}})}}{\ln{2}}\right].
\end{align}
For this, as a first step, we present \cref{lem:upper_bound_error}:
\begin{lemma}
    \label{lem:upper_bound_error}
    Let $f: \mathcal{X} \to \Omega_L$, $(\mathbf{x}, \mathbf{y}) \in \mathcal{X} \times \mathcal{Y}$ and 
    \begin{align*}
        \er(f, \mathbf{x}, \mathbf{y}) = \mathbbm{1}[\argmax_{\mathbf{y}' \in \mathcal{Y}} p(\mathbf{y}'|\mathbf{x}, f) \neq \mathbf{y}],
    \end{align*}
    then
    \begin{align*}
        \er(f, \mathbf{x}, \mathbf{y}) \leq - \frac{\ln{p(\mathbf{y}|\mathbf{x},f)}}{\ln{2}}.
    \end{align*}
\end{lemma}

\begin{proof}
    The proof examines the case of a correct prediction first and of a wrong prediction second. For the correct prediction, the error is $0$ and the bound reduces to the negative log-likelihood being larger than $0$. In the case of a wrong prediction, we use that there is a class with a higher probability and that the correct and the most probable class have together a probability that can be bounded by $1$ from above.
    
    First, consider the case of a correct classification, which means that $y$ has the largest probability: $\argmax_{\mathbf{y}' \in \mathcal{Y}} p(\mathbf{y}'|\mathbf{x}, f) = \mathbf{y}$. Then, $\er(f, \mathbf{x}, \mathbf{y}) = 0$. As $p(\mathbf{y}|\mathbf{x},f)$ is a discrete probability, $p(\mathbf{y}|\mathbf{x},f) \leq 1$ and hence, by taking the negative logarithm on both sides, $- \frac{\ln{p(\mathbf{y}|\mathbf{x},f)}}{\ln{2}} \geq 0 = \er(f, \mathbf{x}, \mathbf{y})$.
    
    Now, let $\argmax_{\mathbf{y}' \in \mathcal{Y}} p(\mathbf{y}'|\mathbf{x}, f) \neq \mathbf{y}$. Therefore, $\er(f, \mathbf{x}, \mathbf{y}) = 1$ and there exists a $\mathbf{y}' \in \mathcal{Y}$ such that $p(\mathbf{y}'|\mathbf{x}, f) \geq p(\mathbf{y}|\mathbf{x}, f)$. Additionally, we have that $1 \geq p(\mathbf{y}'|\mathbf{x}, f) + p(\mathbf{y}|\mathbf{x}, f) \geq 2 p(\mathbf{y}|\mathbf{x}, f)$. This can be written as $p(\mathbf{y}|\mathbf{x}, f) \leq \frac{1}{2}$. Due to the monotony of the logarithm, we can take the natural logarithm on both sides and divide by $-\ln{2}$ to obtain that $- \frac{\ln{p(\mathbf{y}|\mathbf{x},f)}}{\ln{2}} \geq 1 = \er(f, \mathbf{x}, \mathbf{y})$.
    
    Consequently, $- \frac{\ln{p(\mathbf{y}|\mathbf{x},f)}}{\ln{2}}$ is an upper bound of $\er(f, \mathbf{x}, \mathbf{y})$.
\end{proof}
\begin{remark}
    We can see from the proof that the bound is tighter when the correct class has a high likelihood. Hence, the bound will be best for good-performing models, while it might be loose when the negative data log-likelihood is large. 
\end{remark}

We can directly apply \cref{lem:upper_bound_error} to the empirical error of our neural network $f_{\boldsymbol{\theta}}$ to get
    \begin{align}
        \mathbbm{1}[\argmax_{\mathbf{y}'} p(\mathbf{y}'|\mathbf{x}_i, f_{\boldsymbol{\theta}}) \neq \mathbf{y}_i] = \er(f_{\boldsymbol{\theta}}, \mathbf{x}_i, \mathbf{y}_i) \leq - \frac{\ln{p(\mathbf{y}_i|\mathbf{x}_i,f_{\boldsymbol{\theta}})}}{\ln{2}}
    \end{align}
for each element $i$ in the data-set. Therefore, also the sum and the expectation is larger or equal leading to \cref{eq:upper_bound}.

Moreover, we can plug this upper bound into the PAC-Bayes bounds to receive new upper bounds on the expected loss on the true data distribution:
\begin{corollary}
    Let $\varepsilon > 0$, $N \in \mathbb{N}$, and $\rho \ll \pi$ distributions over the weight space. Then
    \begin{align*}
        L_{\text{upper}} + \sqrt{\frac{\mathbb{KL}(\rho\|\pi) + \ln{\frac{2\sqrt{N}}{\varepsilon}}}{2 N}} \quad \text{and} \quad \inf_{c > 0} \frac{1 - \exp(-c L_{\text{upper}} - \frac{\mathbb{KL}(\rho\|\pi) - \ln{\varepsilon}}{N})}{1 - \exp(-c)}
    \end{align*}
    are upper bounds of the expected error $\mathbb{E}_{\boldsymbol{\theta} \sim \rho}[\mathcal{L}^{l}_{P}(f_{\boldsymbol{\theta}})]$ in \cref{eq:pac_bayes} with probability larger or equal to $1 - \varepsilon$.
\end{corollary}

\begin{proof}
    Both bounds come from using the upper bound instead of the expected empirical error in the McAllester~\citep{guedj2019primer} and Catoni~\citep{catoni2007pac} bounds. Since both bounds are monotone with respect to the expected empirical error, we get a new bound for each of them.
\end{proof}

\paragraph{Approximation of the Expected Empirical Error}
Next, we address the approximation of the upper bound from \cref{eq:approximate_upper_bound}:
    \begin{align}
        \label{eq:approximation}
        \mathbb{E}_{\boldsymbol{\theta} \sim \rho}\left[\frac{1}{N} \sum_{i=1}^N - \frac{\ln{p(\mathbf{y}_i|\mathbf{x}_i,f_{\boldsymbol{\theta}})}}{\ln{2}}\right] \approx \frac{- \ln{p(\mathcal{D}|f_{\boldsymbol{\hat{\theta}}})} + \frac{1}{2} \sum_{l \in [L]} \mathbf{\tau}_l \trace \left(\mathbf{F}_l (\mathbf{\beta}_l \mathbf{F}_l + \mathbf{\alpha}_l \mathbf{\tilde{F}}_l)^{-1}\right)}{N \ln{2}}.
    \end{align}
First, we observe that the upper bound in \cref{eq:upper_bound} is the scaled negative data log-likelihood:
    \begin{align*}
        \mathbb{E}_{\boldsymbol{\theta} \sim \rho}\left[\frac{1}{N} \sum_{i=1}^N - \frac{\ln{p(\mathbf{y}_i|\mathbf{x}_i,f_{\boldsymbol{\theta}})}}{\ln{2}}\right] = \frac{1}{N \ln{2}} \mathbb{E}_{\boldsymbol{\theta} \sim \rho}\left[- \ln{p(\mathcal{D}|f_{\boldsymbol{\theta}})}\right].
    \end{align*}
For the approximation, we use the idea from LA to use the second-order Taylor polynomial around the optimal parameters $\boldsymbol{\hat{\theta}}$ and replace the Hessian by the Fisher matrix:
    \begin{align*}
        &\frac{1}{N \ln{2}} \mathbb{E}_{\boldsymbol{\theta} \sim \rho}\left[- \ln{p(\mathcal{D}|f_{\boldsymbol{\theta}})}\right] \\
        &\approx \frac{1}{N \ln{2}} \mathbb{E}_{\boldsymbol{\theta} \sim \rho}\left[- \ln{p(\mathcal{D}|f_{\boldsymbol{\hat{\theta}}})} - {\mathcal{D}\boldsymbol{\theta}}^T (\boldsymbol{\theta} - \boldsymbol{\hat{\theta}}) + \frac{1}{2}(\boldsymbol{\theta} - \boldsymbol{\hat{\theta}})^T \mathbf{F} (\boldsymbol{\theta} - \boldsymbol{\hat{\theta}})\right].
    \end{align*}
In contrast to LA, we don't use the Taylor approximation on the posterior but on the likelihood. Therefore, the quadratic term only includes the Fisher matrix and not the precision of the prior. In the next step, we can move the expectation in because $\ln{p(\mathcal{D}|f_{\boldsymbol{\hat{\theta}}})}$ is independent of $\boldsymbol{\theta}$ and $\mathbb{E}_{\boldsymbol{\theta} \sim \rho}\left[{\mathcal{D}\boldsymbol{\theta}}^T (\boldsymbol{\theta} - \boldsymbol{\hat{\theta}})\right] = {\mathcal{D}\boldsymbol{\theta}}^T (\mathbb{E}_{\boldsymbol{\theta} \sim \rho}\left[\boldsymbol{\theta}\right] - \boldsymbol{\hat{\theta}}) = {\mathcal{D}\boldsymbol{\theta}}^T (\boldsymbol{\hat{\theta}} - \boldsymbol{\hat{\theta}}) = 0$:
    \begin{align*}
        &\frac{\mathbb{E}_{\boldsymbol{\theta} \sim \rho}\left[- \ln{p(\mathcal{D}|f_{\boldsymbol{\hat{\theta}}})} - {\mathcal{D}\boldsymbol{\theta}}^T (\boldsymbol{\theta} - \boldsymbol{\hat{\theta}}) + \frac{1}{2}(\boldsymbol{\theta} - \boldsymbol{\hat{\theta}})^T \mathbf{F} (\boldsymbol{\theta} - \boldsymbol{\hat{\theta}})\right]}{N \ln{2}} = \frac{- \ln{p(\mathcal{D}|f_{\boldsymbol{\hat{\theta}}})} + \frac{1}{2}\mathbb{E}_{\boldsymbol{\theta} \sim \rho}\left[(\boldsymbol{\theta} - \boldsymbol{\hat{\theta}})^T \mathbf{F} (\boldsymbol{\theta} - \boldsymbol{\hat{\theta}})\right]}{N \ln{2}}
    \end{align*}
Finally, we can use that the posterior is a normal distribution $\rho = \mathcal{N}(\boldsymbol{\hat{\theta}}, \mathbf{\hat{F}}^{-1})$, where $\mathbf{\hat{F}}$ is a block-diagonal matrix, where each block is given by $\mathbf{\hat{F}}_l = \frac{1}{\mathbf{\tau}_l} (\mathbf{\beta}_l \mathbf{F}_l + \mathbf{\alpha}_l \mathbf{\tilde{F}}_l)$. Therefore, we can reformulate the expectation of the quadratic term as a trace
    \begin{align*}
        \frac{- \ln{p(\mathcal{D}|f_{\boldsymbol{\hat{\theta}}})} + \frac{1}{2}\mathbb{E}_{\boldsymbol{\theta} \sim \rho}\left[(\boldsymbol{\theta} - \boldsymbol{\hat{\theta}})^T \mathbf{F} (\boldsymbol{\theta} - \boldsymbol{\hat{\theta}})\right]}{N \ln{2}} = \frac{- \ln{p(\mathcal{D}|f_{\boldsymbol{\hat{\theta}}})} + \frac{1}{2} \sum_{l \in [L]} \mathbf{\tau}_l \trace \left(\mathbf{F}_l (\mathbf{\beta}_l \mathbf{F}_l + \mathbf{\alpha}_l \mathbf{\tilde{F}}_l)^{-1}\right)}{N \ln{2}} =: \aer(\mathbf{\alpha}, \mathbf{\beta}, \mathbf{\tau}).
    \end{align*}
This approximation is only dependent on quantities that were already computed during the LA, such as the Fisher matrix and the optimal parameters, or that can be computed without extra effort during the computation of the LA like the negative data log-likelihood for the optimal parameters.

\paragraph{KL-divergence}
The PAC-Bayes bounds are not only dependent on the expected empirical risk but also on the KL-divergence between the prior and posterior. Since we use LA for both, this boils down to the KL-divergence between two multivariate normal distributions which can be computed in closed form. For the prior $\rho=\mathcal{N}(\boldsymbol{\tilde{\theta}}, \mathbf{\tilde{F}}^{-1})$ and posterior $\mathcal{N}(\boldsymbol{\hat{\theta}}, \mathbf{\hat{F}}^{-1})$ defined as above, the KL-divergence can be computed as
    \begin{align}
        \mathbb{KL}(\rho\|\pi) &= \mathbb{KL}(\mathcal{N}(\boldsymbol{\hat{\theta}}, \mathbf{\hat{F}}^{-1})\|\mathcal{N}(\boldsymbol{\tilde{\theta}}, \mathbf{\tilde{F}}^{-1})) \nonumber\\
        &= \frac{1}{2} \left(\trace(\mathbf{\tilde{F}} \mathbf{\hat{F}}^{-1}) - n - \ln{\frac{\det{\mathbf{\tilde{F}}}}{\det{\mathbf{\hat{F}}}}} + (\boldsymbol{\hat{\theta}} - \boldsymbol{\tilde{\theta}})^T \mathbf{\tilde{F}} (\boldsymbol{\hat{\theta}} - \boldsymbol{\tilde{\theta}})\right)\nonumber\\
        &= \frac{1}{2} \left(\sum_{l \in [L]}\trace(\mathbf{\tilde{F}}_l \mathbf{\tau}_l(\mathbf{\beta}_l \mathbf{F}_l + \mathbf{\alpha}_l \mathbf{\tilde{F}}_l)^{-1}) - n_l - \ln{\frac{\det{\mathbf{\tilde{F}}_l}}{\det(\frac{1}{\mathbf{\tau}_l}(\mathbf{\beta}_l \mathbf{F}_l + \mathbf{\alpha}_l \mathbf{\tilde{F}}_l))}} + (\boldsymbol{\hat{\theta}}_l - \boldsymbol{\tilde{\theta}}_l)^T \mathbf{\tilde{F}}_l (\boldsymbol{\hat{\theta}}_l - \boldsymbol{\tilde{\theta}}_l)\right)\nonumber\\
        &= \frac{1}{2} \left(\sum_{l \in [L]}\mathbf{\tau}_l\trace(\mathbf{\tilde{F}}_l (\mathbf{\beta}_l \mathbf{F}_l + \mathbf{\alpha}_l \mathbf{\tilde{F}}_l)^{-1}) - n_l (1 + \ln{\mathbf{\tau}_l}) - \ln{\frac{\det{\mathbf{\tilde{F}}_l}}{\det(\mathbf{\beta}_l \mathbf{F}_l + \mathbf{\alpha}_l \mathbf{\tilde{F}}_l)}} + (\boldsymbol{\hat{\theta}}_l - \boldsymbol{\tilde{\theta}}_l)^T \mathbf{\tilde{F}}_l (\boldsymbol{\hat{\theta}}_l - \boldsymbol{\tilde{\theta}}_l)\right) \nonumber\\
        &=: \kl(\mathbf{\alpha}, \mathbf{\beta}, \mathbf{\tau}).\label{eq:kl}
    \end{align}
Similar to the approximation of the expected empirical error, all relevant quantities to compute the KL-divergence are already given after the LA.

\paragraph{PAC-Bayes Objectives}
Combining the approximation of the expected empirical error $\aer(\mathbf{\alpha}, \mathbf{\beta}, \mathbf{\tau})$ and the KL-divergence $\kl(\mathbf{\alpha}, \mathbf{\beta}, \mathbf{\tau})$, we get approximations of the McAllester bound~\citep{guedj2019primer}:
\begin{align*}
    ma(\mathbf{\alpha}, \mathbf{\beta}, \mathbf{\tau}) = \aer(\mathbf{\alpha}, \mathbf{\beta}, \mathbf{\tau}) + \sqrt{\frac{\kl(\mathbf{\alpha}, \mathbf{\beta}, \mathbf{\tau}) + \ln{\frac{2\sqrt{N}}{\varepsilon}}}{2 N}},
\end{align*}
and Catoni bound~\citep{catoni2007pac}:
\begin{align*}
    ca(\mathbf{\alpha}, \mathbf{\beta}, \mathbf{\tau}) = \inf_{c > 0} \frac{1 - \exp(-c \aer(\mathbf{\alpha}, \mathbf{\beta}, \mathbf{\tau}) - \frac{\kl(\mathbf{\alpha}, \mathbf{\beta}, \mathbf{\tau}) - \ln{\varepsilon}}{N})}{1 - \exp(-c)}.
\end{align*}
These objectives can be evaluated and minimized without going through any data samples.

\subsubsection{Frequentist Projection}
\label{subsubsec:frequentist_projection}
In addition to the curvature scaling, we also propose frequentist projection, which also uses approximations of the bounds to optimize the network parameters. This has the goal of further minimizing the PAC-Bayesian bounds also with the choice of the parameters and not only with the curvature scaling. Nonetheless, as the Fisher matrix is only available after the weights are found, we neglect the terms that depend on the curvature. For the McAllester Bound, this results in the optimization problem
\begin{align*}
    \min_{\boldsymbol{\theta}} - \frac{\ln{p(\mathcal{D}|f_{\boldsymbol{\theta}})}}{\ln(2) N} + \sqrt{\frac{\frac{1}{2 \tau} (\boldsymbol{\theta} - \boldsymbol{\tilde{\theta}})^T \mathbf{\tilde{F}} (\boldsymbol{\theta} - \boldsymbol{\tilde{\theta}}) + \ln(\frac{2 N}{\varepsilon})}{2 N}}.
\end{align*}
The Catoni Bound is difficult to optimize with variations of stochastic gradient descent because of the exponential in the objective function. Nonetheless, when the Catoni scale $c > 0$ is fixed, minimizing the upper bound of the Catoni Bound is equivalent to calculating
\begin{align*}
    \min_{\boldsymbol{\theta}} - c \frac{\ln{p(\mathcal{D}|f_{\boldsymbol{\theta}})}}{\ln(2) N} + \frac{\frac{1}{2 \tau} (\boldsymbol{\theta} - \boldsymbol{\tilde{\theta}})^T \mathbf{\tilde{F}} (\boldsymbol{\theta} - \boldsymbol{\tilde{\theta}}) - \ln{\varepsilon}}{N}.
\end{align*}
Therefore, we heuristically alternate between optimizing $\boldsymbol{\theta}$ and $c$ in practice, where the optimization after the Catoni scale is done using the objective
\begin{align*}
    \min_{c > 0} \frac{1 - \exp(c \frac{\ln{p(\mathcal{D}|f_{\boldsymbol{\theta}})}}{\ln(2) N} - \frac{\frac{1}{2 \tau} (\boldsymbol{\theta} - \boldsymbol{\tilde{\theta}})^T \mathbf{\tilde{F}} (\boldsymbol{\theta} - \boldsymbol{\tilde{\theta}}) - \ln{\varepsilon}}{N})}{1 - \exp(-c)}.
\end{align*}

\subsubsection{Extension to BPNNs}
For BPNNs, we also have to consider the parameters and posteriors from previous distributions. Therefore, we present the adaptions for BPNNs in this section. Since we still use the same approximations as in the transfer learning setup, we need to estimate the upper bound of the expected empirical risk and the KL-divergence as a function of the curvature scales. We reuse the notation from \cref{appendix:bpnn_overview}. Hence, we denote the correspondence to a given column by a superscript. In addition, we denote the precision matrix of the posterior with a tilde, \ie $\mathbf{\tilde{F}}^{(i, t)}_l = \frac{1}{\mathbf{\tau}_l} (\mathbf{\beta}^{(i, t)}_l \mathbf{F}^{(i, t)}_l + \mathbf{\alpha}^{(i, t)}_l \mathbf{\tilde{F}}^{(i, i)}_l)$ for a lateral connection at layer $l$ from column $i$ to column $j$. One can see from $\mathbf{\tilde{F}}^{(i, i)}_l$ that we use the posterior from the column $i$ as the prior for this lateral connection.

\paragraph{Approximation of the Expected Empirical Error}
To compute the approximate upper bound of the expected empirical error, we need to compute the expected empirical error on the training data as in \cref{eq:approximation}. For the proposed continual learning architecture, we can compute the upper bound of the expected empirical error and its approximation with
    \begin{align*}
        &\mathbb{E}_{\boldsymbol{\theta}^{(\leq t)} \sim \rho^{(\leq t)}}\left[\frac{1}{N_t} \sum_{i=1}^{N_t} \er(f_{\boldsymbol{\theta}^{(\leq t)}}, \mathbf{x}_i^{(t)}, \mathbf{y}_i^{(t)})\right]\\
        &\leq \frac{1}{N_t \ln{2}} \mathbb{E}_{\boldsymbol{\theta}^{(\leq t)} \sim \rho^{(\leq t)}}\left[- \ln{p(\mathcal{D}_t|f_{\boldsymbol{\theta}^{(\leq t)}})}\right]\\
        &\approx \frac{1}{N_t \ln{2}} \left(- \ln{p(\mathcal{D}_t|f_{\boldsymbol{\hat{\theta}}^{(\leq t)}})} + \frac{1}{2}\mathbb{E}_{\boldsymbol{\theta}^{(\leq t)} \sim \rho^{(\leq t)}}\left[(\boldsymbol{\theta}^{(\leq t)} - \boldsymbol{\hat{\theta}}^{(\leq t)})^T \mathbf{F}^{(\leq t)} (\boldsymbol{\theta}^{(\leq t)} - \boldsymbol{\hat{\theta}}^{(\leq t)})\right]\right),
    \end{align*}
where $\mathbf{F}^{(\leq t)}$ denotes the Fisher matrix on task $t$ for all weights up to column $t$. The quadratic term in the expectation can be computed exactly, \ie, we can use $\rho^{(\leq t)}$ as a normal distribution with mean $\boldsymbol{\hat{\theta}}^{(\leq t)}$. Moreover, we can decompose this into:
    \begin{align*}
        &\mathbb{E}_{\boldsymbol{\theta}^{(\leq t)} \sim \rho^{(\leq t)}}\left[(\boldsymbol{\theta}^{(\leq t)} - \boldsymbol{\hat{\theta}}^{(\leq t)})^T \mathbf{F}^{(\leq t)} (\boldsymbol{\theta}^{(\leq t)} - \boldsymbol{\hat{\theta}}^{(\leq t)})\right] = \trace(\mathbf{F}^{(\leq t)} (\mathbf{\tilde{F}}^{(\leq t)})^{-1})\\
        &= \sum_{j=1}^{t-1} \sum_{i=1}^{j} \trace(\mathbf{F}^{(i, j)} (\mathbf{\tilde{F}}^{(i, j)})^{-1}) + \sum_{i=1}^{t} \trace(\mathbf{F}^{(i, t)} (\mathbf{\tilde{F}}^{(i, t)})^{-1}).
    \end{align*}
The weights from previous columns, \ie for $i \leq j < t$, are frozen and their corresponding posterior together with their curvature scales are fixed. Therefore, the first term is constant with respect to the new curvature scales. To avoid computing the Fisher matrix with respect to previous network weights, we reuse the fixed Fisher matrix from the corresponding task of the given column. This simplifies the upper bound:
    \begin{align*}
        \sum_{j=1}^{t-1}& \sum_{i=1}^{j} \trace(\mathbf{F}^{(i, j)} (\mathbf{\tilde{F}}^{(i, j)})^{-1}) + \sum_{i=1}^{t} \trace(\mathbf{F}^{(i, t)} (\mathbf{\tilde{F}}^{(i, t)})^{-1})\\
        \approx& \text{tr}^{(\leq t-1)} + \sum_{i=1}^{t} \trace(\mathbf{F}^{(i, t)} (\mathbf{\tilde{F}}^{(i, t)})^{-1}),\\
        =& \text{tr}^{(\leq t-1)} + \sum_{i=1}^{t-1} \sum_{l \in \mathfrak{L}} \trace(\mathbf{F}^{(i, t)}_l \mathbf{\tau}_l (\mathbf{\beta}^{(i, t)}_l \mathbf{F}^{(i, t)}_l + \mathbf{\alpha}^{(i, t)}_l \mathbf{\tilde{F}}^{(i, i)}_l)^{-1}) + \sum_{l = 1}^{L} \trace(\mathbf{F}^{(t, t)}_l \mathbf{\tau}_l (\mathbf{\beta}^{(t, t)}_l \mathbf{F}^{(t, t)}_l + \mathbf{\alpha}^{(t, t)}_l \mathbf{\tilde{F}}^{(0)}_l)^{-1}).\\
    \end{align*}
Here, we introduce the trace for all fixed previous columns and Fisher matrices as $\text{tr}^{(\leq t-1)}$. This term is not dependent on the curvature scales. Moreover, the set of the lateral connections is $\mathfrak{L}$.

\paragraph{KL-divergence}
With a block-diagonal covariance matrix, we assume that the weights of different layers are independent. Using this, the KL-divergence between the posterior and the prior can be written as
\begin{align}
    \mathbb{KL}\left(\rho^{(\leq t)}\|\pi^{(\leq t)}\right) &= \sum_{j = 1}^{t} \sum_{i = 1}^{j} \mathbb{KL}\left(\rho^{(i, j)}\|\pi^{(i, j)}\right)\nonumber\\
    &= \sum_{i = 1}^{t} \mathbb{KL}\left(\rho^{(i, t)}\|\pi^{(i, t)}\right) + \sum_{j = 1}^{t-1} \sum_{i = 1}^{j} \mathbb{KL}\left(\rho^{(i, j)}\|\pi^{(i, j)}\right)\nonumber\\
    &= \sum_{i = 1}^{t} \mathbb{KL}\left(\rho^{(i, t)}\|\pi^{(i, t)}\right)\label{eq:kl_fixed}\\
    &= \sum_{l \in [L]}\mathbb{KL}\left(\rho^{(t, t)}_l\|\mathcal{N}(\boldsymbol{\hat{\theta}}^{(0)}_l, (\mathbf{\tilde{F}}^{(0)}_l)^{-1})\right) + \sum_{l \in \mathfrak{L}}\sum_{i = 1}^{t-1} \mathbb{KL}\left(\rho^{(i, t)}_l\|\rho^{(i, i)}_l\right),\nonumber
\end{align}
where it was used that the prior is equal to the posterior for all fixed columns in \cref{eq:kl_fixed}. We approximate the posterior with LA and curvature scaling. Therefore, we can compute the KL-divergence again in closed form:
\begin{align}
    &\mathbb{KL}(\rho^{(\leq t)}\|\pi^{(\leq t)}) = \sum_{l \in [L]}\mathbb{KL}\left(\mathcal{N}(\boldsymbol{\hat{\theta}}^{(t, t)}_l, \mathbf{\tau}_l (\mathbf{\beta}^{(t, t)}_l \mathbf{F}^{(t, t)}_l + \mathbf{\alpha}^{(t, t)}_l \mathbf{\tilde{F}}^{(0)}_l)^{-1})\|\mathcal{N}(\boldsymbol{\hat{\theta}}^{(0)}_l, (\mathbf{\tilde{F}}^{(0)}_l)^{-1})\right)\\
    &\hphantom{\mathbb{KL}(\rho^{(\leq t)}\|\pi^{(\leq t)})} + \sum_{l \in \mathfrak{L}}\sum_{i = 1}^{t-1} \mathbb{KL}\left(\mathcal{N}(\boldsymbol{\hat{\theta}}^{(i, t)}_l, \mathbf{\tau}_l (\mathbf{\beta}^{(i, t)}_l \mathbf{F}^{(i, t)}_l + \mathbf{\alpha}^{(i, t)}_l \mathbf{\tilde{F}}^{(i, i)}_l)^{-1})\|\mathcal{N}(\boldsymbol{\hat{\theta}}^{(i, i)}_l, (\mathbf{\tilde{F}}^{(i, i)}_l)^{-1})\right)
\end{align}
Altogether, this shows the advantage of BPNN that some related feature embeddings can increase the performance even though only the current column contributes to the KL-divergence.

\section{Some thoughts on the regression}
\label{subsec:regression:comments}

In this section, we share potential extensions of the given PAC-Bayes framework in LA, to regression problems. PAC-Bayes traditionally assumed bounded losses like in classification problems, while in recent years, the theory is being also extended to unbounded loss functions, as in regression problems. While we see this extension as beyond the scope of this work, we comment some thoughts on the regression case. For this, let us start with a general bound from \citet{alquier2016properties}:

\begin{definition}[\citet{alquier2016properties}]
\label{alquier:original}
Given a distribution $\mathcal{D}$ over $\mathcal{X} \times \mathcal{Y}$, a hypothesis set $\mathcal{G}$, a loss function $l: \mathcal{G}\times \mathcal{X} \times \mathcal{Y}$, a prior distribution $\pi$ over $\mathcal{G}$, a $\delta \in (0,1]$ and a real number $\gamma > 0$, with probability at least $1-\delta$ over the choice of $(X,Y) \sim \mathcal{D}^n$, we have
    \begin{align}
    \forall {\rho} \ \ \text{on} \ \ \mathcal{G}: \quad \mathbb{E}_{g \sim \rho} \mathcal{L}^l_\mathcal{D}(g) \leq \mathbb{E}_{g \sim \rho} \hat{\mathcal{L}}^l_{X,Y}(g) + \frac{1}{\lambda}\left [ \text{KL}(\rho||\pi) + \text{ln}\frac{1}{\delta} + \Psi_{l,\pi,\mathcal{D}}(\lambda, n)  \right ], & \\
    \text{where} \quad \Psi_{l,\pi,\mathcal{D}}(\lambda, n) = \text{ln } \mathbb{E}_{g \sim \pi} \mathbb{E}_{X',Y' \sim \mathcal{D}^n} \text{ exp }\left [ \lambda \left ( \mathcal{L}_\mathcal{D}^l(g) - \hat{\mathcal{L}}_{X',Y'}^l(g) \right ) \right ]. &
    \end{align}
\end{definition}
A special case for the unbounded loss functions have been examined in \citet{germain2016pac}. The results showed that, for regression problems, the Alquier bound holds for certain classes of loss functions. One of them is the so-called sub-gamma losses. Sub-gamma losses, given variance factor $s^2$ and scale parameter c, has a variable $\psi_V(\lambda):= \text{ln} \mathbb{E}e^{\lambda V}$ with $V=\mathcal{L}_{\mathcal{D}}^l(g) - l(g,x,y)$:
    \begin{equation}
        \psi_V(\lambda) \leq \frac{s^2}{c^2}\left ( -\text{ln}(1-\lambda c) \right ) \leq \frac{\lambda^2 s^2}{2(1-c\lambda)}, \quad \lambda \in (0, \frac{1}{c}).
    \end{equation}
For these types of losses, the PAC-Bayes bound can be obtained as shown in the theorem below.
\begin{definition}[\citet{germain2016pac}]
Given $D, \mathcal{G}, l, \pi$ and $\delta$ defined in the statement of the above theorem (definition \ref{alquier:original}), if the loss is sub-gamma with variance factor $s^2$ and scale $c < 1$ we have, with probability at least $1-\delta$ over $(X,Y) \sim \mathcal{D}^n$,
    \begin{align}
    \forall \rho \ \ \text{on} \ \ \mathcal{G}: \quad \mathbb{E}_{g \sim \rho} \mathcal{L}^l_\mathcal{D}(g) \leq \mathbb{E}_{g \sim \rho} \hat{\mathcal{L}}^l_{X,Y}(g) + \frac{1}{N}\left [ \text{KL}(\rho||\pi) + \text{ln}\frac{1}{\delta} \right ] + \frac{s^2}{2(1-c)}. & 
    \end{align}
\end{definition}
One concrete example is the case for Bayesian linear regression. Here, our model is: $f_{\boldsymbol{\theta}}(\rvx) = \boldsymbol{\theta} \phi(\rvx)$ with the modelling choices of: $\boldsymbol{\theta} \sim \mathcal{N}(\mathbf{0},\sigma_\pi^2\mathbf{I})$ and $\phi(\rvx) \sim \mathcal{N}(\mathbf{0}, \sigma_x^2 \mathbf{I})$. Also, we add white noise: $\rvy = \hat{\boldsymbol{\theta}} \rvx + \epsilon$ where $\epsilon \sim \mathcal{N}(\mathbf{0}, \sigma_{\epsilon}^2)$, resulting in $p(\rvy|\rvx) \sim \mathcal{N}(\mathbf{\Phi} (\rvx) \cdot \hat{\boldsymbol{\theta}}, \sigma_{\epsilon}^2)$. As in \citet{bishop2009pattern}, $p(\boldsymbol{\theta}|\mathcal{D},\sigma,\sigma_\pi) = \mathcal{N}( \boldsymbol{\theta}|\hat{ \boldsymbol{\theta}}, \mathbf{A}^{-1})$ where $\mathbf{A} := \frac{1}{\sigma^2}\mathbf{\Phi}^T\mathbf{\Phi} + \frac{1}{\sigma_\pi^2}I$. Under these settings and using negative log likelihood, \citet{germain2016pac} shows:
%    \begin{align}
%        & s^2 \geq 2 \left [ \sigma_x^2(\sigma_\pi^2d + \left \| \hat{\boldsymbol{\theta}} \right \|^2) + \sigma_\epsilon^2(1-c)    \right ] \quad \text{and} \quad c \geq 2 \sigma^2_x \sigma^2_\pi,
%    \end{align}
%for the variables of the sub-gamma loss function. Choosing negative log likelihood alternatively, \citet{germain2016pac} shows:
    \begin{align}
        & s^2 \geq \frac{1}{\lambda \sigma^2}\left [ \sigma_x^2(\sigma_\pi^2d + \left \| \hat{\boldsymbol{\theta}} \right \|^2) + \sigma_\epsilon^2(1-\lambda c) \right ] \quad \text{and} \quad c \geq \frac{1}{\sigma^2}(\sigma_x^2 \sigma_\pi^2).
    \end{align}
Moreover, under Gaussian assumptions, the Alquier bound is closed, e.g., $N \mathbb{E}_{\boldsymbol{\theta} \sim \rho} \hat{\mathcal{L}}_{X,Y}^{l_{\text{nll}}}(\boldsymbol{\theta}) = \mathbb{E}_{\boldsymbol{\theta} \sim \rho}\sum_{i=1}^{n}-\text{ln} p(\rvy_i|\rvx_i, \boldsymbol{\theta})= N \hat{\mathcal{L}}_{X,Y}^{l_{\text{nll}}}(\boldsymbol{\theta}) + \frac{1}{2\sigma^2}\text{tr}(\mathbf{\Phi}^T \mathbf{\Phi}  \mathbf{A}^{-1})$, and also the KL-divergence term between two Gaussian distributions. These results are given generalized linear models, while our paper is about neural networks. On the other hand, if neural networks can be approximated with linear regression in some sense, we should intuitively be able to exploit the results so far. Hence, we next make the connections between BNNs and linear regression via LA.

To start with, let us denote $p(\boldsymbol{\theta} | \mathcal{D}) \approx \mathcal{N}(\boldsymbol{\theta} | \hat{\boldsymbol{\theta}}, \mathbf{\Sigma})$ as the posterior distribution of BNNs, which is obtained using LA, \ie, $\mathbf{\Sigma}^{-1} = \sum_{i=1}^{N}\triangledown_{\boldsymbol{\theta} \boldsymbol{\theta}}^2 l_i(\hat{\boldsymbol{\theta}}) + \delta \mathbf{I}$ with $\triangledown_{\boldsymbol{\theta} 
 \boldsymbol{\theta}}^2 l(\boldsymbol{\theta}) \approx \mathbf{J}_{\boldsymbol{\theta}}(\rvx)^T \mathbf{\Lambda}_{\boldsymbol{\theta}}(\rvx,\rvy)\mathbf{J}_{\boldsymbol{\theta}}(\rvx)$. Here, $\mathbf{J}_{\boldsymbol{\theta}}(\rvx)$ is the first derivative of the output w.r.t the weights or jacobians of neural networks and $\mathbf{\Lambda}_{\boldsymbol{\theta}}(\rvx,\rvy)$ is the output noise precision term. These quantities are defined at $\theta^*$, and additionally, we define the residual term $\boldsymbol{r}(\rvx_i,\rvy_i)$ similarly. We note that $\triangledown_{\boldsymbol{\theta}  \boldsymbol{\theta}}^2 l(\boldsymbol{\theta})$ here represents an output space formulation of the Hessian, while previously, we described the weight space formulation for the Hessian. These are different ways of representing the Hessin in neural networks. Then, a transformed data-set can be defined: $\mathcal{\tilde{D}} = \left\{ (\rvx_i,\tilde{\rvy}_i) \right\}_{i=1}^N$ where $\tilde{\rvy}_i :=  \mathbf{J}(\rvx_i) \theta^* - \mathbf{\Lambda}(\rvx_i, \rvy_i)^{-1} \boldsymbol{r}(\rvx_i,\rvy_i)$. Given at the mode, assuming that the residual term is close to zero, \ie, the model fits the data well, a linear model that fits this data-set $\mathcal{\tilde{D}}$ as a subspace, can be defined:
    \begin{equation}
    \label{eq:neural:linear:model}
        \widetilde{\rvy} = \mathbf{J}(\rvx) \boldsymbol{\theta} + \epsilon \text{ where } \epsilon \sim \mathcal{N}\left (\mathbf{0}, \mathbf{\Lambda}^{-1}(\rvx, \rvy)  \right ) \text{ and } \boldsymbol{\theta} \sim \mathcal{N}(\mathbf{0}, \delta^{-1} \mathbf{I}).
    \end{equation}
This alternative LA-based BNN formulation has interesting implication, as \citet{khan2019approximate} shows that the posterior of this linear model is the same as that of original BNNs $p(\boldsymbol{\theta} | \mathcal{D}) \approx \mathcal{N}(\boldsymbol{\theta} | \hat{\boldsymbol{\theta}}, \mathbf{\Sigma})$:

\begin{definition}[\citet{khan2019approximate}]
The Laplace approximation of a neural network posterior $p(\boldsymbol{\theta} | \mathcal{D})$ is equal to the posterior distribution $p(\boldsymbol{\theta}|\mathcal{\tilde{D}})$.
\end{definition}

This means that we may then bring these results together in order to obtain a valid PAC-Bayes bound. Overall, there exists a valid PAC-Bayes bound for Bayesian linear regression models which can be obtained in closed form for tractable optimization. As with LA, a neural network has a linear subspace defined by a Bayesian linear model (equation \ref{eq:neural:linear:model}), one could also analyze LA-based BNNs like \citep{lee2021trust} using PAC-Bayes bound, and consequently the curvature scaling. 

%% file: appendix/appendixD.tex
\section{Appendix: Implementation Details}
\label{appendix:implementation:details}

Now, we present the implementation details, which have been used in our experiments. We plan to officially release the code for the community. While implementation details are provided per different experiments, the general setting is as follows. Our software is based on Pytorch when possible, while only in few-shot learning experiments, we also use TensorFlow and JAX in order to utilize the 'uncertainty baselines' repository.  In all our experiments, the Fisher matrix is computed with the K-FOC method~\citep{schnaus2021kronecker}, which is a sub-class of the KFAC method from Section~\ref{sec:method}. Our computing cluster consists of multiple GPUs, from NVIDIA's Pascal to Ampere architecture. No multi-GPU training was used for the results. The largest GPU memory is 24GB, while CPU RAM of 120GB is available for these large GPUs.

\subsection{Implementation Details: Section~\ref{sec:results:ablations:main}}

The implementation details for section~\ref{sec:results:ablations:main} is as follows. We use $90\%$ of the training data for training and $10\%$ for validation. Furthermore, we use the Adam optimizer with a batch size of 256 and a learning rate of $5 \cdot 10^{-4}$. The learning rate is halved if the validation loss does not decrease for five consecutive steps and training is stopped after 100 epochs or if the loss does not improve for ten consecutive steps. We also optimize the curvature scales using Adam with an exponential learning rate decay with a decay factor of $0.999999$ per step. In our experiments, we compare the performance of our learned prior against an isotropic prior with weight decay $10^{-3}$, $10^{-5}$, and $10^{-8}$ either with mean zero or using the pre-trained model as mean. We also use the same weight decay to compute the learned prior. For all Bayesian neural networks, we use $100$ samples from the posterior to estimate the predictive distribution. The PAC-Bayes bounds are evaluated with $\varepsilon=0.1$. Hence, the bounds are satisfied with at least $90\%$ probability.

We note that cross-validation would be a reasonable method to use the entire training set. However, this would mean that with k folds, one would have to retrain the model k times with each hyperparameter configuration. Even with only one separate validation set, grid search would quickly become infeasible when optimizing three parameters per layer, each with multiple values. This is primarily because one would need to sample multiple weights for each validation sample to estimate the predictive distribution, and then repeat this for the entire validation set for each hyperparameter optimization. Meanwhile, for PAC-Bayes objectives, we reuse some of the quantities that can be estimated during the Laplace approximation (the Fisher matrix, the optimal network parameters, and the log-likelihood of the training data using the optimal parameters). In this way, our objectives can be evaluated without additional forward passes. Therefore, the complexity of hyperparameter optimization is independent of the size of the data-set and the complexity of the model forward pass, and thus in practice, its convergence can be faster than grid search. To illustrate this point, we compared to a grid search over $\alpha \in \left\{ 0.01, 0.1, 0.3, 0.5, 0.8, 0.9, 0.99, 1, 2 \right\}$, $\beta \in \left\{ 0.01, 0.1, 0.3, 0.5, 0.8, 0.9, 0.99, 1, 2 \right\}$, and $\tau \in \left\{ 10^{-i}| i \in [26] \right\}$ shared over each layer. The models maximizing the validation accuracy were chosen. Other ablations are similar to the setting of Figure 2a and we have used five different seeds. Thus, curvature scaling and frequentist projection refer to our PAC-Bayes-based approaches. The results are shown in the corresponding table.

We also quantitatively evaluated the quality of the approximation on the sums-of-Kronecker-products. For this, we measured the relative error in terms of a Frobenius norm, where we compared the approximates to the true sums-of-Kronecker product with positive definite matrices. The baseline is the approximation adapted by \citet{ritter2018scalable,ritter2018online}, which assumes $(\mathbf{L} \otimes \mathbf{R} + \gamma \mathbf{I} )^{-1} \approx (\mathbf{L} + \gamma \mathbf{I})^{-1} \otimes (\mathbf{R}+ \gamma \mathbf{I})^{-1}$ (denoted as 'Sum'). We note again that this rule does not hold in general, and therefore, we proposed a power method to better approximate the sum-of-the-Kronecker products as a Kronecker product. We denote our approach as the' power method' here. 

First, we evaluate by generating positive definite matrices $\mathbf{L} \in \mathbb{R}^{M \times M}$, $\mathbf{R} \in \mathbb{R}^{N \times N}$ with $M=N$. One variation can be the size of the matrices from $2$ to $20$. We obtain them by sampling the elements from a standard normal distribution, then multiplying its transposed, and finally adding $10^{-8} \cdot \mathbf{I}$. We also scale the resulting matrix by a scale from $10^{-5} to 10^{4}$ and also use different weight decays $10^{-10}$ to $1$ within $(\mathbf{L} \otimes \mathbf{R} + \gamma \mathbf{I} )$.  For each size, scale, and weight decay, we use further $100$ different random seeds. The results are shown in Figure~\ref{fig:results_ablations:kronecker}. In the plots, we show the mean and the $95\%$ confidence interval of the relative Frobenius error for the two different approximations. Within these settings, we find that the power method has a relative Frobenius error of $0.01020 \pm 0.03129$. On the other hand, computing the sum of each factor as in \citep{ritter2018scalable} by $(\mathbf{L} + \sqrt{\gamma} \mathbf{I}) \otimes (\mathbf{R} + \sqrt{\gamma} \mathbf{I})$, has a relative Frobenius error of $0.10167 \pm 0.22935$. So, our experiments show that the power method yields at least one magnitude smaller in terms of the Frobenius error.

Furthermore, we also evaluate the sums-of-Kronecker-products for more than two matrices. Such evaluation is important for settings such as BPNN, where such computations are needed for each task. To do so, we compute the relative Frobenius error for $\sum_{k=1}^K \mathbf{L}^k \otimes \mathbf{R}^k - \mathbf{L} \otimes \mathbf{R}$ with $K$ ranging from $2$ to $9$. We use $M=N=5$ and sample the matrices similar to the aforementioned way: $\mathbf{L}$ and $\mathbf{R}$ (see Figure~\ref{fig:results_ablations:kronecker}). The baseline is again the 'sum'. As in \citep{ritter2018online}, we sum over more than two Kronecker factored matrices, where the first matrix is an identity matrix, scaled by the weight decay $\gamma$. The results are depicted in Figure~\ref{fig:results_ablations:kronecker}. We observed that the relative error is not only low with the power iteration method but also, the relative error of the baseline 'sum' significantly increases. These results motivate the proposed method to compute the sums-of-Kronecker-products as two Kronecker factors.

\subsection{Implementation Details: Section~\ref{sec:results:continual:learning}}

The implementation details for section~\ref{sec:results:continual:learning} is as follows. Following \citet{grosse2016kronecker}, we do not set a prior on the parameters of batch normalization.  All images are normalized by the mean and standard deviation from the ImageNet data-set and resized to $224 \times 224$. We further use the Adam optimizer with batch size $4$ and learning rate $10^{-4}$. We compare our method to PNNs with weight decays $10^{-1}, 10^{-2}, \dots, 10^{-10}$, PNNs using Monte Carlo dropout~\citep{gal2015bayesian} without weight decay with dropout probability $0.01$. Moreover, the learned prior is compared to an isotropic Gaussian prior around zero and around the pre-trained weights on the ImageNet-1K data-set with weight decay $10^{-8}$. The temperature scaling is inferred after the training of each task with a coarse grid search.

\subsection{Implementation Details: Section~\ref{sec:results:fewshot}}

The implementation details for section~\ref{sec:results:fewshot} is as follows. Similar to the continual learning experiments (section~\ref{sec:results:continual:learning}), we do not set a prior for batch normalization. All images are resized to $224 \times 224$ after normalization. For comparing isotropic priors and learned priors, we use the Adam with a batch size of $4$ and learning rate of $10^{-4}$. On the other hand, for other baselines, we tried four learning rates more: $10^{-2},10^{-3}, 10^{-4}$ and $10^{-5}$ in order to increase the few-shot learning performance. We did not use one-shot learning due to the validity of LA, but rather increased the number of evaluation points. While we use the software of  \citet{tran2022plex} for creating the few-shot learning data-set pairs. On top of, we created separate validation sets from the remaining pool of data with 1:1 ratio to the few-shot training set. For the baselines, we use a deep ensemble with five members. \citet{lakshminarayanan2017simple} showed that for therein presented classification experiments, the five members were able to closely match the performance with more members. The dropout rates in Monte Carlo dropout the same as what is available. In all the experiments, we used $100$ samples from the posterior for computing the predictive distribution. 

%% file: appendix/appendixE.tex
\section{Appendix: Additional Experiments and Results}
\label{appendix:results}

\subsection{Qualitative Analysis of the Approximate Upper Bound}
\label{appendix:results_approximate_bound}

\begin{figure}
    \centering
    \vskip 0.2in
        \begin{center}
            \centerline{\includegraphics[width=0.8\textwidth]{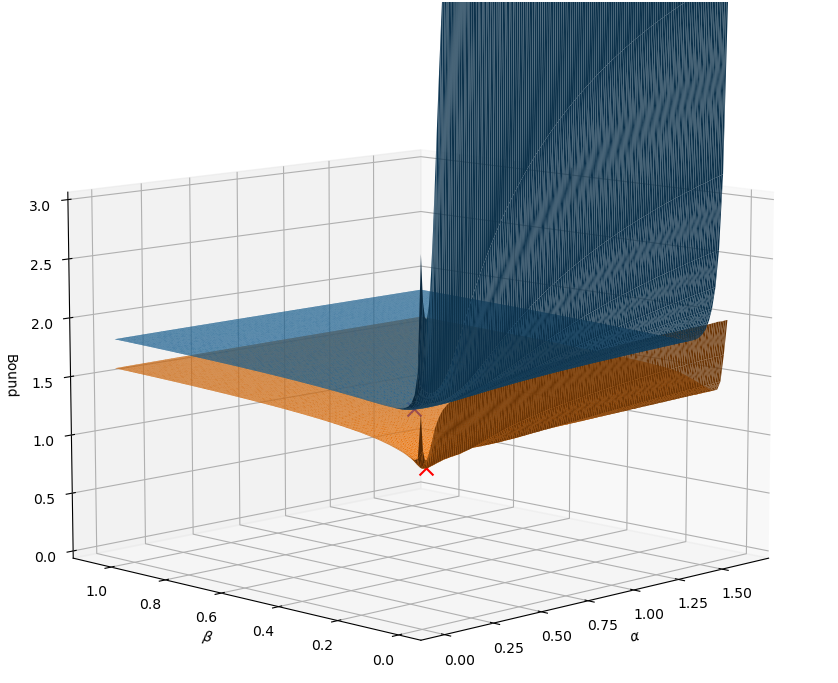}}
            \caption{True bound (orange) Versus the proposed approximation (blue). The results show that the proposed approximation exhibits a similar shape to the true bound, providing qualitative insights about the proposed approximation. Smaller the bound, the better. }
            \label{fig:results:true_vs_approximate}
        \end{center}
    \vskip -0.2in
\end{figure}

Here, we want to qualitatively analyze the quality of the upper bound, and obtain further insights about their validity. To do so, we used the same setting as our presented ablation studies in Section~\ref{sec:results}. One advantage is that the scale is small enough that we can directly compare the true PAC-Bayes bound, and the proposed approximation to that bound. However, unlike other experiments, we share $\alpha$ and $\beta$ across the layers. This was to meaningfully see the shape of the bound in 3D plots. The results are depicted in Figure~\ref{fig:results:true_vs_approximate}. Within this experiment, we observe that the true bound (orange) is tighter, but follows qualitatively similar behavior as the proposed approximation (blue). In particular, we find that the optima (orange cross for a true bound, blue cross for approximation) are close to each other. Therefore, we think this data shows the quality of the proposed approximation at a small scale.

\subsection{Ablations in the Small-Data Regime}
\label{appendix:results_small_data}

\begin{figure}
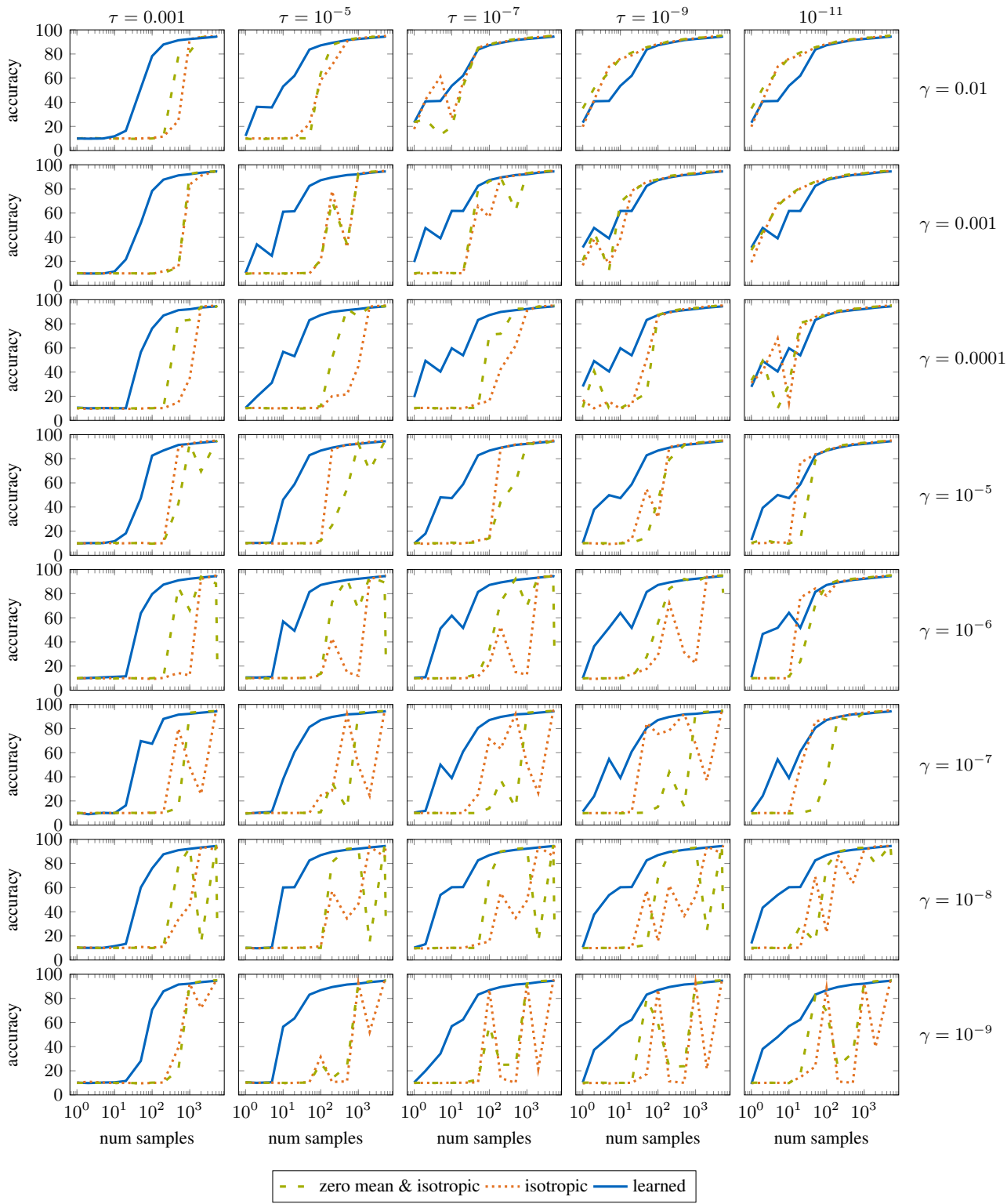

    \centering
    \vskip 0.2in
        \begin{center}
            \centerline{\includestandalone[width=\textwidth]{tikz/appendix/results/results_small_data}}
            \caption{Extension of ablations in the small-data regime via varying two hyperparameters, namely temperature scaling and weight decays. The results empirically show conditions when the learned prior enables more data-efficient learning. Faster the increase in accuracy, the better.}
            \label{fig:results_ablations:appendix:more}
        \end{center}
    \vskip -0.2in
\end{figure}

\begin{figure*}
    \centering
     \begin{subfigure}[b]{0.49\textwidth}
         \centering
         \includestandalone{tikz/appendix/results/cold_posterior_small}
         %\caption{More weight decays to the results in section~\ref{sec:results}.}
         \label{fig:cold_posterior_small:appendix}
     \end{subfigure}
     \hfill
     \begin{subfigure}[b]{0.49\textwidth}
         \centering
         \includestandalone{tikz/appendix/results/cold_posterior_large}
         %\caption{Learned prior from ImageNet and tested on CIFAR10. Maximum reachable accuracy is higher with the learned prior. Learned prior also generates more stable results w.r.t variations in weight decays.}
         \label{fig:cold_posterior_large:appendix}
     \end{subfigure}
    \caption{Right: Extended cold posterior experiments across more weight decay variations. The learned prior produces more stable results w.r.t variations in the weight decay. Left: Further validation with the prior learned from ImageNet. Learned prior from ImageNet and tested on CIFAR-10. Maximum reachable accuracy is higher with the learned prior. Learned prior also generates more stable results w.r.t variations in weight decays.}
    \label{fig:results_ablations:futher_cold_posterior}
\end{figure*}

\begin{figure*}
    \centering
    \includegraphics[width=\textwidth]{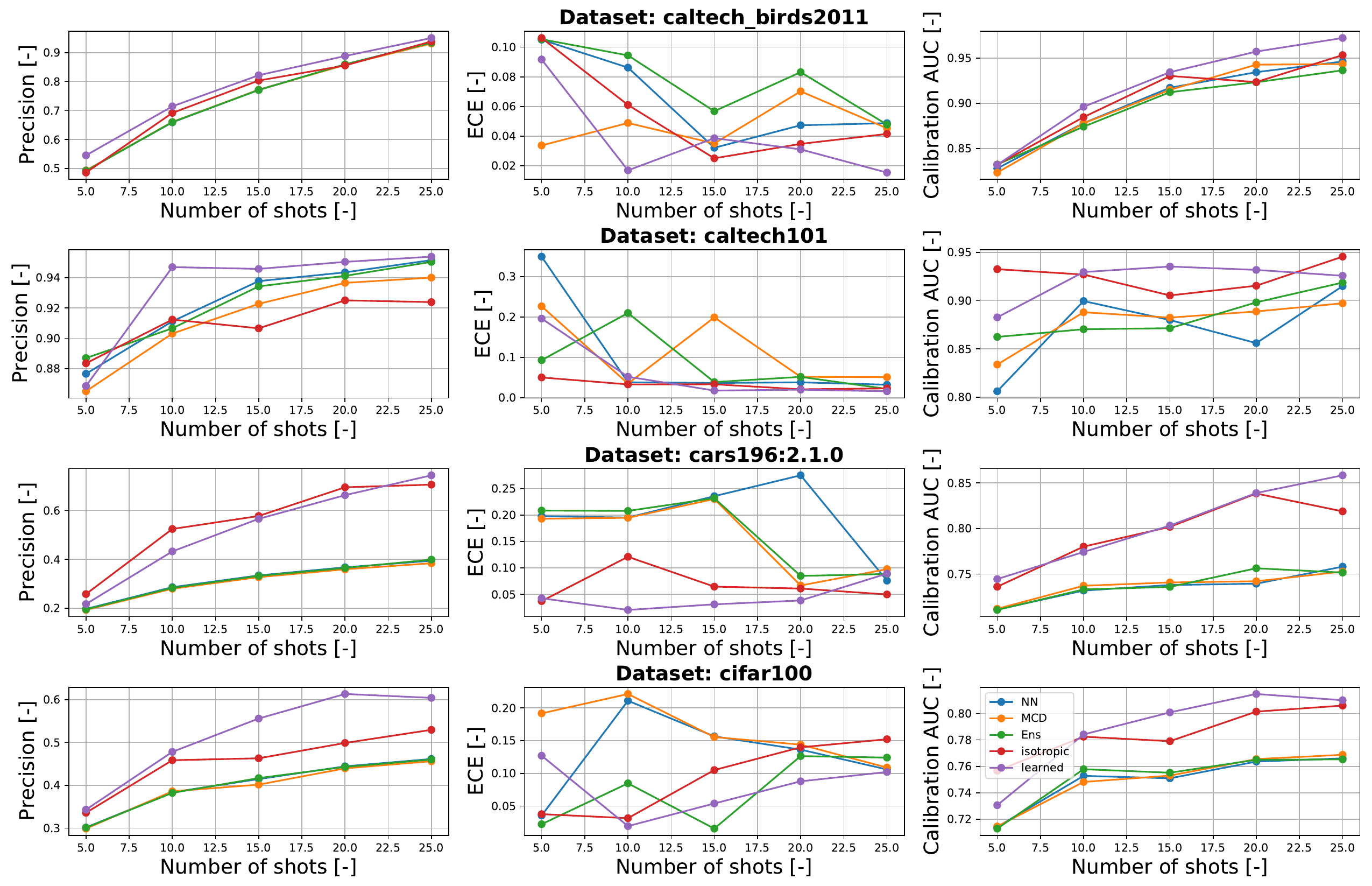}
    \caption{Few-shot learning experiments. Results are presented for each data-set. Higher the better for accuracy and AUC measures, while the lower the better for ECE measures. The results show the high performance of our method in both uncertainty calibration and generalization.}
    \label{fig:results_fewshot:appendix:1}
\end{figure*}

\begin{figure*}
    \centering
    \includegraphics[width=\textwidth]{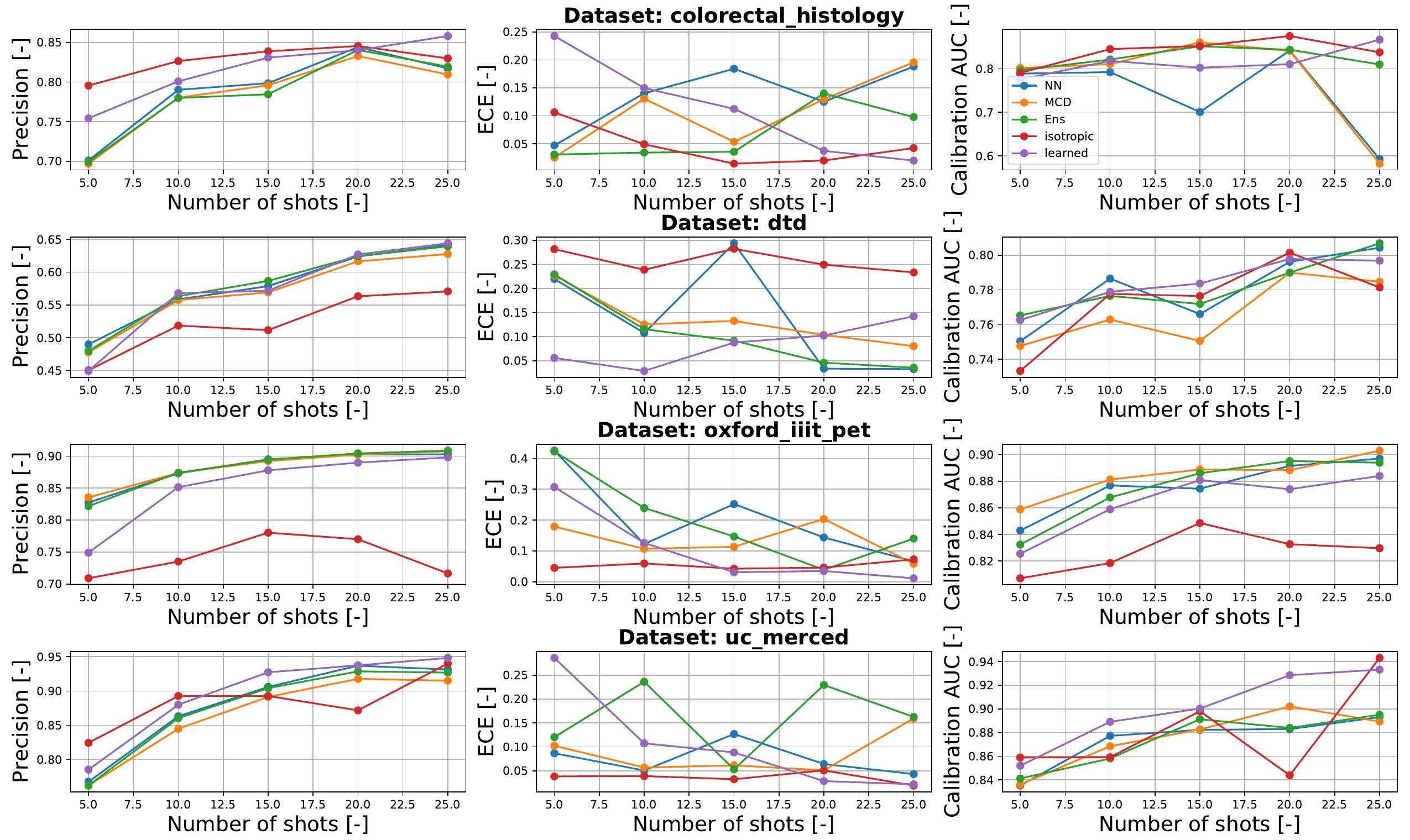}
    \caption{Few-shot learning experiments. Results are presented for each data-set. Higher the better for accuracy and AUC measures, while the lower the better for ECE measures. The results show the high performance of our method in both uncertainty calibration and generalization.}
    \label{fig:results_fewshot:appendix:2}
\end{figure*}

Next, we extend the ablation studies from Section~\ref{sec:results} in order to provide the influence of hyper-parameters namely temperature scaling and weight decays. To achieve this, we keep the same setting as the ablations again and plot for a specified range. Similarly again, as often done in PAC-Bayes literature, we vary the availability of the data and compute the accuracy as a metric. The results can be found in Figure~\ref{fig:results_ablations:appendix:more}. Importantly, this figure shows that the learned prior is in general, more data-efficient. On the other hand, the results also show the conditions when our method is effective. To explain, we observe that for small temperature scaling and large weight decays (top right), the performance of learned prior deteriorates. The results are expected because in the case of large weight decay, the isotropic prior becomes a more dominant term in the learned prior. Thus, the differences become smaller. On the other hand, for a small weight decay, BNNs become closer to deterministic. In that case, learning the prior only helps for small weight decay. Lastly, we find that in a vast range of hyperparameter settings, the learned prior helps to learn more data efficiently, when compared to isotropic priors.

\subsection{Further Cold Posterior Experiments}
\label{appendix:results_cold_posterior}

\subsubsection{Small-Scale}

In Section~\ref{sec:results}, we presented an empirical result on cold posterior effects. There, we showed that the cold posterior effect is not completely vanished, but is largely mitigated with the proposed learned prior. In this section, we aim to examine the stability of these results w.r.t different weight decay settings. To this end, within the setting of the ablations, we examine variations of three weight decays. The results are shown in Figure~\ref{fig:results_ablations:futher_cold_posterior}. In the experiments, we observe similar qualitative behavior, \ie, our learning-based prior largely mitigates the cold posterior effects when compared to the isotropic prior. We also ablate the prior with learned mean, which helps over purely zero mean isotropic Gaussian. Yet, incorporating the curvature, as in our method, improves the results further. Most importantly, our learned prior generates the least deviations due to the changes in the weight decay. Thus, we interpret the results that the learned priors are more stable with respect to changing the weight decay, while better mitigating the cold posterior effects.

\subsubsection{Large-Scale}

Next, we further examine the cold posteriors within a large scale experiment. To do so, we obtain the posteriors from ImageNet using ResNet-50 architecture and use it as a prior on CIFAR-10. Using CIFAR-10, we then obtain the posterior. We use a learning rate of 1e-3 and a batch size of eight. All other settings are kept the same as in the ablations of Section~\ref{sec:results}. We do not evaluate the zero mean $\&$ isotropic prior but only the isotropic prior with the learned mean. This enables direct comparison of the methods, as we only vary the prior. The results are presented in Figure~\ref{fig:results_ablations:futher_cold_posterior}. We observe the followings. First, we find that the cold poster effect is strong for the weight decay 1e-3, whereas with the weight decay 1e-5, the learned prior helps. The weight decay of 1e-8 didn't produce reasonable results for all evaluated temperature scales (up to 1e-15). These results are omitted from the plot for presentation clarity. Finally, the reachable accuracy is higher for the learning-based priors than for the isotropic priors with learned mean, while the influence of random seeds on deviations seems smaller. 

\subsection{Few-Shot Learning: Results per Data-set}
\label{appendix:fewshot:resultsperdata}
In Section~\ref{sec:results}, we have presented few-shot learning results, which were averaged across eight data-sets. Now, we present the results per individual data-set. Hence, the results herein is an elaboration to Section~\ref{sec:results}. Figures~\ref{fig:results_fewshot:appendix:1} and \ref{fig:results_fewshot:appendix:2} depict the results, four different data-sets each. We observe that in the majority of the data-sets, the learned prior outperforms the isotropic priors as well as the presented baselines, in terms of precision and calibration measures. On the other hand, we also observe that the learned prior cannot uniformly outperform all other baselines across different data-set and metrics. For example, the colorectal histology data-set is one where the isotropic prior outperforms the learned prior, while in oxford pet data-set, the standard BNNs and a deterministic neural network show stronger results in terms of precision and calibration AUC. We think that this is expected, as the prior learning methods assume the existence of relevant data and tasks to learn the prior from. Despite this limitation, as we find that the learned prior more often outperforms the methods based on the isotropic priors, our work shows the relevance of the prior learning methods, \ie, when there exists relevant data and tasks to learn from, there are potential performance advantages over relying on the isotropic priors.

%% file: tikz/appendix/results/results_small_data.tex
\begin{tikzpicture}
\begin{groupplot}[
footnotesize,
width=0.25\textwidth,
group style={
group size=5 by 8, 
y descriptions at=edge left, 
x descriptions at=edge bottom,
horizontal sep=.25cm,
vertical sep=.25cm
},
xlabel={num samples},
xmin=0.650699258757847, xmax=8298.76464022464,
xmode=log,
log basis x={10},
ylabel={accuracy},
ymin=0, ymax=100,
]
\nextgroupplot[
title={$\tau = 0.001$},
legend to name={LegendSmallData},
legend style={legend columns=3}, 
]
\addplot [very thick, TUMAccentGreen, loosely dashed]
table {%
1 9.86999988555908
2 10.6099996566772
5 10
10 9.98999977111816
20 9.78999996185303
50 9.64000034332275
100 10.0200004577637
200 11.1700000762939
500 78.379997253418
1000 81.8099975585938
2000 93.8399963378906
5000 95.3300018310547
5400 95.2900009155273
};
\addlegendentry{zero mean \& isotropic}
\addplot [very thick, TUMAccentOrange, dotted]
table {%
1 10.0299997329712
2 9.98999977111816
5 10.1099996566772
10 9.97999954223633
20 9.98999977111816
50 9.94999980926514
100 10.5
200 11.5799999237061
500 24.7099990844727
1000 92.3899993896484
2000 94.0599975585938
5000 94.9400024414062
5400 95.2699966430664
};
\addlegendentry{isotropic}
\addplot [very thick, TUMBlue]
table {%
1 10
2 9.82999992370605
5 9.97999954223633
10 11.8400001525879
20 16.3700008392334
50 51.7200012207031
100 78.2600021362305
200 87.8899993896484
500 91.3899993896484
1000 92.4199981689453
2000 93.1600036621094
5000 94.3300018310547
5400 94.6100006103516
};
\addlegendentry{learned}
\nextgroupplot[
title={$\tau = 10^{-5}$},
]
\addplot [very thick, TUMBlue]
table {%
1 12.1000003814697
2 36.2299995422363
5 35.7299995422363
10 53.1199989318848
20 61.9599990844727
50 83.6800003051758
100 87.1900024414062
200 89.2300033569336
500 91.5400009155273
1000 92.4499969482422
2000 93.25
5000 94.3499984741211
5400 94.5999984741211
};
\addplot [very thick, TUMAccentOrange, dotted]
table {%
1 10.0100002288818
2 10.0200004577637
5 10.0200004577637
20 10.4200000762939
50 21.9699993133545
100 58.1399993896484
200 70.8199996948242
500 90.5
1000 92.9300003051758
2000 94.0500030517578
5000 94.9700012207031
5400 95.2799987792969
};
\addplot [very thick, TUMAccentGreen, loosely dashed]
table {%
1 9.92000007629395
2 9.22000026702881
5 10
10 9.92000007629395
20 10.2200002670288
50 10.1499996185303
100 64.0500030517578
200 87.5699996948242
500 91.8399963378906
1000 93.2600021362305
2000 93.8499984741211
5000 95.3099975585938
5400 95.2600021362305
};
\nextgroupplot[
title={$\tau = 10^{-7}$},
]
\addplot [very thick, TUMBlue]
table {%
1 23.1499996185303
2 40.7400016784668
5 41.0900001525879
10 53.6199989318848
20 62.0099983215332
50 83.6500015258789
100 87.25
200 89.2099990844727
500 91.5299987792969
1000 92.4499969482422
2000 93.25
5000 94.3499984741211
5400 94.6100006103516
};
\addplot [very thick, TUMAccentOrange, dotted]
table {%
1 17.7900009155273
2 42.1599998474121
5 61.1599998474121
10 26.6100006103516
20 56.4799995422363
50 85.0800018310547
100 88.370002746582
200 89.6500015258789
500 92.0500030517578
1000 92.9899978637695
2000 94.0500030517578
5000 94.9700012207031
5400 95.2900009155273
};
\addplot [very thick, TUMAccentGreen, loosely dashed]
table {%
1 23.6200008392334
2 25.8199996948242
5 12.9899997711182
10 19.6399993896484
20 53.3400001525879
50 85.1500015258789
100 87.6800003051758
200 90.6600036621094
500 92.3300018310547
1000 93.3499984741211
2000 93.8600006103516
5000 95.3000030517578
5400 95.2600021362305
};
\nextgroupplot[
title={$\tau = 10^{-9}$},
]
\addplot [very thick, TUMBlue]
table {%
1 23.0699996948242
2 40.8300018310547
5 41.0800018310547
10 53.6199989318848
20 61.9900016784668
50 83.6500015258789
100 87.2600021362305
200 89.2099990844727
500 91.5299987792969
1000 92.4599990844727
2000 93.25
5000 94.3499984741211
5400 94.6100006103516
};
\addplot [very thick, TUMAccentOrange, dotted]
table {%
1 19.7299995422363
2 41.8699989318848
5 69.4899978637695
10 76.129997253418
20 78.870002746582
50 85.370002746582
100 88.370002746582
200 89.8300018310547
500 92.0299987792969
1000 92.9899978637695
2000 94.0500030517578
5000 94.9700012207031
5400 95.2900009155273
};
\addplot [very thick, TUMAccentGreen, loosely dashed]
table {%
1 34.9000015258789
2 51.4500007629395
5 64.8300018310547
10 76.6500015258789
20 81.1500015258789
50 85.5899963378906
100 87.75
200 90.6900024414062
500 92.3099975585938
1000 93.3600006103516
2000 93.8600006103516
5000 95.3000030517578
5400 95.2600021362305
};
\nextgroupplot[
title={$10^{-11}$},
ylabel style={rotate=-90},
ytick pos=right,%
ylabel={$\gamma = 0.01$},%
]
\addplot [very thick, TUMBlue]
table {%
1 23.0499992370605
2 40.8199996948242
5 41.0699996948242
10 53.6199989318848
20 62
50 83.6500015258789
100 87.2600021362305
200 89.2099990844727
500 91.5299987792969
1000 92.4599990844727
2000 93.25
5000 94.3499984741211
5400 94.6100006103516
};
\addplot [very thick, TUMAccentOrange, dotted]
table {%
1 19.7299995422363
2 41.8400001525879
5 69.4599990844727
10 76.1800003051758
20 78.8399963378906
50 85.379997253418
100 88.3600006103516
200 89.8300018310547
500 92.0299987792969
1000 92.9899978637695
2000 94.0500030517578
5000 94.9700012207031
5400 95.2900009155273
};
\addplot [very thick, TUMAccentGreen, loosely dashed]
table {%
1 34.810001373291
2 51.4099998474121
5 65.4899978637695
10 77.1699981689453
20 81.2300033569336
50 85.5999984741211
100 87.7399978637695
200 90.6900024414062
500 92.3099975585938
1000 93.3600006103516
2000 93.8600006103516
5000 95.3000030517578
5400 95.2600021362305
};
\nextgroupplot
\addplot [very thick, TUMBlue]
table {%
1 9.97000026702881
2 9.90999984741211
5 9.9399995803833
10 11.5500001907349
20 21.5499992370605
50 51.6300010681152
100 78.3499984741211
200 87.6800003051758
500 91.2200012207031
1000 92.2099990844727
2000 93.4199981689453
5000 94.5299987792969
5400 94.379997253418
};
\addplot [very thick, TUMAccentOrange, dotted]
table {%
1 10.1700000762939
2 10
5 10.0600004196167
10 10.0600004196167
20 9.98999977111816
50 10.0100002288818
100 9.64000034332275
200 11.7399997711182
500 15.6899995803833
1000 83.5999984741211
2000 90.9800033569336
5000 94.8499984741211
5400 95.1600036621094
};
\addplot [very thick, TUMAccentGreen, loosely dashed]
table {%
1 10.2700004577637
2 9.9399995803833
5 9.98999977111816
10 10.0100002288818
20 10.039999961853
50 10.0500001907349
100 9.9399995803833
200 10.2200002670288
500 15.2399997711182
1000 92.4499969482422
2000 94.0800018310547
5000 94.5199966430664
5400 94.9899978637695
};
\nextgroupplot
\addplot [very thick, TUMBlue]
table {%
1 10.3299999237061
2 34.0499992370605
5 24.5699996948242
10 61.0999984741211
20 61.4799995422363
50 82.4499969482422
100 87.2200012207031
200 89.5
500 91.5199966430664
1000 92.1100006103516
2000 93.4000015258789
5000 94.4800033569336
5400 94.4000015258789
};
\addplot [very thick, TUMAccentOrange, dotted]
table {%
1 10.1499996185303
2 9.96000003814697
5 9.89999961853027
10 9.65999984741211
20 10.289999961853
50 9.92000007629395
100 20.6800003051758
200 78.3099975585938
500 31.6700000762939
1000 92.8499984741211
2000 93.7799987792969
5000 94.8899993896484
5400 95.1399993896484
};
\addplot [very thick, TUMAccentGreen, loosely dashed]
table {%
1 9.55000019073486
2 10.4700002670288
5 10.0100002288818
10 10
20 10
50 10.460000038147
100 21.1900005340576
200 69.9300003051758
500 34.0699996948242
1000 92.6900024414062
2000 94.2099990844727
5000 94.9000015258789
5400 94.9400024414062
};
\nextgroupplot
\addplot [very thick, TUMBlue]
table {%
1 19.3600006103516
2 47.5999984741211
5 39.060001373291
10 61.7900009155273
20 61.7000007629395
50 82.4899978637695
100 87.3000030517578
200 89.5
500 91.5299987792969
1000 92.120002746582
2000 93.4199981689453
5000 94.4899978637695
5400 94.4000015258789
};
\addplot [very thick, TUMAccentOrange, dotted]
table {%
1 9.72999954223633
2 9.90999984741211
5 10.710000038147
10 10.0200004577637
20 10.1599998474121
50 65.8199996948242
100 56.7099990844727
200 89.3199996948242
500 91.0699996948242
1000 93.0100021362305
2000 93.9700012207031
5000 94.8899993896484
5400 95.2600021362305
};
\addplot [very thick, TUMAccentGreen, loosely dashed]
table {%
1 10.1199998855591
2 10.8800001144409
5 10
10 10.3000001907349
20 10.1199998855591
50 79.2699966430664
100 86.4599990844727
200 89.4800033569336
500 62.2799987792969
1000 93.0100021362305
2000 94.1800003051758
5000 95.0599975585938
5400 94.9599990844727
};
\nextgroupplot
\addplot [very thick, TUMBlue]
table {%
1 31.4899997711182
2 47.7000007629395
5 38.9900016784668
10 61.75
20 61.7000007629395
50 82.4800033569336
100 87.3099975585938
200 89.5
500 91.5199966430664
1000 92.120002746582
2000 93.4199981689453
5000 94.4899978637695
5400 94.4000015258789
};
\addplot [very thick, TUMAccentOrange, dotted]
table {%
1 16.6399993896484
2 38.1699981689453
5 17.7999992370605
10 38.689998626709
20 78.3499984741211
50 85.8300018310547
100 87.8000030517578
200 89.8199996948242
500 92.0400009155273
1000 93.0599975585938
2000 93.9599990844727
5000 94.879997253418
5400 95.2900009155273
};
\addplot [very thick, TUMAccentGreen, loosely dashed]
table {%
1 20.6599998474121
2 44.4799995422363
5 12.6800003051758
10 68.9499969482422
20 77.5400009155273
50 85.8899993896484
100 87.8899993896484
200 90.7600021362305
500 92.3399963378906
1000 93.0299987792969
2000 94.1999969482422
5000 95.0699996948242
5400 94.9499969482422
};
\nextgroupplot[
ylabel style={rotate=-90},
ytick pos=right,%
ylabel={$\gamma = 0.001$},%
]
\addplot [very thick, TUMBlue]
table {%
1 31.3199996948242
2 47.7000007629395
5 39.0099983215332
10 61.75
20 61.6800003051758
50 82.4800033569336
100 87.3099975585938
200 89.5
500 91.5199966430664
1000 92.120002746582
2000 93.4199981689453
5000 94.4899978637695
5400 94.4000015258789
};
\addplot [very thick, TUMAccentOrange, dotted]
table {%
1 19.1299991607666
2 42.060001373291
5 67.7099990844727
10 73.5400009155273
20 80.3199996948242
50 85.7399978637695
100 88.3399963378906
200 89.8300018310547
500 92.0199966430664
1000 93.0599975585938
2000 93.9499969482422
5000 94.879997253418
5400 95.2900009155273
};
\addplot [very thick, TUMAccentGreen, loosely dashed]
table {%
1 29.2600002288818
2 44.4799995422363
5 65.9100036621094
10 74.7099990844727
20 80.9899978637695
50 85.9400024414062
100 87.9100036621094
200 90.75
500 92.3499984741211
1000 93.0199966430664
2000 94.1999969482422
5000 95.0699996948242
5400 94.9499969482422
};
\nextgroupplot
\addplot [very thick, TUMBlue]
table {%
1 10.6000003814697
2 9.97999954223633
5 10.210000038147
10 10
20 10.039999961853
50 56.4700012207031
100 76.129997253418
200 86.879997253418
500 91.3300018310547
1000 92.0999984741211
2000 93.4899978637695
5000 94.4499969482422
5400 94.6100006103516
};
\addplot [very thick, TUMAccentOrange, dotted]
table {%
1 9.96000003814697
2 9.5600004196167
5 9.97000026702881
10 10.4499998092651
20 9.5
50 9.75
100 10.3699998855591
200 10.0900001525879
500 14.9300003051758
1000 34.9099998474121
2000 94.3300018310547
5000 95.129997253418
5400 95.0899963378906
};
\addplot [very thick, TUMAccentGreen, loosely dashed]
table {%
1 10.1899995803833
2 10.4399995803833
5 10
10 10.2399997711182
20 9.82999992370605
50 10.0600004196167
100 10.039999961853
200 10.5699996948242
500 82.0400009155273
1000 83.3199996948242
2000 94.2099990844727
5000 94.1699981689453
5400 95.129997253418
};
\nextgroupplot
\addplot [very thick, TUMBlue]
table {%
1 10.3599996566772
2 19.6700000762939
5 31.1499996185303
10 56.810001373291
20 53.1399993896484
50 83.0400009155273
100 87.2699966430664
200 89.879997253418
500 91.3600006103516
1000 92.2900009155273
2000 93.4300003051758
5000 94.4400024414062
5400 94.6100006103516
};
\addplot [very thick, TUMAccentOrange, dotted]
table {%
1 10.039999961853
2 10.3500003814697
5 9.98999977111816
10 9.97000026702881
20 10.3800001144409
50 10.3100004196167
100 10.039999961853
200 20
500 21.5200004577637
1000 44.8699989318848
2000 94.3199996948242
5000 95.1500015258789
5400 95.129997253418
};
\addplot [very thick, TUMAccentGreen, loosely dashed]
table {%
1 10.2200002670288
2 10.2200002670288
5 9.98999977111816
10 10.1000003814697
20 9.76000022888184
50 9.93000030517578
100 11.789999961853
200 51.439998626709
500 92.5
1000 86.0199966430664
2000 94.25
5000 94.5500030517578
5400 95.1399993896484
};
\nextgroupplot
\addplot [very thick, TUMBlue]
table {%
1 19.2000007629395
2 49.3600006103516
5 40.3800010681152
10 59.7700004577637
20 53.8899993896484
50 83.1900024414062
100 87.25
200 89.8499984741211
500 91.3499984741211
1000 92.3000030517578
2000 93.4300003051758
5000 94.4400024414062
5400 94.620002746582
};
\addplot [very thick, TUMAccentOrange, dotted]
table {%
1 9.89000034332275
2 10.4499998092651
5 9.77000045776367
10 10.1999998092651
20 9.82999992370605
50 13.710000038147
100 16.1000003814697
200 42.4099998474121
500 61.8899993896484
1000 90.8600006103516
2000 94.3300018310547
5000 95.1600036621094
5400 95.129997253418
};
\addplot [very thick, TUMAccentGreen, loosely dashed]
table {%
1 10.3000001907349
2 10.289999961853
5 10.0200004577637
10 10.3999996185303
20 10.0799999237061
50 10.4300003051758
100 71.0400009155273
200 71.75
500 92.5100021362305
1000 92.5899963378906
2000 94.25
5000 95.0999984741211
5400 95.1399993896484
};
\nextgroupplot
\addplot [very thick, TUMBlue]
table {%
1 28.0400009155273
2 49.1699981689453
5 40.4700012207031
10 59.7900009155273
20 53.8699989318848
50 83.1999969482422
100 87.2300033569336
200 89.8499984741211
500 91.3499984741211
1000 92.3000030517578
2000 93.4300003051758
5000 94.4400024414062
5400 94.620002746582
};
\addplot [very thick, TUMAccentOrange, dotted]
table {%
1 16.7800006866455
2 10.0200004577637
5 15.5600004196167
10 9.65999984741211
20 12.8599996566772
50 50.4700012207031
100 86.4300003051758
200 89.2300033569336
500 92.1999969482422
1000 93.1800003051758
2000 94.4000015258789
5000 95.1600036621094
5400 95.129997253418
};
\addplot [very thick, TUMAccentGreen, loosely dashed]
table {%
1 10.8299999237061
2 40.8800010681152
5 10
10 10
20 15.1999998092651
50 22.4400005340576
100 87.2300033569336
200 90.5899963378906
500 92.4700012207031
1000 93.0100021362305
2000 94.25
5000 95.1600036621094
5400 95.120002746582
};
\nextgroupplot[
ylabel style={rotate=-90},
ytick pos=right,%
ylabel={$\gamma = 0.0001$},%
]
\addplot [very thick, TUMBlue]
table {%
1 27.75
2 49.1800003051758
5 40.4799995422363
10 59.7799987792969
20 53.8699989318848
50 83.2099990844727
100 87.25
200 89.8499984741211
500 91.3499984741211
1000 92.3000030517578
2000 93.4300003051758
5000 94.4400024414062
5400 94.620002746582
};
\addplot [very thick, TUMAccentOrange, dotted]
table {%
1 31.1299991607666
2 40.810001373291
5 67.8099975585938
10 14.4700002670288
20 76.870002746582
50 85.5899963378906
100 88.4199981689453
200 89.8099975585938
500 92.1800003051758
1000 93.1699981689453
2000 94.370002746582
5000 95.1600036621094
5400 95.129997253418
};
\addplot [very thick, TUMAccentGreen, loosely dashed]
table {%
1 33.1100006103516
2 49.1300010681152
5 9.97999954223633
10 29.6200008392334
20 80.6500015258789
50 84.5100021362305
100 87.3600006103516
200 90.7699966430664
500 92.4700012207031
1000 92.9700012207031
2000 94.25
5000 95.1600036621094
5400 95.129997253418
};
\nextgroupplot
\addplot [very thick, TUMBlue]
table {%
1 9.8100004196167
2 10.1000003814697
5 10.0900001525879
10 11.6400003433228
20 18.1800003051758
50 47.2599983215332
100 82.620002746582
200 86.8600006103516
500 91.4199981689453
1000 92.2699966430664
2000 93.3000030517578
5000 94.4400024414062
5400 94.4100036621094
};
\addplot [very thick, TUMAccentOrange, dotted]
table {%
1 9.93000030517578
2 9.94999980926514
5 10.5699996948242
10 9.78999996185303
20 9.96000003814697
50 10.2799997329712
100 9.80000019073486
200 10.0100002288818
500 89.2099990844727
1000 92.6399993896484
2000 94.0599975585938
5000 95
5400 94.6999969482422
};
\addplot [very thick, TUMAccentGreen, loosely dashed]
table {%
1 9.88000011444092
2 9.81999969482422
5 10
10 10.2600002288818
20 9.97999954223633
50 10.1999998092651
100 9.57999992370605
200 10.4499998092651
500 42.810001373291
1000 92.5800018310547
2000 69.5500030517578
5000 94.0899963378906
5400 94.9400024414062
};
\nextgroupplot
\addplot [very thick, TUMBlue]
table {%
1 10.1899995803833
2 10.1899995803833
5 10.5200004577637
10 46.0999984741211
20 58.8300018310547
50 82.9300003051758
100 86.7799987792969
200 89.1699981689453
500 91.5100021362305
1000 92.2699966430664
2000 93.3000030517578
5000 94.4499969482422
5400 94.4300003051758
};
\addplot [very thick, TUMAccentOrange, dotted]
table {%
1 9.89999961853027
2 10.1099996566772
5 9.85000038146973
10 9.76000022888184
20 9.97999954223633
50 9.86999988555908
100 10.6499996185303
200 87.5699996948242
500 91.5999984741211
1000 92.6500015258789
2000 94.0199966430664
5000 94.9700012207031
5400 95.1900024414062
};
\addplot [very thick, TUMAccentGreen, loosely dashed]
table {%
1 10.0200004577637
2 10.5699996948242
5 10.0299997329712
10 9.71000003814697
20 10.0500001907349
50 9.97999954223633
100 11.9899997711182
200 24.7299995422363
500 54.9000015258789
1000 92.5299987792969
2000 70.4700012207031
5000 94.2099990844727
5400 95.0100021362305
};
\nextgroupplot
\addplot [very thick, TUMBlue]
table {%
1 9.64999961853027
2 17.9400005340576
5 48.0299987792969
10 47.3499984741211
20 58.8199996948242
50 82.8399963378906
100 86.75
200 89.1600036621094
500 91.5100021362305
1000 92.2799987792969
2000 93.3099975585938
5000 94.4499969482422
5400 94.4400024414062
};
\addplot [very thick, TUMAccentOrange, dotted]
table {%
1 9.97000026702881
2 9.47000026702881
5 9.90999984741211
10 9.90999984741211
20 9.69999980926514
50 12.210000038147
100 14.1499996185303
200 88.7799987792969
500 91.870002746582
1000 92.6399993896484
2000 93.9199981689453
5000 94.9700012207031
5400 95.1399993896484
};
\addplot [very thick, TUMAccentGreen, loosely dashed]
table {%
1 10.1300001144409
2 10.0799999237061
5 9.97000026702881
10 10.5
20 10.1400003433228
50 10.6800003051758
100 14.1800003051758
200 44.2000007629395
500 60.4900016784668
1000 92.9199981689453
2000 91.6900024414062
5000 94.6600036621094
5400 95.0999984741211
};
\nextgroupplot
\addplot [very thick, TUMBlue]
table {%
1 10.1000003814697
2 37.9300003051758
5 49.8199996948242
10 47.3499984741211
20 58.7700004577637
50 82.8600006103516
100 86.75
200 89.1699981689453
500 91.5100021362305
1000 92.2799987792969
2000 93.3099975585938
5000 94.4499969482422
5400 94.4400024414062
};
\addplot [very thick, TUMAccentOrange, dotted]
table {%
1 10.1199998855591
2 9.88000011444092
5 9.31999969482422
10 9.76000022888184
20 13.1099996566772
50 54.8300018310547
100 31.3400001525879
200 89.6800003051758
500 92.129997253418
1000 92.7099990844727
2000 93.9800033569336
5000 94.9700012207031
5400 95.2200012207031
};
\addplot [very thick, TUMAccentGreen, loosely dashed]
table {%
1 10.039999961853
2 10.0500001907349
5 10.0100002288818
10 9.89999961853027
20 10.3199996948242
50 14.6199998855591
100 45.7900009155273
200 79.3399963378906
500 91.25
1000 93.1900024414062
2000 94.0199966430664
5000 95
5400 95.129997253418
};
\nextgroupplot[
ylabel style={rotate=-90},
ytick pos=right,%
ylabel={$\gamma = 10^{-5}$},%
]
\addplot [very thick, TUMBlue]
table {%
1 12.5799999237061
2 39.2799987792969
5 49.9799995422363
10 47.3499984741211
20 58.7700004577637
50 82.8600006103516
100 86.75
200 89.1699981689453
500 91.5100021362305
1000 92.2799987792969
2000 93.3099975585938
5000 94.4499969482422
5400 94.4400024414062
};
\addplot [very thick, TUMAccentOrange, dotted]
table {%
1 10
2 10.0500001907349
5 10.6599998474121
10 10.3400001525879
20 76.5999984741211
50 83.6500015258789
100 86.9100036621094
200 89.7200012207031
500 92.1399993896484
1000 92.8199996948242
2000 93.9400024414062
5000 94.9700012207031
5400 95.25
};
\addplot [very thick, TUMAccentGreen, loosely dashed]
table {%
1 10.1599998474121
2 12.960000038147
5 10.1000003814697
10 9.72000026702881
20 13.0200004577637
50 78.6100006103516
100 87.379997253418
200 90.7200012207031
500 92.4000015258789
1000 93.2799987792969
2000 94.129997253418
5000 94.9899978637695
5400 95.129997253418
};
\nextgroupplot
\addplot [very thick, TUMBlue]
table {%
1 9.72000026702881
2 10.1000003814697
5 10.6499996185303
10 11.0699996948242
20 11.4799995422363
50 63.7299995422363
100 79.6399993896484
200 87.5
500 91.0500030517578
1000 92.3600006103516
2000 93.3499984741211
5000 94.5400009155273
5400 94.1800003051758
};
\addplot [very thick, TUMAccentOrange, dotted]
table {%
1 9.72999954223633
2 9.90999984741211
5 9.81999969482422
10 10.0200004577637
20 10
50 9.89000034332275
100 9.25
200 10.1899995803833
500 14.2399997711182
1000 12.2299995422363
2000 92.2600021362305
5000 94.9800033569336
5400 95.2900009155273
};
\addplot [very thick, TUMAccentGreen, loosely dashed]
table {%
1 10.0100002288818
2 10.4099998474121
5 9.71000003814697
10 9.64000034332275
20 10.1300001144409
50 9.65999984741211
100 9.9399995803833
200 10.789999961853
500 85.1900024414062
1000 66.2699966430664
2000 94.1100006103516
5000 88.4300003051758
5400 23.5400009155273
};
\nextgroupplot
\addplot [very thick, TUMBlue]
table {%
1 10.5799999237061
2 10.460000038147
5 11.0200004577637
10 56.9700012207031
20 49.2599983215332
50 81.370002746582
100 87.2200012207031
200 89.2900009155273
500 91.3499984741211
1000 92.3199996948242
2000 93.2799987792969
5000 94.5100021362305
5400 94.1699981689453
};
\addplot [very thick, TUMAccentOrange, dotted]
table {%
1 9.76000022888184
2 10.1199998855591
5 10.0600004196167
10 10.0100002288818
20 10
50 10.3500003814697
100 12.3599996566772
200 42.7900009155273
500 14.8999996185303
1000 11.960000038147
2000 92.2099990844727
5000 95
5400 95.2799987792969
};
\addplot [very thick, TUMAccentGreen, loosely dashed]
table {%
1 10.1899995803833
2 9.67000007629395
5 9.69999980926514
10 10.0699996948242
20 9.80000019073486
50 10.3400001525879
100 13.5500001907349
200 72.9800033569336
500 92.1999969482422
1000 67.0599975585938
2000 94.120002746582
5000 89.25
5400 23.7800006866455
};
\nextgroupplot
\addplot [very thick, TUMBlue]
table {%
1 10.1899995803833
2 10.7399997711182
5 51.0699996948242
10 61.7799987792969
20 51.5699996948242
50 81.4300003051758
100 87.1999969482422
200 89.2900009155273
500 91.3499984741211
1000 92.3099975585938
2000 93.2799987792969
5000 94.5199966430664
5400 94.1699981689453
};
\addplot [very thick, TUMAccentOrange, dotted]
table {%
1 9.81999969482422
2 9.53999996185303
5 9.78999996185303
10 9.97000026702881
20 9.75
50 11.3900003433228
100 21.4300003051758
200 52.6599998474121
500 15.0100002288818
1000 12.8699998855591
2000 93.2600021362305
5000 94.9599990844727
5400 95.2799987792969
};
\addplot [very thick, TUMAccentGreen, loosely dashed]
table {%
1 9.85000038146973
2 9.97000026702881
5 9.8100004196167
10 9.96000003814697
20 10.289999961853
50 13.3900003433228
100 33.6300010681152
200 73.2900009155273
500 92.3099975585938
1000 69.8499984741211
2000 94.129997253418
5000 94.1999969482422
5400 27.1499996185303
};
\nextgroupplot
\addplot [very thick, TUMBlue]
table {%
1 9.92000007629395
2 36.3300018310547
5 51.7200012207031
10 64.0699996948242
20 51.6500015258789
50 81.4199981689453
100 87.1900024414062
200 89.3000030517578
500 91.3499984741211
1000 92.3099975585938
2000 93.2799987792969
5000 94.5199966430664
5400 94.1699981689453
};
\addplot [very thick, TUMAccentOrange, dotted]
table {%
1 9.76000022888184
2 9.34000015258789
5 9.90999984741211
10 10.0299997329712
20 13.25
50 17.6200008392334
100 28.7800006866455
200 71.8300018310547
500 31.2999992370605
1000 22.3299999237061
2000 93.9199981689453
5000 94.9899978637695
5400 95.2799987792969
};
\addplot [very thick, TUMAccentGreen, loosely dashed]
table {%
1 9.82999992370605
2 10.0100002288818
5 9.78999996185303
10 9.90999984741211
20 10.6300001144409
50 27.5799999237061
100 55.6100006103516
200 84.0899963378906
500 92.1999969482422
1000 89.7399978637695
2000 94.1399993896484
5000 95.120002746582
5400 80.6900024414062
};
\nextgroupplot[
ylabel style={rotate=-90},
ytick pos=right,%
ylabel={$\gamma = 10^{-6}$},%
]
\addplot [very thick, TUMBlue]
table {%
1 10.5200004577637
2 46.5800018310547
5 51.6399993896484
10 64.1699981689453
20 51.6500015258789
50 81.4199981689453
100 87.1900024414062
200 89.3000030517578
500 91.3499984741211
1000 92.3099975585938
2000 93.2799987792969
5000 94.5199966430664
5400 94.1699981689453
};
\addplot [very thick, TUMAccentOrange, dotted]
table {%
1 9.85999965667725
2 9.88000011444092
5 9.97999954223633
10 9.80000019073486
20 76.629997253418
50 84.25
100 78.4100036621094
200 89.8600006103516
500 91.6500015258789
1000 92.9400024414062
2000 94.0199966430664
5000 94.9899978637695
5400 95.2799987792969
};
\addplot [very thick, TUMAccentGreen, loosely dashed]
table {%
1 9.97999954223633
2 10
5 10.0200004577637
10 11.460000038147
20 23.6800003051758
50 70.5999984741211
100 87.5100021362305
200 90.9100036621094
500 92.3199996948242
1000 93.0400009155273
2000 94.1399993896484
5000 95.129997253418
5400 95.370002746582
};
\nextgroupplot
\addplot [very thick, TUMBlue]
table {%
1 10.2299995422363
2 9
5 10.1499996185303
10 9.84000015258789
20 16.3199996948242
50 69.6699981689453
100 67.4400024414062
200 88.0100021362305
500 91.5400009155273
1000 92.2300033569336
2000 93.1600036621094
5000 94.2200012207031
5400 94.5
};
\addplot [very thick, TUMAccentOrange, dotted]
table {%
1 9.65999984741211
2 10
5 9.98999977111816
10 9.97999954223633
20 10
50 9.86999988555908
100 10.4399995803833
200 10.2299995422363
500 79.9599990844727
1000 48.0299987792969
2000 25.6000003814697
5000 94.8399963378906
5400 95.379997253418
};
\addplot [very thick, TUMAccentGreen, loosely dashed]
table {%
1 10
2 9.98999977111816
5 10.0100002288818
10 10.1400003433228
20 10.2299995422363
50 10
100 10.1300001144409
200 10.1300001144409
500 14.2700004577637
1000 93.1500015258789
2000 93.9599990844727
5000 94.5400009155273
5400 94.9800033569336
};
\nextgroupplot
\addplot [very thick, TUMBlue]
table {%
1 9.51000022888184
2 10.1999998092651
5 10.9499998092651
10 38.060001373291
20 60.4900016784668
50 81.25
100 87.0899963378906
200 89.6399993896484
500 91.7600021362305
1000 92.2099990844727
2000 93.1900024414062
5000 94.2099990844727
5400 94.5100021362305
};
\addplot [very thick, TUMAccentOrange, dotted]
table {%
1 9.93000030517578
2 9.98999977111816
5 10.1099996566772
10 9.81999969482422
20 10.0600004196167
50 10.1800003051758
100 24.5699996948242
200 28.5699996948242
500 91.7799987792969
1000 49.5400009155273
2000 24.5799999237061
5000 94.8399963378906
5400 95.379997253418
};
\addplot [very thick, TUMAccentGreen, loosely dashed]
table {%
1 9.61999988555908
2 9.97000026702881
5 10
10 10.1800003051758
20 10.3000001907349
50 9.96000003814697
100 10.6899995803833
200 35.2099990844727
500 13.4700002670288
1000 93.120002746582
2000 93.9499969482422
5000 94.5899963378906
5400 94.9599990844727
};
\nextgroupplot
\addplot [very thick, TUMBlue]
table {%
1 10.3999996185303
2 11.8400001525879
5 50.0499992370605
10 38.9700012207031
20 60.8800010681152
50 80.5899963378906
100 87.0699996948242
200 89.6699981689453
500 91.7600021362305
1000 92.2200012207031
2000 93.1800003051758
5000 94.2099990844727
5400 94.5100021362305
};
\addplot [very thick, TUMAccentOrange, dotted]
table {%
1 10.1599998474121
2 10
5 10.0600004196167
10 9.9399995803833
20 10
50 25.2999992370605
100 71.1900024414062
200 63.3199996948242
500 92.1399993896484
1000 50.310001373291
2000 26.6900005340576
5000 94.879997253418
5400 95.379997253418
};
\addplot [very thick, TUMAccentGreen, loosely dashed]
table {%
1 10
2 9.97000026702881
5 10.0100002288818
10 9.72000026702881
20 9.97999954223633
50 9.97000026702881
100 13.9799995422363
200 37.439998626709
500 14.3599996566772
1000 93.120002746582
2000 93.9400024414062
5000 94.6699981689453
5400 95.0699996948242
};
\nextgroupplot
\addplot [very thick, TUMBlue]
table {%
1 11.1599998474121
2 23.9300003051758
5 54.7000007629395
10 39.0499992370605
20 60.8800010681152
50 80.5899963378906
100 87.0800018310547
200 89.6600036621094
500 91.7600021362305
1000 92.2200012207031
2000 93.1800003051758
5000 94.2099990844727
5400 94.5100021362305
};
\addplot [very thick, TUMAccentOrange, dotted]
table {%
1 9.77000045776367
2 10
5 9.89000034332275
10 10.3699998855591
20 10.8299999237061
50 84.5400009155273
100 75.5500030517578
200 78.4100036621094
500 92.1699981689453
1000 67.2699966430664
2000 36.9700012207031
5000 94.9400024414062
5400 95.379997253418
};
\addplot [very thick, TUMAccentGreen, loosely dashed]
table {%
1 9.92000007629395
2 10.0500001907349
5 9.97999954223633
10 9.73999977111816
20 10
50 10.3500003814697
100 15.0799999237061
200 43.9199981689453
500 15.3699998855591
1000 93.120002746582
2000 93.9400024414062
5000 94.8499984741211
5400 95.120002746582
};
\nextgroupplot[
ylabel style={rotate=-90},
ytick pos=right,%
ylabel={$\gamma = 10^{-7}$},%
]
\addplot [very thick, TUMBlue]
table {%
1 10.7600002288818
2 24.1200008392334
5 54.4500007629395
10 39.0400009155273
20 60.8600006103516
50 80.5899963378906
100 87.0800018310547
200 89.6600036621094
500 91.7600021362305
1000 92.2200012207031
2000 93.1800003051758
5000 94.2099990844727
5400 94.5100021362305
};
\addplot [very thick, TUMAccentOrange, dotted]
table {%
1 10.0200004577637
2 10
5 9.94999980926514
10 9.92000007629395
20 48.0200004577637
50 85.870002746582
100 87.8300018310547
200 89.4499969482422
500 92.25
1000 93.2099990844727
2000 93.870002746582
5000 95.129997253418
5400 95.379997253418
};
\addplot [very thick, TUMAccentGreen, loosely dashed]
table {%
1 9.64000034332275
2 10.0299997329712
5 10.0100002288818
10 9.65999984741211
20 10.0200004577637
50 13.0600004196167
100 39.4799995422363
200 88.8899993896484
500 86.6500015258789
1000 93.120002746582
2000 93.9400024414062
5000 94.9700012207031
5400 95.2099990844727
};
\nextgroupplot
\addplot [very thick, TUMBlue]
table {%
1 10.1700000762939
2 10.0600004196167
5 10.1000003814697
10 11.5200004577637
20 13.3900003433228
50 60.0699996948242
100 75.9599990844727
200 87.6100006103516
500 91.0400009155273
1000 92.1699981689453
2000 93.25
5000 94.4800033569336
5400 94.6100006103516
};
\addplot [very thick, TUMAccentOrange, dotted]
table {%
1 10.5699996948242
2 10.2700004577637
5 9.93000030517578
10 10.1099996566772
20 9.98999977111816
50 10.8699998855591
100 10.3000001907349
200 11.1000003814697
500 33.560001373291
1000 45.3199996948242
2000 94.0199966430664
5000 91.4700012207031
5400 95.0699996948242
};
\addplot [very thick, TUMAccentGreen, loosely dashed]
table {%
1 10.3100004196167
2 10.0600004196167
5 10
10 10.039999961853
20 10.1499996185303
50 10.0500001907349
100 9.93000030517578
200 10.7299995422363
500 78.1699981689453
1000 92.9700012207031
2000 14.5799999237061
5000 95.25
5400 26.25
};
\nextgroupplot
\addplot [very thick, TUMBlue]
table {%
1 10.1999998092651
2 9.73999977111816
5 10.3800001144409
10 60.1599998474121
20 60.3199996948242
50 82.5100021362305
100 86.9300003051758
200 89.5899963378906
500 91.3499984741211
1000 92.2699966430664
2000 93.2699966430664
5000 94.4499969482422
5400 94.5800018310547
};
\addplot [very thick, TUMAccentOrange, dotted]
table {%
1 9.76000022888184
2 9.71000003814697
5 9.97000026702881
20 10.0100002288818
50 10.3900003433228
100 11.3800001144409
200 56.9099998474121
500 34.4700012207031
1000 46.4199981689453
2000 94.0500030517578
5000 85.6500015258789
5400 95.0800018310547
};
\addplot [very thick, TUMAccentGreen, loosely dashed]
table {%
1 10.1099996566772
2 10.0699996948242
5 10.0100002288818
10 10.1800003051758
20 10.3900003433228
50 9.92000007629395
100 10.039999961853
200 81.120002746582
500 92.0999984741211
1000 93
2000 14.8299999237061
5000 95.25
5400 27.7999992370605
};
\nextgroupplot
\addplot [very thick, TUMBlue]
table {%
1 10.2799997329712
2 13.0900001525879
5 53.9300003051758
10 60.3199996948242
20 60.4300003051758
50 82.5999984741211
100 86.8899993896484
200 89.5699996948242
500 91.3300018310547
1000 92.2699966430664
2000 93.2699966430664
5000 94.4499969482422
5400 94.5800018310547
};
\addplot [very thick, TUMAccentOrange, dotted]
table {%
1 9.71000003814697
2 9.5600004196167
5 9.96000003814697
10 9.98999977111816
20 10.0799999237061
50 13.0900001525879
100 15.4099998474121
200 55.5
500 34.9000015258789
1000 49.25
2000 94.0299987792969
5000 86.1600036621094
5400 95.0199966430664
};
\addplot [very thick, TUMAccentGreen, loosely dashed]
table {%
1 9.75
2 9.84000015258789
5 10
10 10.1099996566772
20 10
50 10
100 69.2799987792969
200 89.9199981689453
500 92.0800018310547
1000 92.9499969482422
2000 15.1000003814697
5000 95.2600021362305
5400 27.7600002288818
};
\nextgroupplot
\addplot [very thick, TUMBlue]
table {%
1 10.4799995422363
2 37.6100006103516
5 53.7799987792969
10 60.310001373291
20 60.4700012207031
50 82.5899963378906
100 86.9100036621094
200 89.5699996948242
500 91.3300018310547
1000 92.2699966430664
2000 93.2699966430664
5000 94.4499969482422
5400 94.5800018310547
};
\addplot [very thick, TUMAccentOrange, dotted]
table {%
1 9.94999980926514
2 9.85999965667725
5 9.96000003814697
10 9.97000026702881
20 10.0200004577637
50 57.4000015258789
100 15.6400003433228
200 61.7799987792969
500 36.6399993896484
1000 52.75
2000 94.0599975585938
5000 88.4599990844727
5400 95.0599975585938
};
\addplot [very thick, TUMAccentGreen, loosely dashed]
table {%
1 9.85999965667725
2 10.0600004196167
5 10
10 10.1099996566772
20 10.6400003433228
50 12.3900003433228
100 67.7300033569336
200 88.4000015258789
500 92.1500015258789
1000 93.0500030517578
2000 24.2700004577637
5000 95.2600021362305
5400 33.4700012207031
};
\nextgroupplot[
ylabel style={rotate=-90},
ytick pos=right,%
ylabel={$\gamma = 10^{-8}$},%
]
\addplot [very thick, TUMBlue]
table {%
1 13.6999998092651
2 43.4000015258789
5 53.8899993896484
10 60.3400001525879
20 60.4700012207031
50 82.5899963378906
100 86.9100036621094
200 89.5699996948242
500 91.3300018310547
1000 92.2699966430664
2000 93.2699966430664
5000 94.4499969482422
5400 94.5800018310547
};
\addplot [very thick, TUMAccentOrange, dotted]
table {%
1 9.43000030517578
2 10.289999961853
5 9.93000030517578
10 9.96000003814697
20 10.3100004196167
50 69.1900024414062
100 22.9500007629395
200 86.9700012207031
500 63.4900016784668
1000 91.1699981689453
2000 94.129997253418
5000 94.5999984741211
5400 95.1399993896484
};
\addplot [very thick, TUMAccentGreen, loosely dashed]
table {%
1 10.1099996566772
2 9.96000003814697
5 10
10 10.0500001907349
20 28.1200008392334
50 15.0500001907349
100 77.5899963378906
200 89.6800003051758
500 92.629997253418
1000 93.0899963378906
2000 79.9199981689453
5000 95.2600021362305
5400 80.7900009155273
};
\nextgroupplot
\addplot [very thick, TUMBlue]
table {%
1 10.2399997711182
2 9.81999969482422
5 10.2700004577637
10 10.4200000762939
20 11.5500001907349
50 28.1599998474121
100 70.6100006103516
200 85.879997253418
500 91.5
1000 92.1900024414062
2000 93.4499969482422
5000 94.5999984741211
5400 94.620002746582
};
\addplot [very thick, TUMAccentOrange, dotted]
table {%
1 10.0699996948242
2 10.9300003051758
5 10.2700004577637
10 9.77999973297119
20 10
50 9.63000011444092
100 9.90999984741211
200 10.3000001907349
500 39.4300003051758
1000 92.2600021362305
2000 71.9100036621094
5000 94.75
5400 95.2399978637695
};
\addplot [very thick, TUMAccentGreen, loosely dashed]
table {%
1 9.9399995803833
2 10.0299997329712
5 10
10 10
20 10.0200004577637
50 9.61999988555908
100 10.1899995803833
200 10.1899995803833
500 20.8899993896484
1000 92.1600036621094
2000 94.1699981689453
5000 95.0199966430664
5400 95.0100021362305
};
\nextgroupplot
\addplot [very thick, TUMBlue]
table {%
1 10.4700002670288
2 10.0200004577637
5 10.460000038147
10 56.5299987792969
20 63.4099998474121
50 83.0699996948242
100 86.8399963378906
200 89.4100036621094
500 91.4800033569336
1000 92.1999969482422
2000 93.4400024414062
5000 94.5899963378906
5400 94.629997253418
};
\addplot [very thick, TUMAccentOrange, dotted]
table {%
1 10.1400003433228
2 10
5 10.1499996185303
10 9.67000007629395
20 10
50 10.5699996948242
100 31.4500007629395
200 10.7200002670288
500 11.2399997711182
1000 92.5100021362305
2000 52.1100006103516
5000 94.8399963378906
5400 95.2099990844727
};
\addplot [very thick, TUMAccentGreen, loosely dashed]
table {%
1 10.460000038147
2 10.1999998092651
5 9.98999977111816
20 10.0100002288818
50 10.3000001907349
100 26.3799991607666
200 13.6800003051758
500 23.8099994659424
1000 92.379997253418
2000 94.2099990844727
5000 94.9499969482422
5400 94.9300003051758
};
\nextgroupplot
\addplot [very thick, TUMBlue]
table {%
1 10.3500003814697
2 19.9400005340576
5 34.1800003051758
10 56.9700012207031
20 62.4199981689453
50 83.120002746582
100 86.8600006103516
200 89.4400024414062
500 91.4800033569336
1000 92.1999969482422
2000 93.4400024414062
5000 94.5899963378906
5400 94.629997253418
};
\addplot [very thick, TUMAccentOrange, dotted]
table {%
1 9.89999961853027
2 9.98999977111816
5 9.84000015258789
10 10.1700000762939
20 10.0600004196167
50 10.4499998092651
100 87.4100036621094
200 13.960000038147
500 10.0100002288818
1000 92.5999984741211
2000 21.4200000762939
5000 94.7399978637695
5400 95.2099990844727
};
\addplot [very thick, TUMAccentGreen, loosely dashed]
table {%
1 10.1099996566772
2 9.88000011444092
5 10
10 9.98999977111816
20 10.1099996566772
50 12.75
100 57.3199996948242
200 24.9699993133545
500 25.1200008392334
1000 92.2099990844727
2000 94.2200012207031
5000 94.9300003051758
5400 94.8899993896484
};
\nextgroupplot
\addplot [very thick, TUMBlue]
table {%
1 10.4499998092651
2 37.3600006103516
5 47.9500007629395
10 56.9700012207031
20 62.3199996948242
50 83.120002746582
100 86.8600006103516
200 89.4400024414062
500 91.4899978637695
1000 92.1999969482422
2000 93.4400024414062
5000 94.5899963378906
5400 94.629997253418
};
\addplot [very thick, TUMAccentOrange, dotted]
table {%
1 9.60999965667725
2 10
5 9.56999969482422
10 9.97999954223633
20 10.0100002288818
50 17.25
100 87.4499969482422
200 11.289999961853
500 10.0299997329712
1000 92.4499969482422
2000 21.4699993133545
5000 94.7799987792969
5400 95.2099990844727
};
\addplot [very thick, TUMAccentGreen, loosely dashed]
table {%
1 10.1199998855591
2 10.3400001525879
5 10
10 9.97999954223633
20 9.88000011444092
50 79.8199996948242
100 59.189998626709
200 23.2900009155273
500 23.8199996948242
1000 92.6600036621094
2000 94.2200012207031
5000 95
5400 94.9599990844727
};
\nextgroupplot[
ylabel style={rotate=-90},
ytick pos=right,%
ylabel={$\gamma = 10^{-9}$},%
]
\addplot [very thick, TUMBlue]
table {%
1 9.02999973297119
2 38.1199989318848
5 48.0699996948242
10 56.9700012207031
20 62.4199981689453
50 83.120002746582
100 86.8600006103516
200 89.4400024414062
500 91.4899978637695
1000 92.1999969482422
2000 93.4400024414062
5000 94.5899963378906
5400 94.629997253418
};
\addplot [very thick, TUMAccentOrange, dotted]
table {%
1 9.84000015258789
2 10
5 10
10 10.1000003814697
20 14.6800003051758
50 23.4699993133545
100 88.3899993896484
200 11.1099996566772
500 10.0500001907349
1000 92.6500015258789
2000 39.7299995422363
5000 94.7799987792969
5400 95.2099990844727
};
\addplot [very thick, TUMAccentGreen, loosely dashed]
table {%
1 10.0699996948242
2 10.1899995803833
5 10
10 9.80000019073486
20 16.5699996948242
50 81.9499969482422
100 69.6500015258789
200 23.5699996948242
500 35.1699981689453
1000 92.9199981689453
2000 94.2200012207031
5000 95.0500030517578
5400 95
};
\end{groupplot}
\path (group c1r8.south east) -- node[below=1.2cm]{\pgfplotslegendfromname{LegendSmallData}} (group c5r8.south west);
\end{tikzpicture}

%% file: tikz/appendix/results/cold_posterior_small.tex
\begin{tikzpicture}
\begin{axis}[
name=boundary,
footnotesize,
width=\textwidth,
height=0.6\textwidth,
log basis x={10},
xlabel={Temperature scaling $\tau$},
xmode=log,
max space between ticks=40pt,
ylabel={Accuracy},
ymin=0, ymax=100,
]
\path [draw=TUMBlue, fill=TUMBlue, opacity=0.2]
(axis cs:1e-09,94.4730461956522)
--(axis cs:1e-09,94.4057554347826)
--(axis cs:1e-08,94.4049945652174)
--(axis cs:1e-07,94.4026059782609)
--(axis cs:1e-06,94.4033695652174)
--(axis cs:1e-05,94.401847826087)
--(axis cs:0.0001,94.3995434782609)
--(axis cs:0.001,94.4040217391305)
--(axis cs:0.01,94.4179347826087)
--(axis cs:0.1,94.3010679347826)
--(axis cs:0.3,92.5034782608696)
--(axis cs:1,63.6435923913044)
--(axis cs:1,73.5165597826087)
--(axis cs:1,73.5165597826087)
--(axis cs:0.3,93.1888179347826)
--(axis cs:0.1,94.3600271739131)
--(axis cs:0.01,94.4880652173913)
--(axis cs:0.001,94.4719565217391)
--(axis cs:0.0001,94.4716331521739)
--(axis cs:1e-05,94.4690244565217)
--(axis cs:1e-06,94.4699130434783)
--(axis cs:1e-07,94.4684809782609)
--(axis cs:1e-08,94.472285326087)
--(axis cs:1e-09,94.4730461956522)
--cycle;
\path [draw=TUMBlue, fill=TUMBlue, opacity=0.2]
(axis cs:1e-09,94.4257690217391)
--(axis cs:1e-09,94.3645516304347)
--(axis cs:1e-08,94.3651086956522)
--(axis cs:1e-07,94.3640163043478)
--(axis cs:1e-06,94.36325)
--(axis cs:1e-05,94.3651059782609)
--(axis cs:0.0001,94.3604320652174)
--(axis cs:0.001,94.3659782608696)
--(axis cs:0.01,94.3749782608696)
--(axis cs:0.1,94.2519565217391)
--(axis cs:0.3,92.3869510869565)
--(axis cs:1,64.0437798913044)
--(axis cs:1,73.6619701086957)
--(axis cs:1,73.6619701086957)
--(axis cs:0.3,93.1041114130435)
--(axis cs:0.1,94.3155625)
--(axis cs:0.01,94.4333858695652)
--(axis cs:0.001,94.4262065217391)
--(axis cs:0.0001,94.4204402173913)
--(axis cs:1e-05,94.4253288043478)
--(axis cs:1e-06,94.4229347826087)
--(axis cs:1e-07,94.4220733695652)
--(axis cs:1e-08,94.4232663043478)
--(axis cs:1e-09,94.4257690217391)
--cycle;
\path [draw=TUMBlue, fill=TUMBlue, opacity=0.2]
(axis cs:1e-09,94.4223131868132)
--(axis cs:1e-09,94.3625274725275)
--(axis cs:1e-08,94.3657087912088)
--(axis cs:1e-07,94.3627119565217)
--(axis cs:1e-06,94.3629347826087)
--(axis cs:1e-05,94.3597798913043)
--(axis cs:0.0001,94.3629266304348)
--(axis cs:0.001,94.3599918478261)
--(axis cs:0.01,94.3733668478261)
--(axis cs:0.1,94.2356521739131)
--(axis cs:0.3,92.4872994505494)
--(axis cs:1,64.073402173913)
--(axis cs:1,73.7797173913043)
--(axis cs:1,73.7797173913043)
--(axis cs:0.3,93.0761401098901)
--(axis cs:0.1,94.3001141304348)
--(axis cs:0.01,94.4314157608696)
--(axis cs:0.001,94.417722826087)
--(axis cs:0.0001,94.4212010869565)
--(axis cs:1e-05,94.4182663043478)
--(axis cs:1e-06,94.4196875)
--(axis cs:1e-07,94.4213233695652)
--(axis cs:1e-08,94.4220934065934)
--(axis cs:1e-09,94.4223131868132)
--cycle;
\path [draw=TUMAccentOrange, dotted, fill=TUMAccentOrange, dotted, opacity=0.2]
(axis cs:1e-09,91.4422079207921)
--(axis cs:1e-09,84.3317722772277)
--(axis cs:1e-08,82.2899034653465)
--(axis cs:1e-07,78.8668737623762)
--(axis cs:1e-06,78.5922797029703)
--(axis cs:1e-05,77.2138094059406)
--(axis cs:0.0001,77.3369925742574)
--(axis cs:0.001,76.8623762376237)
--(axis cs:0.01,76.2221930693069)
--(axis cs:0.1,75.2677178217822)
--(axis cs:0.3,68.2386633663366)
--(axis cs:1,43.2894975247525)
--(axis cs:1,55.6192004950495)
--(axis cs:1,55.6192004950495)
--(axis cs:0.3,78.5704752475248)
--(axis cs:0.1,84.7154084158416)
--(axis cs:0.01,86.105004950495)
--(axis cs:0.001,86.3148366336633)
--(axis cs:0.0001,86.4148292079208)
--(axis cs:1e-05,86.7231188118812)
--(axis cs:1e-06,87.3487351485148)
--(axis cs:1e-07,88.1179603960396)
--(axis cs:1e-08,89.9049925742574)
--(axis cs:1e-09,91.4422079207921)
--cycle;
\path [draw=TUMAccentOrange, dotted, fill=TUMAccentOrange, dotted, opacity=0.2]
(axis cs:1e-09,95.1996994949495)
--(axis cs:1e-09,94.8928661616162)
--(axis cs:1e-08,90.7155732323232)
--(axis cs:1e-07,85.5477247474748)
--(axis cs:1e-06,82.5818484848485)
--(axis cs:1e-05,78.9540675)
--(axis cs:0.0001,72.813625)
--(axis cs:0.001,71.450975)
--(axis cs:0.01,70.4267054455446)
--(axis cs:0.1,68.5627400990099)
--(axis cs:0.3,62.6555681818182)
--(axis cs:1,39.1733910891089)
--(axis cs:1,50.7825198019802)
--(axis cs:1,50.7825198019802)
--(axis cs:0.3,74.3323939393939)
--(axis cs:0.1,79.2065272277228)
--(axis cs:0.01,80.9370198019802)
--(axis cs:0.001,82.379475)
--(axis cs:0.0001,82.94416)
--(axis cs:1e-05,87.7429225)
--(axis cs:1e-06,90.4418383838384)
--(axis cs:1e-07,91.917702020202)
--(axis cs:1e-08,94.6179595959596)
--(axis cs:1e-09,95.1996994949495)
--cycle;
\path [draw=TUMAccentOrange, dotted, fill=TUMAccentOrange, dotted, opacity=0.2]
(axis cs:1e-09,95.2257601010101)
--(axis cs:1e-09,95.1757575757576)
--(axis cs:1e-08,95.1785858585859)
--(axis cs:1e-07,95.150404040404)
--(axis cs:1e-06,93.6342373737374)
--(axis cs:1e-05,90.4725631313131)
--(axis cs:0.0001,88.4959040404041)
--(axis cs:0.001,86.6577676767677)
--(axis cs:0.01,83.3282095959596)
--(axis cs:0.1,79.5617676767677)
--(axis cs:0.3,72.0609166666667)
--(axis cs:1,43.8232954545455)
--(axis cs:1,56.6672777777778)
--(axis cs:1,56.6672777777778)
--(axis cs:0.3,81.1660202020202)
--(axis cs:0.1,87.6034419191919)
--(axis cs:0.01,89.9588333333333)
--(axis cs:0.001,92.2966641414141)
--(axis cs:0.0001,93.2747727272727)
--(axis cs:1e-05,94.1031944444444)
--(axis cs:1e-06,95.0157777777778)
--(axis cs:1e-07,95.2083863636364)
--(axis cs:1e-08,95.2258737373737)
--(axis cs:1e-09,95.2257601010101)
--cycle;
\path [draw=TUMAccentGreen, loosely dashed, fill=TUMAccentGreen, loosely dashed, opacity=0.2]
(axis cs:1e-09,91.8694615384615)
--(axis cs:1e-09,75.9917884615385)
--(axis cs:1e-08,75.3292307692308)
--(axis cs:1e-07,74.337625)
--(axis cs:1e-06,73.6776442307692)
--(axis cs:1e-05,73.505)
--(axis cs:0.0001,73.7437884615385)
--(axis cs:0.001,73.2862403846154)
--(axis cs:0.01,71.6508942307692)
--(axis cs:0.1,70.3256153846154)
--(axis cs:0.3,55.1269519230769)
--(axis cs:1,15.7692115384615)
--(axis cs:1,20.8703269230769)
--(axis cs:1,20.8703269230769)
--(axis cs:0.3,74.2320961538462)
--(axis cs:0.1,88.2959326923077)
--(axis cs:0.01,89.1400192307692)
--(axis cs:0.001,89.5157788461538)
--(axis cs:0.0001,89.6608653846154)
--(axis cs:1e-05,89.8945384615385)
--(axis cs:1e-06,90.1220384615385)
--(axis cs:1e-07,90.7751538461538)
--(axis cs:1e-08,90.2712980769231)
--(axis cs:1e-09,91.8694615384615)
--cycle;
\path [draw=TUMAccentGreen, loosely dashed, fill=TUMAccentGreen, loosely dashed, opacity=0.2]
(axis cs:1e-09,95.2050096153846)
--(axis cs:1e-09,94.99575)
--(axis cs:1e-08,85.4929038461539)
--(axis cs:1e-07,72.3305)
--(axis cs:1e-06,66.7356730769231)
--(axis cs:1e-05,63.43175)
--(axis cs:0.0001,63.2595576923077)
--(axis cs:0.001,63.1035096153846)
--(axis cs:0.01,63.6133461538461)
--(axis cs:0.1,63.0602019230769)
--(axis cs:0.3,45.1547596153846)
--(axis cs:1,12.637875)
--(axis cs:1,15.7209134615385)
--(axis cs:1,15.7209134615385)
--(axis cs:0.3,65.0158076923077)
--(axis cs:0.1,83.4797019230769)
--(axis cs:0.01,84.7387115384615)
--(axis cs:0.001,84.5898365384615)
--(axis cs:0.0001,84.2302980769231)
--(axis cs:1e-05,84.2529230769231)
--(axis cs:1e-06,85.8731826923077)
--(axis cs:1e-07,89.929375)
--(axis cs:1e-08,94.5071538461538)
--(axis cs:1e-09,95.2050096153846)
--cycle;
\path [draw=TUMAccentGreen, loosely dashed, fill=TUMAccentGreen, loosely dashed, opacity=0.2]
(axis cs:1e-09,95.2338557692308)
--(axis cs:1e-09,95.1469134615385)
--(axis cs:1e-08,95.1457692307692)
--(axis cs:1e-07,94.8326730769231)
--(axis cs:1e-06,83.5637692307692)
--(axis cs:1e-05,75.9562692307692)
--(axis cs:0.0001,75.2325865384615)
--(axis cs:0.001,71.6749423076923)
--(axis cs:0.01,70.8860192307693)
--(axis cs:0.1,69.6582788461538)
--(axis cs:0.3,57.4634807692308)
--(axis cs:1,14.7155961538462)
--(axis cs:1,19.6518653846154)
--(axis cs:1,19.6518653846154)
--(axis cs:0.3,76.6077307692308)
--(axis cs:0.1,90.5640480769231)
--(axis cs:0.01,91.5362115384615)
--(axis cs:0.001,91.9795961538461)
--(axis cs:0.0001,92.6337115384615)
--(axis cs:1e-05,94.0239903846154)
--(axis cs:1e-06,95.030125)
--(axis cs:1e-07,95.2157788461539)
--(axis cs:1e-08,95.2373076923077)
--(axis cs:1e-09,95.2338557692308)
--cycle;
\addplot [very thick, TUMBlue, mark=|]
table {%
9.99999971718069e-10 94.4372863769531
0.000999999465420842 94.4379348754883
0.00999999418854713 94.4517364501953
0.100000001490116 94.3306503295898
0.299999982118607 92.8483657836914
1 68.5803298950195
};
\label{pgfplots:cold_posterior_small_learned_8}
\addplot [very thick, TUMBlue]
table {%
9.99999971718069e-10 94.3940200805664
9.99998883344233e-06 94.3946762084961
9.99999974737875e-05 94.3903274536133
0.00999999418854713 94.4031524658203
0.100000001490116 94.2835845947266
0.299999982118607 92.7732620239258
1 69.0018463134766
};
\label{pgfplots:cold_posterior_small_learned_5}
\addplot [very thick, TUMBlue, mark=*]
table {%
9.99999971718069e-10 94.3925247192383
0.000999999465420842 94.3893508911133
0.00999999418854713 94.4010848999023
0.100000001490116 94.2701110839844
0.299999982118607 92.7926406860352
1 69.0880432128906
};
\label{pgfplots:cold_posterior_small_learned_3}
\addplot [very thick, TUMAccentOrange, dotted, mark=|]
table {%
9.99999971718069e-10 88.0781173706055
9.9999777347648e-09 86.2591094970703
9.9999887481772e-08 83.4696044921875
9.99998860606865e-07 83.029899597168
9.99998883344233e-06 82.2221755981445
9.99999974737875e-05 82.0034637451172
0.000999999465420842 81.7080230712891
0.00999999418854713 81.3465347290039
0.100000001490116 80.1801986694336
0.299999982118607 73.6076202392578
1 49.6616821289062
};
\label{pgfplots:cold_posterior_small_iso_8}
\addplot [very thick, TUMAccentOrange, dotted]
table {%
9.99999971718069e-10 95.0874710083008
9.9999777347648e-09 93.0270690917969
9.9999887481772e-08 89.0841445922852
9.99998860606865e-07 86.8150482177734
9.99998883344233e-06 83.7592010498047
9.99999974737875e-05 78.3619003295898
0.000999999465420842 77.3093032836914
0.00999999418854713 76.0009918212891
0.100000001490116 73.8769302368164
0.299999982118607 68.5429306030273
1 44.5184173583984
};
\label{pgfplots:cold_posterior_small_iso_5}
\addplot [very thick, TUMAccentOrange, dotted, mark=*]
table {%
9.99999971718069e-10 95.2022247314453
9.9999777347648e-09 95.2021179199219
9.9999887481772e-08 95.1786880493164
9.99998860606865e-07 94.4511108398438
9.99998883344233e-06 92.5080795288086
9.99999974737875e-05 91.1372756958008
0.000999999465420842 89.7161636352539
0.00999999418854713 86.9089889526367
0.100000001490116 83.8316192626953
0.299999982118607 76.5757598876953
1 50.3021202087402
};
\label{pgfplots:cold_posterior_small_iso_3}
\addplot [very thick, TUMAccentGreen, loosely dashed, mark=|]
table {%
9.99999971718069e-10 84.7576904296875
9.9999777347648e-09 83.5638427734375
9.9999887481772e-08 83.1165390014648
9.99998860606865e-07 82.6365356445312
9.99998883344233e-06 82.2303848266602
9.99999974737875e-05 81.9961547851562
0.000999999465420842 81.982307434082
0.00999999418854713 80.903076171875
0.100000001490116 79.5192337036133
0.299999982118607 65.2061538696289
1 18.3103847503662
};
\label{pgfplots:cold_posterior_small_0_8}
\addplot [very thick, TUMAccentGreen, loosely dashed]
table {%
9.99999971718069e-10 95.114616394043
9.9999777347648e-09 90.9373092651367
9.9999887481772e-08 82.0988464355469
9.99998860606865e-07 76.6053848266602
9.99998883344233e-06 74.7415390014648
9.99999974737875e-05 74.2584609985352
0.000999999465420842 73.8057708740234
0.00999999418854713 74.7180786132812
0.100000001490116 73.7850036621094
0.299999982118607 54.7165374755859
1 14.0569229125977
};
\label{pgfplots:cold_posterior_small_0_5}
\addplot [very thick, TUMAccentGreen, loosely dashed, mark=*]
table {%
9.99999971718069e-10 95.1903839111328
9.9999777347648e-09 95.1930770874023
9.9999887481772e-08 95.0707702636719
9.99998860606865e-07 90.4026947021484
9.99998883344233e-06 86.4796142578125
9.99999974737875e-05 84.8807678222656
0.000999999465420842 83.0853881835938
0.00999999418854713 82.5784606933594
0.100000001490116 81.4203872680664
0.299999982118607 67.1846160888672
1 17.0653839111328
};
\label{pgfplots:cold_posterior_small_0_3}

\end{axis}
\node[draw,fill=white,inner sep=2pt, below=1cm] at (boundary.south) {\small
    \begin{tabular}{cccl}
    $\gamma = 10^{-3}$ & $\gamma = 10^{-5}$ & $\gamma = 10^{-8}$ \\
    \ref{pgfplots:cold_posterior_small_0_3} & \ref{pgfplots:cold_posterior_small_0_5} & \ref{pgfplots:cold_posterior_small_0_8} & zero mean \& isotropic\\
    \ref{pgfplots:cold_posterior_small_iso_3} & \ref{pgfplots:cold_posterior_small_iso_5} & \ref{pgfplots:cold_posterior_small_iso_8} & isotropic\\
    \ref{pgfplots:cold_posterior_small_learned_3} & \ref{pgfplots:cold_posterior_small_learned_5} & \ref{pgfplots:cold_posterior_small_learned_8} & learned
    \end{tabular}};
\end{tikzpicture}

%% file: tikz/appendix/results/cold_posterior_large.tex
\begin{tikzpicture}
\begin{axis}[
name=boundary,
footnotesize,
width=\textwidth,
height=0.6\textwidth,
log basis x={10},
xlabel={Temperature scaling $\tau$},
xmode=log,
max space between ticks=40pt,
ylabel={Accuracy},
ymin=0, ymax=100,
]
\path [draw=TUMBlue, fill=TUMBlue, opacity=0.2]
(axis cs:1e-15,90.1)
--(axis cs:1e-15,88.78)
--(axis cs:1e-14,88.77)
--(axis cs:1e-13,77.44)
--(axis cs:1e-12,9.96)
--(axis cs:1e-11,9.6)
--(axis cs:1e-10,10.21)
--(axis cs:1e-09,9.7)
--(axis cs:1e-08,10.03)
--(axis cs:1e-08,10.32)
--(axis cs:1e-08,10.32)
--(axis cs:1e-09,10.35)
--(axis cs:1e-10,10.37)
--(axis cs:1e-11,10.32)
--(axis cs:1e-12,10.79)
--(axis cs:1e-13,84.42)
--(axis cs:1e-14,90.12)
--(axis cs:1e-15,90.1)
--cycle;
\path [draw=TUMBlue, fill=TUMBlue, opacity=0.2]
(axis cs:1e-15,90.77)
--(axis cs:1e-15,88.31)
--(axis cs:1e-14,88.31)
--(axis cs:1e-13,88.31)
--(axis cs:1e-12,88.28)
--(axis cs:1e-11,10.2)
--(axis cs:1e-10,9.99)
--(axis cs:1e-09,10)
--(axis cs:1e-08,10)
--(axis cs:1e-08,10.02)
--(axis cs:1e-08,10.02)
--(axis cs:1e-09,10)
--(axis cs:1e-10,10)
--(axis cs:1e-11,36.96)
--(axis cs:1e-12,90.73)
--(axis cs:1e-13,90.78)
--(axis cs:1e-14,90.77)
--(axis cs:1e-15,90.77)
--cycle;
\path [draw=TUMAccentOrange, fill=TUMAccentOrange, opacity=0.2]
(axis cs:1e-15,85.39)
--(axis cs:1e-15,10.07)
--(axis cs:1e-14,9.73)
--(axis cs:1e-13,10)
--(axis cs:1e-12,10)
--(axis cs:1e-11,10)
--(axis cs:1e-10,10)
--(axis cs:1e-09,10)
--(axis cs:1e-08,10)
--(axis cs:1e-08,10)
--(axis cs:1e-08,10)
--(axis cs:1e-09,10)
--(axis cs:1e-10,10)
--(axis cs:1e-11,10)
--(axis cs:1e-12,10)
--(axis cs:1e-13,10)
--(axis cs:1e-14,14.56)
--(axis cs:1e-15,85.39)
--cycle;
\path [draw=TUMAccentOrange, fill=TUMAccentOrange, opacity=0.2]
(axis cs:1e-15,87.93)
--(axis cs:1e-15,84.87)
--(axis cs:1e-14,84.87)
--(axis cs:1e-13,84.87)
--(axis cs:1e-12,85.53)
--(axis cs:1e-11,84.83)
--(axis cs:1e-10,84.8)
--(axis cs:1e-09,10)
--(axis cs:1e-08,10)
--(axis cs:1e-08,10.09)
--(axis cs:1e-08,10.09)
--(axis cs:1e-09,87.07)
--(axis cs:1e-10,87.88)
--(axis cs:1e-11,87.93)
--(axis cs:1e-12,87.57)
--(axis cs:1e-13,87.93)
--(axis cs:1e-14,87.93)
--(axis cs:1e-15,87.93)
--cycle;
\addplot [very thick, TUMAccentOrange, dotted]
table {%
1e-15 35.31
1e-14 11.43
1e-13 10
1e-12 10
1e-11 10
1e-10 10
1e-09 10
1e-08 10
};
\label{pgfplots:cold_posterior_iso_5}
\addplot [very thick, TUMAccentOrange, dotted, mark=*]
table {%
1e-15 86.55
1e-14 86.55
1e-13 86.55
1e-12 86.55
1e-11 86.54
1e-10 86.52
1e-09 35.69
1e-08 10.03
};
\label{pgfplots:cold_posterior_iso_3}
\addplot [very thick, TUMBlue]
table {%
1e-15 89.3666666666667
1e-14 89.37
1e-13 81.8233333333333
1e-12 10.26
1e-11 10.0433333333333
1e-10 10.2633333333333
1e-09 10.0066666666667
1e-08 10.13
};
\label{pgfplots:cold_posterior_learned_5}
\addplot [very thick, TUMBlue, mark=*]
table {%
1e-15 89.3933333333333
1e-14 89.3933333333333
1e-13 89.3966666666667
1e-12 89.3866666666667
1e-11 19.81
1e-10 9.99666666666667
1e-09 10
1e-08 10.0066666666667
};
\label{pgfplots:cold_posterior_learned_3}

\end{axis}
\node[draw,fill=white,inner sep=2pt, below=1cm] at (boundary.south) {\small
    \begin{tabular}{ccl}
    $\gamma = 10^{-3}$ & $\gamma = 10^{-5}$ \\
    \ref{pgfplots:cold_posterior_iso_3} & \ref{pgfplots:cold_posterior_iso_5} & isotropic\\
    \ref{pgfplots:cold_posterior_learned_3} & \ref{pgfplots:cold_posterior_learned_5} & learned
    \end{tabular}};
\end{tikzpicture}